\documentclass{article}

\usepackage[final,nonatbib]{neurips_2021}

\usepackage[utf8]{inputenc} 
\usepackage[T1]{fontenc}    
\usepackage{hyperref}       
\usepackage{url}            
\usepackage{booktabs}       
\usepackage{amsfonts}       
\usepackage{nicefrac}       
\usepackage{microtype}      
\usepackage{longtable}
\usepackage{geometry}

\PassOptionsToPackage{usenames,dvipsnames,svgnames}{xcolor}
\usepackage{tikz} 
\usepackage{wrapfig,subcaption}
\usepackage{latexsym}
\usepackage{tikz-cd}
\usetikzlibrary{shapes}
\usetikzlibrary{arrows}
\usetikzlibrary{er,positioning}
\usepackage[vlined,linesnumbered]{algorithm2e}
\usepackage{algpseudocode}
\usepackage{amsmath,amssymb,amsthm}
\usepackage{booktabs}
\usepackage[scaled=0.85]{beramono}
\newtheorem{example}{Example}
\newtheorem{theorem}{Theorem}

\newtheorem{proposition}{Proposition}
\newtheorem{remark}{Remark}

\newtheorem{definition}{Definition}

\newtheorem{task}{Task}
\newtheorem{assumption}{Assumption}
\newcommand{\rctl}{{\sc \texttt{SCTL}}}
\newcommand{\ess}{{\sc \texttt{ESS}}}
\newcommand{\mb}{\textrm{\textbf{Mb}}}
\newcommand{\nbr}{\textrm{\textbf{ne}}}
\newcommand{\pa}{\textrm{\textbf{pa}}}
\newcommand{\ch}{\textrm{\textbf{ch}}}
\newcommand{\an}{\textrm{\textbf{an}}}
\newcommand{\adj}{\textrm{\textbf{adj}}}
\newcommand{\spouse}{\textrm{\textbf{sp}}}

\title{Scalable Causal Domain Adaptation}

\author{%
  Mohammad Ali Javidian \\
  School of Electrical and Computer Engineering \\
  Purdue University\\
  \texttt{mjavidia@purdue.edu}
   \And
   Om Pandey \\
   School of Computer Engineering \\
   Kalinga Institute of Industrial Technology \\
   \texttt{opandey108@gmail.com} \\
   \AND
   Pooyan Jamshidi \\
   Computer Science \& Engineering Department\\
   University of South Carolina \\
   \texttt{pjamshid@cse.sc.edu} \\
}

\begin{document}

\maketitle

\begin{abstract}
  One of the most critical  problems in transfer learning is the task of domain adaptation, where the goal is to apply an algorithm trained in one or more source domains to a different (but related) target domain. This paper deals with domain adaptation in the presence of \textit{covariate shift} while invariance exist across domains. One of the main limitations of existing causal inference methods for solving this problem is \textit{scalability}.
To overcome this difficulty, we propose \rctl, an algorithm that avoids an exhaustive search and identifies invariant causal features across source and target domains based on Markov blanket discovery. \rctl~ does not require having  prior knowledge of the causal structure,
the type of interventions, or the intervention targets. There is an intrinsic locality associated with \rctl~that makes it practically scalable and robust because local causal discovery increases the power of computational independence tests and makes the task of domain adaptation computationally tractable. We show the scalability and robustness of \rctl~for domain adaptation using synthetic and real data sets in low-dimensional  and high-dimensional settings.
\end{abstract}

\section{Introduction}\label{sec:intro}
Standard supervised learning usually assumes that both training and test data are drawn from the same distribution. However, this is a strong assumption and often violated in practice if (1) training
examples have been obtained through a biased method (i.e., \textit{sample selection bias}), or (2) there exist a significant physical or temporal difference between training and test data sources (i.e., \textit{non-stationary environments}) \cite{li19b}. Domain adaptation approaches aim to learn domain invariant features to mitigate the problem of data shift and enhance the quality of predictions \cite{redko19a,stojanov19b,park20b}.

Many domain adaptation
techniques in the literature consider the \textit{covariate shift}  \cite{chen16d,kugelgen19a,stojanov19a,li20b,kisamori20a}, where the marginal
distribution of the features differs across the source and target domains, while the conditional
distribution of the target given the features does not change. In causal inference, it has
been noted that covariate shift corresponds to causal learning, i.e., predicting effects from
causes \cite{ICML2012Schoelkopf_625}. Therefore, taking into account the causal structure of a system of interest and finding causally invariant features in both source and target domains enable us to safely transfer the predictions of the target variable based only on causally invariant features to the target domain \cite{MagliacaneNIPS18,rojas2018invariant,subbaswamy2018counterfactual,subbaswamy2019preventing}. Following the same line of work, in this paper, we formalize and study the problem of domain adaptation as a feature selection problem where we aim to find an \textit{optimal subset} that the conditional distribution of the target variable given this subset of predictors is invariant across domains under certain assumptions, as formally discussed in section \ref{sec:method}. To illustrate the importance and effectiveness of this approach, consider the following example.
\begin{example}[Domain Adaptation: Prediction of Diabetes	at Early Stages]
	According to Diabetes Australia, early diagnosis and initiation of appropriate treatment plays a pivotal role in: (1) helping patients to manage the disease early 
	and (2) reducing the substantial economic impact of diabetes on the healthcare systems and national economies (700 million dollars each year \cite{Australia}). To predict diabetes using machine learning techniques, we need its symptoms and clinical data.
	The common symptoms and possible causes 
	of diabetes Type II are weakness, obesity, delayed healing, visual blurring, partial paresis, muscle stiffness, alopecia, among others \cite{ADA}. 
	Although we do not know what causes diabetes Type II \cite{Australia}, one may argue that the occurrence, rate, or frequency of \textit{Delayed Healing}, \textit{Blurred Vision}, \textit{Partial Paresis}, and \textit{Weakness} increases with age, and \textit{Delayed Healing} and \textit{Partial Paresis} may cause some complications that result in pancreas malfunction.
	
\begin{figure}[!ht]
\centering
\captionsetup[subfigure]{justification=centering}
\centering
    \subcaptionbox{{Causal Graph for Diabetes}}[.33\textwidth]{%
    \begin{tikzpicture}[transform shape,scale=.55]
	\tikzset{vertex/.style = {shape=ellipse,inner sep=2pt,align=center,draw=black, fill=white}}
\tikzset{edge/.style = {->,> = latex',thick}}
	\node[vertex,thick,fill=Cyan] (f) at  (2.5,5.5) {\large Delayed Healing};
	\node[vertex,thick,fill=green] (g) at  (3.35,1) {\large Diabetes};
	\node[vertex,thick,fill=Dandelion] (b) at  (-.35,2.25) {\large Weakness};
	\node[vertex,thick] (a) at  (-2.5,5.5) {\large Age};
	\node[vertex,thick,fill=Dandelion] (c) at  (-1.5,1) {\large Blurred Vision};
	\node[vertex,thick,fill=Cyan] (e) at  (.75,3.75) {\large Partial Paresis};
	\draw[edge] (a) to (b);
	\draw[edge] (a) to (e);
	\draw[edge] (a) to (f);
	\draw[edge] (a) to (c);
	\draw[edge] (e) to (g);
	\draw[edge] (g) to (b);
	\draw[edge] (g) to (c);
	\draw[edge] (f) to (g);
\end{tikzpicture}
}%
\subcaptionbox{Diabetes (Young Patients)}[.33\textwidth]{%
\includegraphics[width=.9\linewidth]{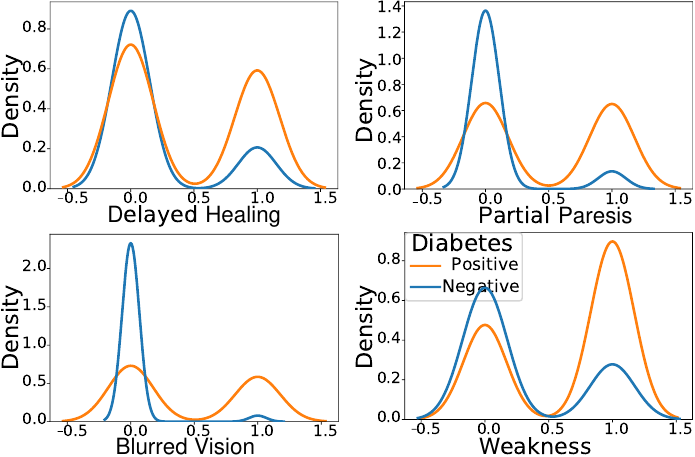}
}
\subcaptionbox{Diabetes (Old Patients)}[.33\textwidth]{%
\includegraphics[width=.9\linewidth]{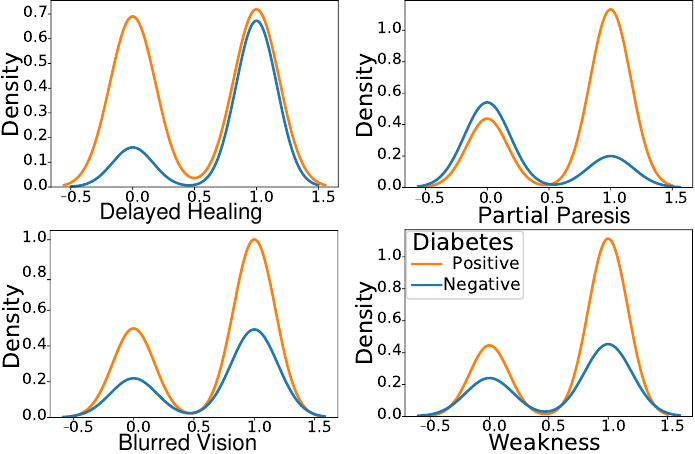}
}
    \caption{\footnotesize{Prediction of Diabetes
at an early stage: Here, age intervention leads to a shift of distributions between the source (Young Patients $<50$) and target (Old Patients $\ge 50$) domains (see also Example 1). In all cases, the \textcolor{orange}{orange curves} indicate the distribution of \textcolor{orange}{tested positive} patients, and the \textcolor{blue}{blue curves} indicate  the distribution of \textcolor{blue}{tested negative} patients. A standard feature selection method that does snot consider the causal structure and would use all features to predict \textit{Diabetes} would obtain biased predictions in the target domain (MSE = 0.29  and SSE = 65.322). Using only \textit{Delayed Healing} and \textit{Partial Paresis} instead yields less accurate predictions in the source domain, but much more accurate ones in the target domain (MSE = 0.221 and SSE = 49.8927)}.}
    \label{fig:diabetes}
\end{figure}
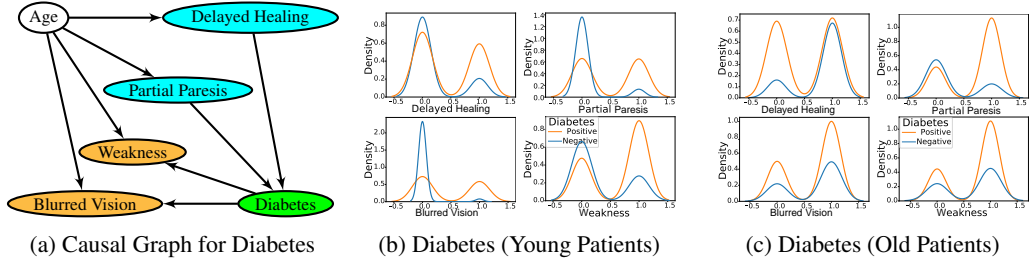
The causal graph of this scenario can be represented as the directed acyclic graph (DAG)  in Figure \ref{fig:diabetes} (a). To provide an instance of a domain adaptation problem, we divided a diabetes dataset \cite{islam2020likelihood} into two subpopulations: (1) source domain with patients of the age less than 50 (dubbed young) and (2) target domain with patients of the age greater than or equal to 50 (old patients). As shown in Figure \ref{fig:diabetes} (b) and (c), intervention in age leads to shifting  distributions across domains. In our experiments, feature selection methods that do not consider the causal structure select \textit{highly relevant features} to diabetes, i.e., all four variables \textit{Delayed Healing}, \textit{Blurred Vision}, \textit{Partial Paresis}, and \textit{Weakness} to achieve high prediction accuracy in the source domain.  This per se causes worse predictions in the target domain (MSE = 0.29 and SSE = 65.322) than the case that we only consider causally invariant features, i.e., \textit{Delayed Healing} and \textit{Partial Paresis} for predictions in the target domain (MSE = 0.221 and SSE = 49.8927). The reason is that conditioning on the variables \textit{Blurred Vision} and \textit{Weakness} makes the paths between age and diabetes open, and hence  the set of all features does not generalize	to the target domain. However, conditioning only on \textit{Delayed Healing} and \textit{Partial Paresis} blocks the paths between age and diabetes. Hence,  these causally invariant features enable us to predict diabetes with higher accuracy in the target domain, even in distribution shifts due to the age intervention.
\end{example}

An important limitation of existing causal inference methods \cite{MagliacaneNIPS18,rojas2018invariant,subbaswamy2018counterfactual,subbaswamy2019preventing} is that they, currently, \textit{do not scale beyond dozens of variables} due to either exponentially large number of conditional independence (CI) tests or difficulties in causal structure recovery \cite{kouw2019review}.
To overcome the problem of \textit{scalability}, we propose an algorithm based on Markov blanket discovery (see Appendix \ref{sec:defs} for the definition) that in contrast to existing methods: (1) it takes advantage of local computation by finding only the Markov blanket of the target variable(s), and hence (2) reduces the search space for finding causally invariant features drastically because of the small size of Markov blankets  ($\le 10$) in (many) causal models in practice, e.g., see \cite{bnlearn}. As a result, the CI tests becomes more reliable and robust when the number of variables increases, an essential characteristic in high dimensional and low sample size scenarios.  Our main contributions are as follows:\\
$\bullet$ We propose a new algorithm, called \textbf{S}calable \textbf{C}ausal \textbf{T}ransfer \textbf{L}earning (\rctl), to solve the problem of domain adaptation in the presence of covariate shift and scales to high-dimensional problems (section \ref{sec:method}).\\
$\bullet$ For the first time we characterize Markov blankets in \textbf{A}cyclic \textbf{D}irected \textbf{M}ixed \textbf{G}raphs (ADMGs) i.e., causal graphs in the presence of unmeasured confounders\footnote{In this paper we will assume that there is
no selection bias.}, and we prove that the standard Markov blanket discovery algorithms such as Grow-Shrink (GSMB), IAMB and its variants are still correct under the faithfulness assumption where causal sufficiency is not assumed. We prove the correctness of \rctl~based on these new theoretical results (section \ref{sec:method}).\\
$\bullet$ We demonstrate on synthesized and real-world data that our proposed algorithm improves performance over several state-of-the-art algorithms in the presence of covariate shift (section \ref{sec:results}). 
Code and data for reproducing our results is available at supplementary materials.
\url{https://github.com/softsys4ai/SCTL}.

\section{Related Work}\label{sec:relatedwork}
\textbf{Domain Adaptation.} Here, we provide a brief overview of the main domain adaptation scenarios that can be found in the literature. There are three main domain adaptation problems: (1) \textbf{Covariate shift}, which is one of the most studied forms of data shift, occurs if the marginal distributions of \textit{context variables} change across the source and target domains while the posterior (conditional) distributions are the same between source and target domains \cite{SHIMODAIRA2000,sugiyama2008direct,johansson19a}. (2) \textbf{Target shift} occurs if the marginal distributions of the \textit{target variable} change across the  source and target domains while the posterior distributions remain the same \cite{Storkey09,zhang2013domain,Lipton18}.
 (3) \textbf{Concept shift} occurs if marginal distributions between source and target domains remain unchanged while the posteriors change across the domains \cite{moreno2012unifying,Zhang15,gong2016domain}. 
We only focus on the \textit{covariate shift} in this paper. Since assuming invariance of conditionals makes sense if the conditionals represent causal mechanisms \cite{rojas2018invariant}, we use the relaxed version of the usual covariate shift assumption and assume that it holds for a subset of predictor variables, as suggested and used in \cite{rojas2018invariant,MagliacaneNIPS18}. 

\noindent\textbf{Causal Inference in Domain Adaptation.} Here, we briefly provide an overview of causal inference methods that address the problem of covariate shift: (1) \textbf{Transportability} \cite{BareinboimPearl11,BareinboimPearl12,BareinboimPearl14,Correa-ijcai2019} expresses knowledge about differences and commonalities between the source and target domains in a formalism called \textit{selection diagram}. Using this representation and the \textit{do-calculus} \cite{Pearl09}, enable us to derive a procedure for deciding whether effects in the target domain can be inferred from experiments conducted in the source domain(s). 
(2) \textbf{Invariant Causal Prediction} (ICP) techniques \cite{peters2016causal,pfister2019invariant,pfister2019stabilizing} search
for finding a subset of variables to estimate individual regressions for
each domain in a way that produces regression residuals with equal distribution across all
domains. Theoretical identifiability guarantees for the set of direct causal predictors for ICP techniques are limited to the linear Gaussian models. In other words, when the model is non-linear, non-Gaussian, or there are latent confounders, there is no guarantee that the obtained set is the set of direct causal predictors \cite{peters2016causal,pfister2019stabilizing}. 
(3) \textbf{Graph surgery} \cite{subbaswamy2018counterfactual,subbaswamy2019preventing}
removes variables generated by unstable mechanisms from the joint factorization to yield a distribution \textit{invariant} to the differences across domains. (4) \textbf{Graph pruning} methods \cite{MagliacaneNIPS18,rojas2018invariant} formalized as a feature selection problem in which the goal is to find
the \textit{optimal subset} that the conditional distribution of the target variable given this subset of predictors is invariant across domains. Both transportability and graph surgery methods need to know the causal model of interest in advance.
However, learning causal models from data is a challenging task, especially in the presence of \textit{unmeasured confounders} \cite{Glymour19}. On the other hand, graph pruning methods do not rely on prior knowledge of the causal graph, but they currently \textit{do not scale beyond dozens of variables} \cite{kouw2019review}.

\section{Theory}\label{sec:method}
In this section, first, we prove that the domain adaptation task under the  \textbf{C}ausal \textbf{D}omain \textbf{A}daptation (CDA) assumptions can be done effectively via searching inside the set of \textit{Markov blanket} of the target variable $T$ rather than a brute-force search over all variables as it has been done in \textbf{E}xhaustive \textbf{S}ubset \textbf{S}earch (\ess) in \cite{MagliacaneNIPS18}, the closest work  to~\rctl. In the next section, we present an \textit{efficient} and \textit{scalable} algorithm, called the \rctl, that searches over the \textit{Markov blanket} of the target variable $T$ to find causally invariant features that provide an accurate prediction for the target variable $T$ given the distribution shift in terms of context variables. To apply \rctl~in practice, we need an algorithm for Markov blanket discovery in the presence of unmeasured confounders. For this purpose, we prove that GSMB, IAMB, and its variants are still sound under the faithfulness assumption, even when the causal sufficiency assumption does not hold. The proof of theorems can be found in Appendix \ref{app:A}.

\subsection{Problem: Causal Domain Adaptation}\label{sec:defCDA}
Here we formally state the causal domain adaptation task that we address in this work:
\begin{task}[Domain Adaptation Task]
We are given data for a source and a target domain such that the marginal distribution of the context variable $C$ changes across domains. Assume the source and target domains data are complete (i.e., no missing values), except for all values of a specific target variable $T$. The task is to predict these missing values of the target variable $T$ given the available source and target domains data.  
\end{task}

For predicting $T$ from a subset of features $S \subseteq V \setminus \{C,T\}$, where $T$ is the target variable, $C=\{C_{s\in D_s},C_{t\in D_t}\}$ is the context variable with the values $C_{s\in D_s}$ in the source domain(s) and the values $C_{t\in D_t}$ in the target domain (note that we can extend our formalism to multiple context variables), and $S$ is a set that separates $T$ from $C$ in the causal model, we define the \textit{transfer bias} as $\hat{T}^t_S-\hat{T}^s_S$, where $\hat{T}^t_S:=\mathbb{E}(T|S, C=C_{t\in D_t})$ and $\hat{T}^s_S:=\mathbb{E}(T|S, C=C_{s\in D_s})$. We define the \textit{incomplete information bias} as $\hat{T}^t_{V \setminus \{C,T\}}-\hat{T}^t_S$. The \textit{total bias} when using $\hat{T}^s_S$ to predict $T$ is the sum of the transfer bias and the incomplete
information bias:
$$\underbrace{\hat{T}^t_{V \setminus \{C,T\}}-\hat{T}^s_S}
_\text{total bias} = \underbrace{\hat{T}^t_S-\hat{T}^s_S}
_\text{transfer bias}+\underbrace{\hat{T}^t_{V \setminus \{C,T\}}-\hat{T}^t_S.}
_\text{incomplete information bias}$$
For more details, see \cite{MagliacaneNIPS18}.
Note that using invariant features only guarantees transfer bias to be zero, but the incomplete information bias can be quite large for the invariant feature set. 
There is a trade-off
between the two terms in the total bias expression, which is hard to determine because the two terms are not identifiable using the source data alone.

\subsection{Assumptions}\label{sec:mainass}
We consider the same assumptions as discussed in \cite{MagliacaneNIPS18}, the closest work to ours. 
\begin{assumption}[\textbf{J}oint \textbf{C}ausal \textbf{I}nference (JCI) assumptions]\label{assum:jci}
We consider that causal graph $G=(V,E)$ is an ADMG with  the variable set $V$. From now on, we will distinguish \textit{system variables} $ X_{j \in J}$ describing the system of interest,
and \textit{context variables} $C_{i \in I}$ describing the context in which the system has been observed:\\
(a) Context variables are never caused by system variables: $(\forall j\in J,i\in I: X_j \rightarrow C_i \notin E )$,\\
(b) System variables are not confounded by context variables: $(\forall j\in J,i\in I: X_j \leftrightarrow C_i \notin E )$, 
(c) All pairs of context variables are confounded: $(\forall i,k \in I: C_i \leftrightarrow C_k \in E \textrm{, and }\forall i,k \in I: C_i \rightarrow C_k \notin E)$.
\end{assumption}
The case (a) in the assumption is called \textit{exogeneity} and captures what we mean by “context”. Cases (b), (c) are not as important as the exogeneity and can be relaxed, depending on the application \cite{MagliacaneNIPS18}. 

\begin{assumption}[\textbf{C}ausal \textbf{D}omain \textbf{A}daptation (CDA) assumptions]\label{assum:cda}
We consider a causal graph $G$ that satisfies Assumption \ref{assum:jci}. We say $G$ satisfies CDA assumptions if

\noindent(a) The probability distribution of $V$ and the causal graph $G=(V,E)$ satisfy Markov condition and faithfulness assumption,\\
(b) For the target variable $T$ and a set $A$ that $T \perp\!\!\!\perp  A |S$ (i.e., $T$ is independent of $A$ given $S$) in the source domain we have the same CI  in the target domain, where $\{T\}\not\subseteq A\cup S$,\\
(c) No context variable $C$ is the parent of the target variable $T$, i.e., $C \rightarrow T \notin E_G$.
\end{assumption}

Assumption \ref{assum:cda}(b) states that the pooled source and target domains distributions are Markov and faithful to the local causal structure around the target variable of interest $T$. This assumption implies that the causal structure of the target variable $T$ is invariant when going from the source to the target domain. Note that Assumption \ref{assum:cda}(b) of the CDA assumption in \cite{MagliacaneNIPS18} is different from ours: both the pooled source and pooled target domains distribution are \textit{Markov} and \textit{faithful} to $G$'s subgraph, which excludes the context variable $C$. This strong assumption forces them to do an \textbf{exhaustive search}. Assumption \ref{assum:cda}(c) is strong and restrictive, which might not hold in real data sets. However, since interventions are local in their nature\footnote{This is called \textit{modularity assumption} in the literature \cite{neal2020introduction}.  It means that intervening on a variable $X_i$ only changes the causal mechanism for $X_i$, i.e., $p(X_i|\pa({X_i}))$ and it does not change the causal mechanisms that generate
any other variables.}, if the
interventions are targeted precisely, then Assumption \ref{assum:cda}(c) is more likely to be satisfied.

\subsection{Theoretical Results}\label{sec:mainthm}
The following theorem enables us to localize the task of finding invariant features under the CDA assumptions. 
\begin{theorem}\label{thm:rctl}
Assume that CDA assumptions hold for the context $C_{i \in I}$ and system variables $X_{j \in J}$ given data for  single or multiple source domains. To find the best separating set(s) of features that $m$-separate(s) $C_{i \in I}$ from the target variable $T$ in the causal graph $G$, it is enough to restrict our search to the set of Markov blanket, $\mb(T)$, of the target variable $T$. 
\end{theorem}

Theorem \ref{thm:rctl} enables us to develop a new \textit{efficient} and \textit{scalable} algorithm, called \textbf{S}calable \textbf{C}ausal \textbf{T}ransfer \textbf{L}earning (\rctl), that exploits locality for learning \textit{invariant causal features} to be used for domain adaptation. 
Since CDA assumptions do not require \textit{causal sufficiency assumption}, we need sound and scalable algorithms for Markov blanket discovery in the presence of unmeasured confounders. For this purpose, we first need to provide a graphical characterization of Markov blankets in ADMGs. 

Let $G = (V,E,P)$ be an ADMG model. Then, $V$ is a set of random variables, $(V,E)$ is an ADMG, and $P$ is a joint probability distribution over $V$.  Let $T\in V$, then the \textit{Markov blanket} $\textbf{Mb}(T)$ is the set of all variables that there is a \textit{collider path} between them and $T$.
We now show that the Markov blanket of the target variable $T$ in an ADMG probabilistically shields $T$ from the rest of the variables. Under the faithfulness assumption, the Markov blanket is the \textit{smallest set} with this property. Formally we have:

\begin{theorem}\label{thm:admg} 
	Let $G=(V,E,P)$ be an ADMG model.
	Then, $T{\!\perp\!\!\!\perp}_p V\setminus\{T,\textbf{Mb}(T)\}|\textbf{Mb}(T)$.
\end{theorem}

Our following  theorem safely enables us to use standard Markov blanket recovery algorithms for domain adaptation task without causal sufficiency assumption:
\begin{theorem}\label{thm:Mbalgs}
Given the Markov assumption and the faithfulness assumption, a causal system represented by an ADMG, and i.i.d. sampling, in the large sample limit, the Markov blanket recovery algorithms GSMB \cite{Margaritis2003}, IAMB \cite{Tsamardinos0Mb}, Fast-IAMB \cite{Yaramakala05}, Interleaved Incremental Association (IIAMB) \cite{Tsamardinos0Mb}, and IAMB-FDR \cite{Pena08Mb} correctly identify all Markov blankets for each variable.  (Note that Causal Sufficiency is not assumed.)
\end{theorem}

\section{{\sc \texttt{\textbf{SCTL}}}: \underline{S}calable \underline{C}ausal \underline{T}ransfer \underline{L}earning}\label{sec:rctlalg}
\rctl~addresses the task of domain adaptation by finding a separating set $S \subset V$, where $V$ is the set of context variables $C_{i \in I}$ and system variables $X_{j \in J}$ such that for the target variable $T$ we have $C_i \perp\!\!\!\perp  T| S$, for every $i\in I$ in the source domain. Since Assumption  \ref{assum:cda}(b) implies that this conditional independence holds across domains, if such a separating set $S$ can be found, $S$ is considered as a set of causally \textit{invariant} features for $T$ across environments. 
Our \rctl~algorithm, described in Algorithm \ref{alg:rctl}, consists of two main steps: \\
\textbf{Step 1.} We find the Markov blanket of the target variable $T$, i.e., $\mb(T)$ (line 3 in Algorithm \ref{alg:rctl}). Using the property of Markov blankets, i.e., $T \perp\!\!\!\perp  V\setminus \{T,\mb(T)\}| \mb(T)$, if there is no context variable in the Markov blanket of the target variable $T$, then $T{\perp\!\!\!\perp}  C_{i \in I}|\mb(T)$. This means that the Markov blanket of $T$ provides a minimal feature set required
for predicting the target variable $T$ with maximum predictivity \cite{JMLRaliferis10a}.

\noindent\textbf{Step 2.} If there exists a context variable $C$ in $\mb(T)$, we consider all possible subsets of the Markov blanket of the target variable $T$, $\mb(T)$, to find those subsets $S_i$s that satisfy the separating condition $T\perp\!\!\!\perp  C_{i \in I}|S_i$. For this purpose, the algorithm filters out those subsets for which the $p_{value}$ is below the significance level $\alpha$, i.e., $T$ is \textit{not} conditionally independent of the context variables given those sets. At the end of line 14, we have a list $S$ that contains all possible separating sets. If $S$ is empty, we abstain from making predictions as the domain adaptation task is unsuccessful. This may happen if there exists a context variable $C$ in the adjacency of the target variable $T$, which means the data set does not satisfy the CDA assumptions because it violates the Assumption \ref{assum:jci}(a), (b), or Assumption \ref{assum:cda}(c). In this case, the algorithm throws a failure because 
there is no subset of features that is causally invariant across domains.  Otherwise, we sort the list $S$ based on the obtained prediction errors in the source domain (line 18). Then, the algorithm returns those subsets $sub_i$ in $S$ with the lowest prediction error in the source domain as the best possible separating sets, because according to Theorem \ref{thm:rctl}, these sets $m$-separates $T$ from the context variable(s) and minimize(s) the transfer bias. The computational complexity of \rctl~is discussed in the supplementary material.

\begin{algorithm}[!ht]
 \caption{\rctl: A Scalable Transfer Learning Algorithm for Causal Domain Adaptation}\label{alg:rctl}
\SetAlgoLined
\scriptsize\KwIn{A Dataset with variable set $V$ is the set of context variables $C_{i \in I}$ and system variables $X_{j \in J}$, target variable \textit{T}, and significance level $ \alpha $.} 
 
\KwOut{A list of separating sets or null.}

\textcolor{ForestGreen}{
\tcp{Step 1: Find $\mb(T)$, Markov blanket of $T$.}}
Set $S$ as an empty list\;

Find $\mb(T)$\;

\uIf{$\mb(T)\cap C_i,_{i \in I} = \emptyset $ }{
\textbf{return} $S=\mb(T)$\;
}
\Else{
\textcolor{Red}{
\tcp{Step 2: Find a subset of $\mb(T)$ that makes $T$ $m$-separated from the context variable(s).}}

Set Subs = \textbf{Subset}($\mb(T)$) \tcc*[H]{All possible combinations of subsets for the obtained set $\mb(T)$} 
\For{\textit{each} $sub_i \in $  Subs }{
$p_{val}$ = $p_{value}(C_{i \in I} \perp\!\!\!\perp T | sub_i )$ \;

\If{$p_{val} > \alpha $}{
\tcp{This means $T$ is conditionally independent of the context variables given the set $sub_i$.}
Add $sub_i$ to \textit{S}\;
}
}
\If{S = $\emptyset$}{
\textbf{return null}\;
}
Sort $S$ \tcc*[H]{Sort the list $S$ in descending order based on prediction errors in the source domain(s)}
\textbf{return} \textit{a list of elements of $S$} with the lowest prediction error in the source domain(s);
}
\end{algorithm}

\begin{remark}\label{rem:binarycontextvars}
Assume the causal graph in Figure \ref{fig:rem} with the binary context variables $C_1$, $C_2$, and the target variable $T$.
Also, assume that the domain adaptation task is to predict the missing values of $T$ in the target
domains ($C_1 = 1$), based on the observed data from the source domain ($C_1 = 0$)  and the target domain
without knowledge of the causal graph. In this case, the true Markov blanket of the target variable $T$ is the set of $\mb(T)=\{X,Y,C_1\}$.
However, since the value of the context variable $C_1$ is fixed in the source domain, Markov blanket discovery algorithms cannot discover the true Markov blanket in such cases. If we have access to the part of the data with no missing values for $T$ in the target domain, which is a realistic scenario in some situations, e.g., clinical data from a target hospital, the data from the source and the target domains can be used for learning the true $\mb(T)$. Otherwise, the problem is reduced to Markov blanket discovery in the presence of \textit{missing data}. For this purpose, we can use the proposed method in \cite{Friedman97}.     
\end{remark}

\begin{remark}\label{rem:SR}
It has already been suggested in \cite[Theorem 3.5]{pfister2019invariant} that the optimal set of causally 
invariant features is a subset of the Markov blanket of the target variable, called \emph{stable blanket}. Here, we highlight the differences between \cite{pfister2019invariant} and \rctl:\\
\textbf{(1)} In \cite{pfister2019invariant}, the underlying causal structure for the system variables is considered as a DAG, which means there is no latent confounder among system variables (causal sufficiency assumption). In contrast, in our work, this restrictive assumption is relaxed. Note that causal sufficiency assumption is a strong assumption and often violated in practice. Moreover, the paper provides no theoretical guarantee that the stable blanket is a subset of
the Markov blanket of the target variable $T$, $\mb(T)$, in the presence of latent confounders. In fact, the invariance principle used in \cite{peters2016causal} shows that the stable set returned by the ICP techniques is a subset of ancestors of $T$. However, in our work, Theorem \ref{thm:rctl} relaxes the causal sufficiency assumption and states that searching in $\mb(T)$ is enough to find causally invariant features even in the presence of latent confounders.\\
\textbf{(2)} In \cite{pfister2019invariant}, the joint probability distribution of variables is assumed to be continuous and experimental evaluations are restricted to the linear Gaussian models, while in our work, there is no restriction on the type of data and the probability distribution of variables. In fact, in section \ref{sec:eval}, we do a comprehensive evaluation that is not restricted to the linear Gaussian models.
\end{remark}
\begin{figure}[!ht]
    \centering
    \begin{minipage}{.45\textwidth}
        \centering
        \begin{tikzpicture}[transform shape,scale=.75]
	\tikzset{vertex/.style = {shape=circle,align=center,draw=black, fill=white}}
\tikzset{edge/.style = {->,> = latex',thick}}
    \node[vertex,thick,fill=green] (c1) at  (-2.5,5) {$C_1$};
    \node[vertex,thick,fill= gray!40,dashed] (u) at  (-.5,6) {U};
    \node[vertex,thick,fill=green] (c2) at  (1.5,5) {$C_2$};
    \node[vertex,thick](x) at  (1.5,3) {X};
    \node[vertex,thick,fill=Dandelion ] (t) at  (-0.5,2) {T};
    \node[vertex,thick ] (y) at  (-2.5,3) {Y};
	\draw[edge,fill = Cyan] (u) to (c1);
	\draw[edge,fill = Cyan] (u) to (c2);
	\draw[edge] (c1) to (y);
	\draw[edge] (c2) to (x);
	\draw[edge] (x) to (t);
	\draw[edge] (t) to (y);
\end{tikzpicture}
 \caption{\footnotesize{Ground truth synthetic graph: The grey node \textcolor{gray!40}{U} represents the unknown confounder connecting the context variables filled in green \textcolor{green}{$C_1$} \& \textcolor{green}{$C_2$}. The node in orange, \textcolor{Dandelion}{T}, represent the target variable.}}\label{fig:rem}
    \end{minipage}\quad
    \begin{minipage}{0.5\textwidth}
        \centering
        \begin{tikzpicture}[transform shape,scale=.7]
	\tikzset{vertex/.style = {shape=circle,align=center,draw=black, fill=white}}
\tikzset{edge/.style = {->,> = latex',thick}}
	\node[vertex,thick,dotted] (l) at  (-5.3,5) {L};
	\node[vertex,thick,dotted] (k) at  (-3.7,5) {K};
	\node[vertex,thick,dotted] (j) at  (-3.7,3.8) {J};
	\node[vertex,thick,fill=Dandelion,dotted,dotted] (m) at  (-5.3,3.8) {M};
	\node[vertex,thick,fill=Dandelion,dotted] (n) at  (-4.5,2.6) {N};
    \node[vertex,thick,fill=green] (c1) at  (-2.5,5) {$C_1$};
    \node[vertex,thick,fill= gray!40,dashed] (u) at  (-1.5,6.2) {U};
    \node[vertex,thick,dotted] (z) at  (-1.4,2.75) {Z};
    \node[vertex,thick,fill=green] (c2) at  (-0.5,5) {$C_2$};
    \node[vertex,thick](x) at  (-0.5,3.8) {X};
    \node[vertex,thick,fill=Dandelion ] (t) at  (-0.5,1.4) {T};
    \node[vertex,thick ] (y) at  (-1.5,.2) {Y};
	\node[vertex,thick] (p) at  (0.8,1.4) {P};
	\node[vertex,thick] (q) at  (0.8,0.2) {Q};
	\node[vertex,thick] (b) at  (0.8,3.8) {B};
    \node[vertex,thick,dotted] (d) at  (0.8,5) {D};
    \node[vertex,thick,fill=Dandelion,dotted] (e) at  (2,5) {E};
	\node[vertex,thick,dotted] (i) at  (2,3.8) {I};
    \node[vertex,thick,dotted] (f) at  (0.8,6.2) {F};
    \node[vertex,thick,dotted] (g) at  (2,6.2) {G};
    \node[vertex,thick,dotted] (h) at  (3.2,6.2) {H};
	\draw[edge,dashed] (k) to (l);
	\draw[edge,dashed] (k) to (m);
	\draw[edge,dashed] (k) to (j);
	\draw[edge,dashed] (m) to (n);
	\draw[edge,dashed] (j) to (n);
	\draw[edge,dashed] (l) to (m);
	\draw[edge,fill = Cyan] (u) to (c1);
	\draw[edge,fill = Cyan] (u) to (c2);
	\draw[edge] (c1) to (y);
	\draw[edge] (c2) to (x);
	\draw[edge] (x) to (t);
	\draw[edge] (t) to (y);
	\draw[edge] (t) to (p);
	\draw[edge,] (p) to (q);
	\draw[edge,dashed] (d) to (b);
	\draw[edge] (c2) to (b);    
	\draw[edge,dashed] (e) to (b);
	\draw[edge,dashed] (e) to (i);
	\draw[edge,dashed] (f) to (d);
	\draw[edge,dashed] (g) to (e);
	\draw[edge,dashed] (e) to (i);
	\draw[edge,dashed] (h) to (e);
	\draw[edge,dashed] (j) to (x);
	\draw[edge,dashed] (c1) to (z);
	\draw[edge,dashed] (z) to (t);
    \draw[edge,dashed] (b) to (p);
\end{tikzpicture}
 \caption{\footnotesize{Ground truth synthetic graph: The grey node \textcolor{gray!40}{U} represents the unknown confounder connecting the context variables filled in green \textcolor{green}{$C_1$} \& \textcolor{green}{$C_2$}. The nodes in orange, like \textcolor{Dandelion}{E}, \textcolor{Dandelion}{T}, represent possible target variables. The central structure represented by solid lines is common to all graphs. The dashed lines and nodes are used to introduce variability in the structure across domains}.}\label{fig:Graph}
    \end{minipage}
\end{figure}

\section{Experimental Results}\label{sec:results}
Experimental results, over various settings as discussed in Section~\ref{sec:eval}, have been shown in Figure~\ref{fig:3},\ref{fig:4},\ref{fig:5},\ref{fig:6},\ref{fig:7},\ref{fig:8}, and \ref{fig:9}, the remaining results are in Appendix \ref{sec:moreresults}.
\subsection{Synthetic Dataset}
\vspace{-.2cm}
\rctl~outperforms (in some cases a comparable performance) other approaches in all environment settings, considered on the synthetic dataset. In Figure~\ref{fig:3}-\ref{fig:9}, we report the most complicated scenarios involving severe domain shifts, extreme sample sizes, and multiple ground truth models; results for other scenarios can be found in the supplementary material. \rctl~is, overall, as good as or even better than the state-of-the-art feature selection and domain adaptation algorithms listed in section \ref{sec:eval}. 

\textbf{Key highlights.} (1) \textit{Scalability}:~\ess, the closest work  to~\rctl, employs CI test based feature selection by conducting an  exhaustive search over the entire feature set, which increases the time required for subset generation exponentially, so cannot scale on datasets beyond ten variables. We used~\ess~for computation on a 12 variable synthetic dataset for 72 hours, and it crashed without generating any results.~\rctl~drastically reduces the time required for subset search as it searches only those variables in the locality of the target variable. Note that the target could be a subset of variables
rather than a single variable; parallel computation can speed up Markov blanket
recovery because Markov blankets of different nodes can be
learned independently \cite{Scutari17}. 
This makes ~\rctl~scalable to high-dimensional data and has potential applications in big data (as we will show in a real dataset with 400k variables in Section \ref{sec:cancer}).\\
(2) \textit{Robustness to Conditional Independence Tests}: In cases where CI tests have to be estimated from data, mistakes occur in keeping or removing members from the estimated separating sets. Erroneous CI
tests' primary  source are large condition sets in high-dimensional low sample size scenarios ~\cite{cheng1997learning}. In such cases, the resulting changes in the CI tests can lead to different separating sets. However, $\rctl$~based on Markov blanket discovery only uses a small fraction of variables in the vicinity of the target, increasing the reliance on CI tests, making the algorithm robust in practice. Our experimentation confirms that for most environment settings, where we know the ground truth causal graph, \rctl~can find the correct invariant features (e.g., see Table \ref{table:robust1} and \ref{table:robust2} in the appendix).\\
(3) \textit{Quantitative analysis of results on Gaussian settings}: Although we report only the most interesting scenarios in Figures~\ref{fig:3}-\ref{fig:9}, we notice that for Gaussian settings (see  Figure~\ref{fig:8} (b)), the error rates of many of these feature selection algorithms are slightly lower than~\rctl. Considering total bias, as discussed in section \ref{sec:method}, these approaches enjoy higher predictivity  due to lower incomplete information bias. For example, in Figure~\ref{fig:8} (b), Adaboost has a slightly lower MSE than~\rctl~for the specified setting. We perform t-tests and F-tests to confirm our observation
from the results for all such scenarios. In these tests, we found that the difference in error rates is almost insignificant (Figures \ref{ttest1} - \ref{fig:targetM}).\\
\noindent(4) \textit{Increase in Error on small sample size}: We do see a slight increase in the error on small sample size setting for~\rctl~(see  Figure~\ref{fig:9} (b)). Further investigation reveals that, since we do not utilize the underlying ground truth graph structure, the lack of data is straining the Markov blanket algorithms, causing them sometimes incorrectly to learn the Markov blanket of the target variable. For example, for the setting in Figure~\ref{fig:9} (b) using Figure~\ref{fig:9} (a) as the ground truth, we found that the Markov blanket of the target variable $T$ using IAMB is $\{P, X, C_1, G\}$. In this case, the Markov blanket is wrong, but the separating set learned $\{P,X\}$ was satisfactory and causally invariant. \\

\newgeometry{left=2cm,bottom=2cm,top = 2cm,right=2cm}
\twocolumn
\begin{figure}[!ht]
\centering%
\captionsetup[subfigure]{font=footnotesize,justification=centering}
    \subcaptionbox{}[.45\linewidth]{%
	\includegraphics[scale=.225]{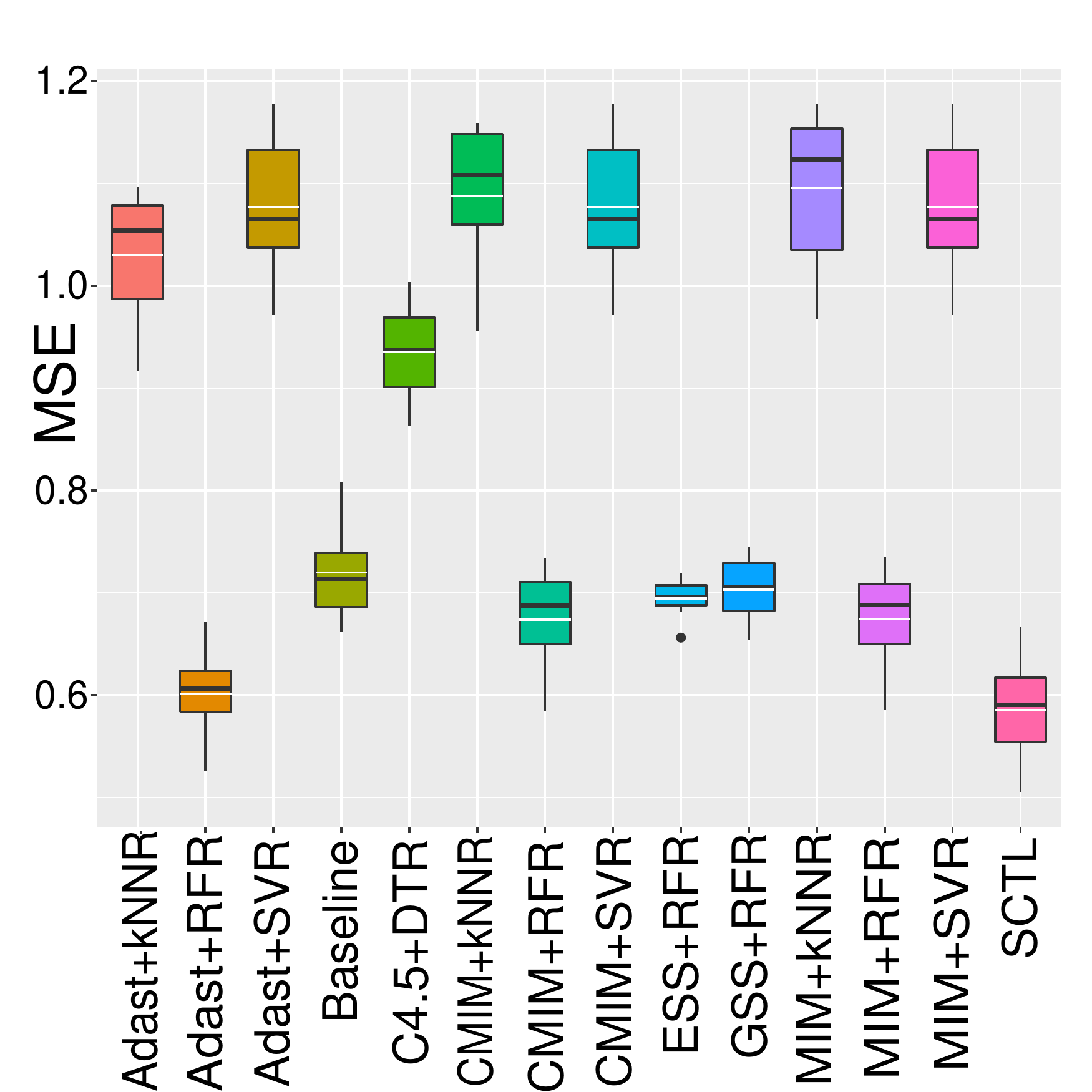}}%
	\subcaptionbox{}[.33\textwidth]{%
	\includegraphics[scale=.225]{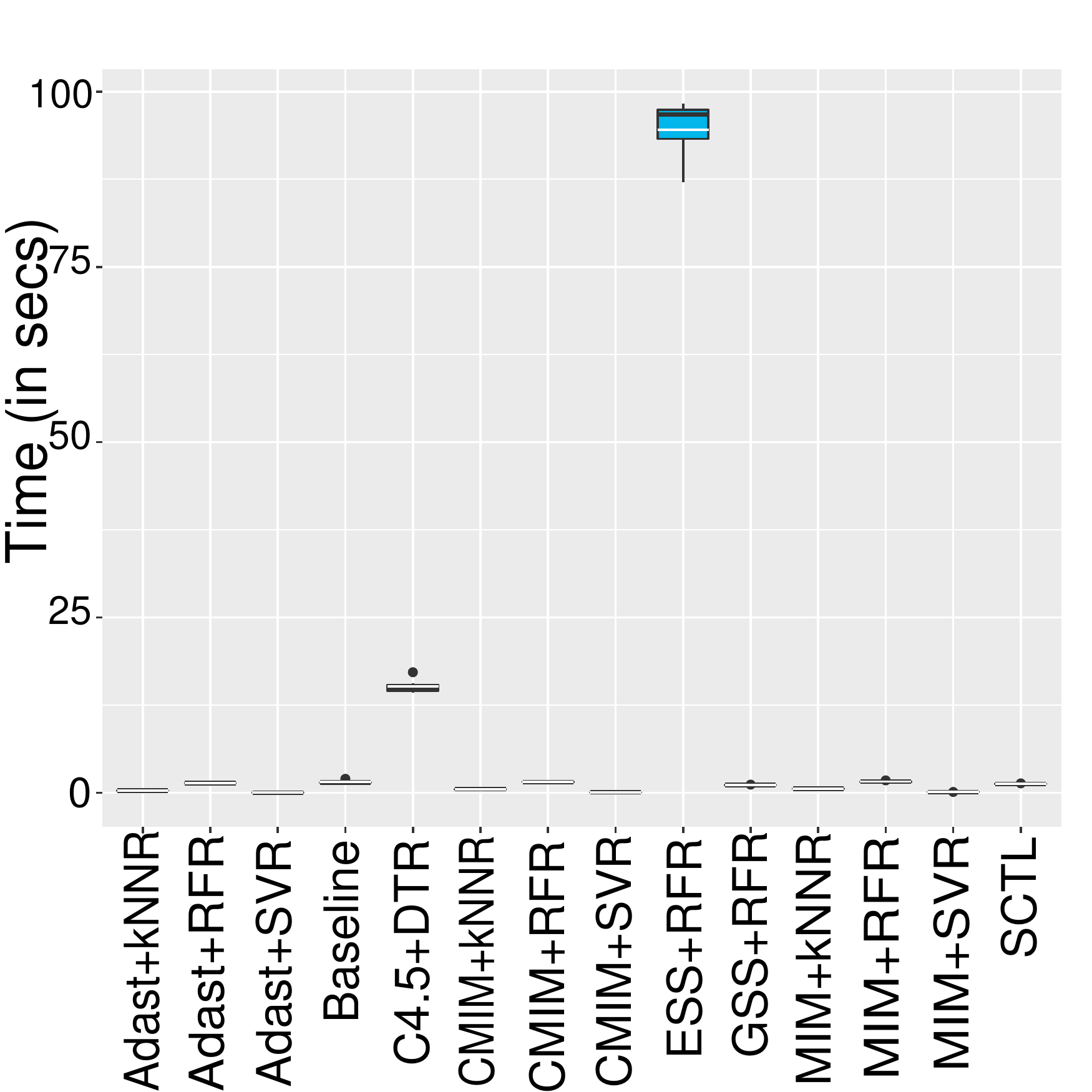}}\vspace{-.3cm}
    \caption{\footnotesize{(a) prediction error, (b) runtime. Data shift: \textcolor{green}{$C_1$} \& \textcolor{green}{$C_2$}, target variable: \textcolor{orange}{$T$}, sample size = 1000, data type: discrete, Ground truth: Figure \ref{fig:Graph} with solid lines.}}
    \label{fig:3}
\centering%
\vspace{-.3cm}\captionsetup[subfigure]{font=footnotesize,justification=centering}
\subcaptionbox{}[.45\linewidth]{%
 \begin{tikzpicture}[transform shape,scale=.45]
	\tikzset{vertex/.style = {shape=circle,align=center,draw=black, fill=white}}
\tikzset{edge/.style = {->,> = latex',thick}}
	\node[vertex,thick] (l) at  (-5.3,5) {L};
	\node[vertex,thick] (k) at  (-3.7,5) {K};
	\node[vertex,thick] (j) at  (-3.7,3.8) {J};
	\node[vertex,thick](m) at  (-5.3,3.8) {M};
	\node[vertex,thick] (n) at  (-4.5,2.6) {N};
    \node[vertex,thick,fill=green] (c1) at  (-2.5,5) {c1};
    \node[vertex,thick,fill= gray!40,dashed] (u) at  (-1.5,6.2) {U};
    \node[vertex,thick] (c2) at  (-0.5,5) {c2};
    \node[vertex,thick](x) at  (-0.5,3.8) {X};
    \node[vertex,thick,fill=Dandelion ] (t) at  (-0.5,1.4) {T};
    \node[vertex,thick ] (y) at  (-1.5,.2) {Y};
	\node[vertex,thick] (p) at  (0.8,1.4) {P};
	\node[vertex,thick] (q) at  (0.8,0.2) {Q};
	\node[vertex,thick] (b) at  (0.8,3.8) {B};
    \node[vertex,thick] (d) at  (0.8,5) {D};
    \node[vertex,thick] (e) at  (2,5) {E};
	\node[vertex,thick] (i) at  (2,3.8) {I};
    \node[vertex,thick] (f) at  (-0.5,6.2) {F};
    \node[vertex,thick] (g) at  (.8,6.2) {G};
    \node[vertex,thick] (h) at  (2,6.2) {H};
	\draw[edge] (k) to (l);
	\draw[edge] (k) to (j);
	\draw[edge] (m) to (n);
	\draw[edge] (j) to (n);
	\draw[edge] (l) to (m);
	\draw[edge,fill = Cyan] (u) to (c1);
	\draw[edge,fill = Cyan] (u) to (c2);
	\draw[edge] (c1) to (y);
	\draw[edge] (c2) to (x);
	\draw[edge] (x) to (t);
	\draw[edge] (t) to (y);
	\draw[edge] (t) to (p);
	\draw[edge,] (p) to (q);
	\draw[edge] (d) to (b);
	\draw[edge] (c2) to (b);    
	\draw[edge] (e) to (b);
	\draw[edge] (e) to (i);
	\draw[edge] (f) to (d);
	\draw[edge] (g) to (e);
	\draw[edge] (e) to (i);
	\draw[edge] (h) to (e);
	\draw[edge] (j) to (x);

\end{tikzpicture}
}%
	\subcaptionbox{}[.33\textwidth]{%
	\includegraphics[scale=.225]{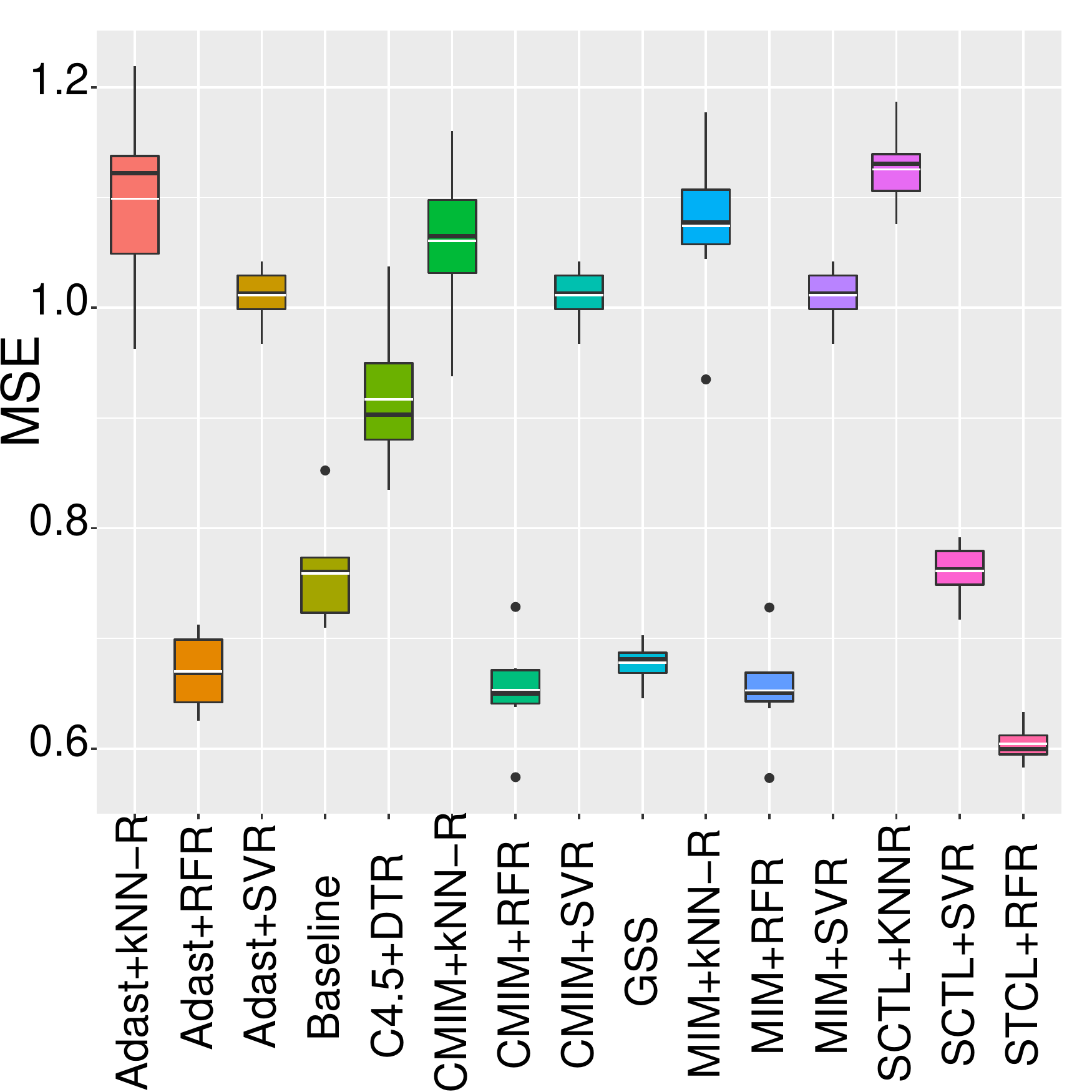}}%
     \vspace{-.3cm}\caption{\footnotesize{(a) the ground truth graph, (b) prediction error. Data shift: \textcolor{green}{$C_1$}, target variable: \textcolor{orange}{$T$}, sample size = 1000, and data type: discrete.}}
     \label{fig:4}
\centering%
\vspace{-.1cm}\captionsetup[subfigure]{font=footnotesize,justification=centering}
\subcaptionbox{}[.45\linewidth]{%
	    \begin{tikzpicture}[transform shape,scale=.45]
	\tikzset{vertex/.style = {shape=circle,align=center,draw=black, fill=white}}
\tikzset{edge/.style = {->,> = latex',thick}}
	\node[vertex,thick] (l) at  (-5.3,5) {L};
	\node[vertex,thick] (k) at  (-3.7,5) {K};
	\node[vertex,thick] (j) at  (-3.7,3.8) {J};
	\node[vertex,thick](m) at  (-5.3,3.8) {M};
	\node[vertex,thick] (n) at  (-4.5,2.6) {N};
    \node[vertex,thick,fill=green] (c1) at  (-2.5,5) {c1};
    \node[vertex,thick,fill= gray!40,dashed] (u) at  (-1.5,6.2) {U};
    \node[vertex,thick,fill=green] (c2) at  (-0.5,5) {c2};
    \node[vertex,thick](x) at  (-0.5,3.8) {X};
    \node[vertex,thick,fill=Dandelion ] (t) at  (-0.5,1.4) {T};
    \node[vertex,thick ] (y) at  (-1.5,.2) {Y};
	\node[vertex,thick] (p) at  (0.8,1.4) {P};
	\node[vertex,thick] (q) at  (0.8,0.2) {Q};
	\node[vertex,thick] (b) at  (0.8,3.8) {B};
    \node[vertex,thick] (d) at  (0.8,5) {D};
    \node[vertex,thick] (e) at  (2,5) {E};
	\node[vertex,thick] (i) at  (2,3.8) {I};
    \node[vertex,thick] (f) at  (-0.5,6.2) {F};
    \node[vertex,thick] (g) at  (.8,6.2) {G};
    \node[vertex,thick] (h) at  (2,6.2) {H};
	\draw[edge] (k) to (l);
	\draw[edge] (k) to (j);
	\draw[edge] (m) to (n);
	\draw[edge] (j) to (n);
	\draw[edge] (l) to (m);
	\draw[edge,fill = Cyan] (u) to (c1);
	\draw[edge,fill = Cyan] (u) to (c2);
	\draw[edge] (c1) to (y);
	\draw[edge] (c2) to (x);
	\draw[edge] (x) to (t);
	\draw[edge] (t) to (y);
	\draw[edge] (t) to (p);
	\draw[edge,] (p) to (q);
	\draw[edge] (d) to (b);
	\draw[edge] (c2) to (b);    
	\draw[edge] (e) to (b);
	\draw[edge] (e) to (i);
	\draw[edge] (f) to (d);
	\draw[edge] (g) to (e);
	\draw[edge] (e) to (i);
	\draw[edge] (h) to (e);
	\draw[edge] (j) to (x);
\end{tikzpicture}
}%
\subcaptionbox{}[.33\textwidth]{%
	\includegraphics[scale=.225]{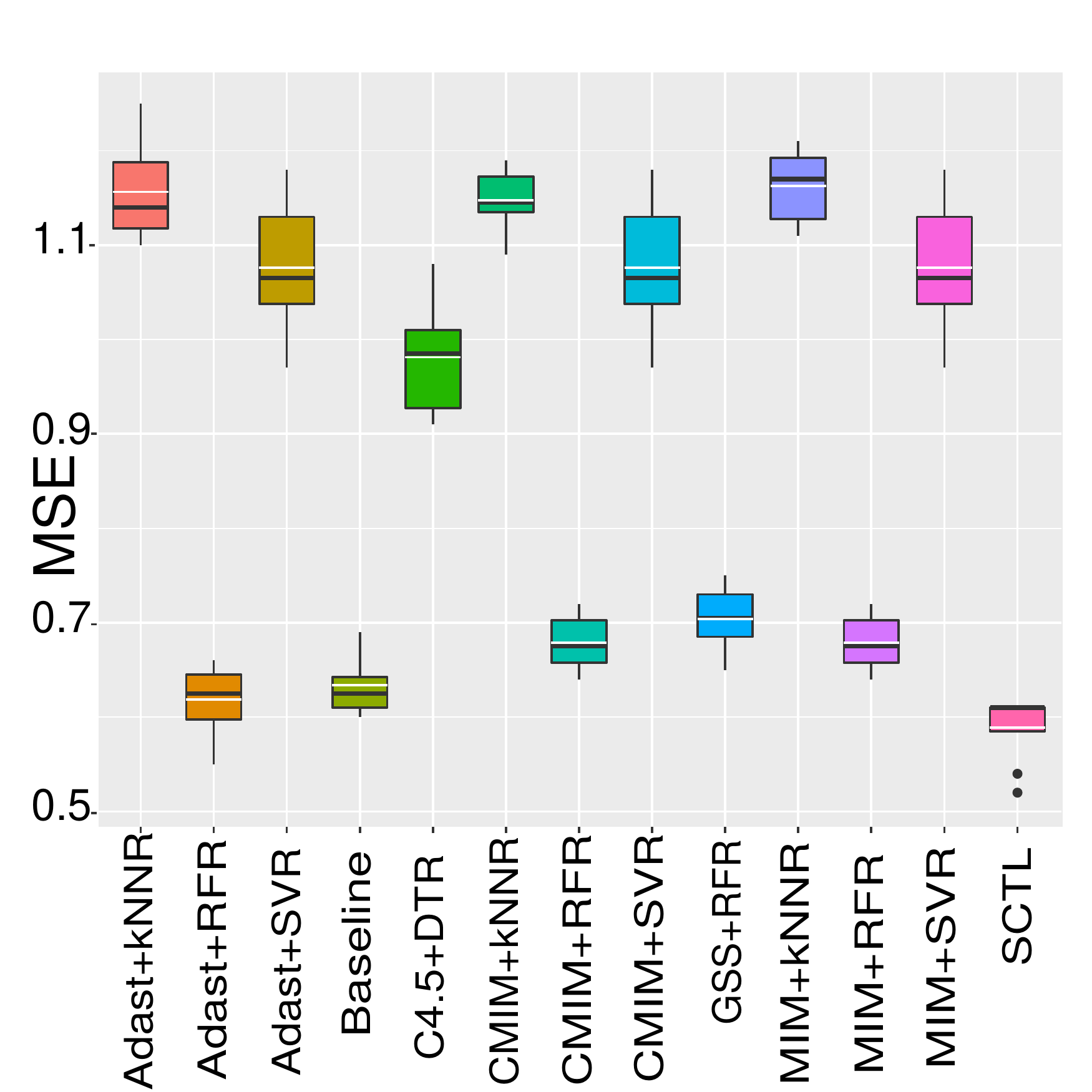}}\vspace{-.3cm}%
    \caption{\footnotesize{(a) the ground truth graph, (b) prediction error. Data shift: \textcolor{green}{$C_1$} \& \textcolor{green}{$C_2$}, target variable: \textcolor{orange}{$T$}, sample size = 1000, and data type: discrete.}}
    \label{fig:5}
\centering
\captionsetup[subfigure]{font=footnotesize,justification=centering}
\subcaptionbox{}[.45\linewidth]{%
	    \begin{tikzpicture}[transform shape,scale=.45]
	\tikzset{vertex/.style = {shape=circle,align=center,draw=black, fill=white}}
\tikzset{edge/.style = {->,> = latex',thick}}
	\node[vertex,thick] (l) at  (-5.3,5) {L};
	\node[vertex,thick] (k) at  (-3.7,5) {K};
	\node[vertex,thick] (j) at  (-3.7,3.8) {J};
	\node[vertex,thick,fill = Dandelion](m) at  (-5.3,3.8) {M};
	\node[vertex,thick] (n) at  (-4.5,2.6) {N};
    \node[vertex,thick,fill=green] (c1) at  (-2.5,5) {c1};
    \node[vertex,thick,fill= gray!40,dashed] (u) at  (-1.5,6.2) {U};
    \node[vertex,thick] (c2) at  (-0.5,5) {c2};
    \node[vertex,thick](x) at  (-0.5,3.8) {X};
    \node[vertex,thick] (t) at  (-0.5,1.4) {T};
    \node[vertex,thick ] (y) at  (-1.5,.2) {Y};
	\node[vertex,thick] (p) at  (0.8,1.4) {P};
	\node[vertex,thick] (q) at  (0.8,0.2) {Q};
	\node[vertex,thick] (b) at  (0.8,3.8) {B};
    \node[vertex,thick] (d) at  (0.8,5) {D};
    \node[vertex,thick] (e) at  (2,5) {E};
	\node[vertex,thick] (i) at  (2,3.8) {I};
    \node[vertex,thick] (f) at  (-0.5,6.2) {F};
    \node[vertex,thick] (g) at  (.8,6.2) {G};
    \node[vertex,thick] (h) at  (2,6.2) {H};
	\draw[edge] (k) to (l);
	\draw[edge] (k) to (j);
	\draw[edge] (m) to (n);
	\draw[edge] (j) to (n);
	\draw[edge] (l) to (m);
	\draw[edge,fill = Cyan] (u) to (c1);
	\draw[edge,fill = Cyan] (u) to (c2);
	\draw[edge] (c1) to (y);
	\draw[edge] (c2) to (x);
	\draw[edge] (x) to (t);
	\draw[edge] (t) to (y);
	\draw[edge] (t) to (p);
	\draw[edge,] (p) to (q);
	\draw[edge] (d) to (b);
	\draw[edge] (c2) to (b);    
	\draw[edge] (e) to (b);
	\draw[edge] (e) to (i);
	\draw[edge] (f) to (d);
	\draw[edge] (g) to (e);
	\draw[edge] (e) to (i);
	\draw[edge] (h) to (e);
	\draw[edge] (j) to (x);
    \draw[edge] (b) to (p);
\end{tikzpicture}
}%
	\subcaptionbox{}[.33\textwidth]{%
	\includegraphics[scale=.225]{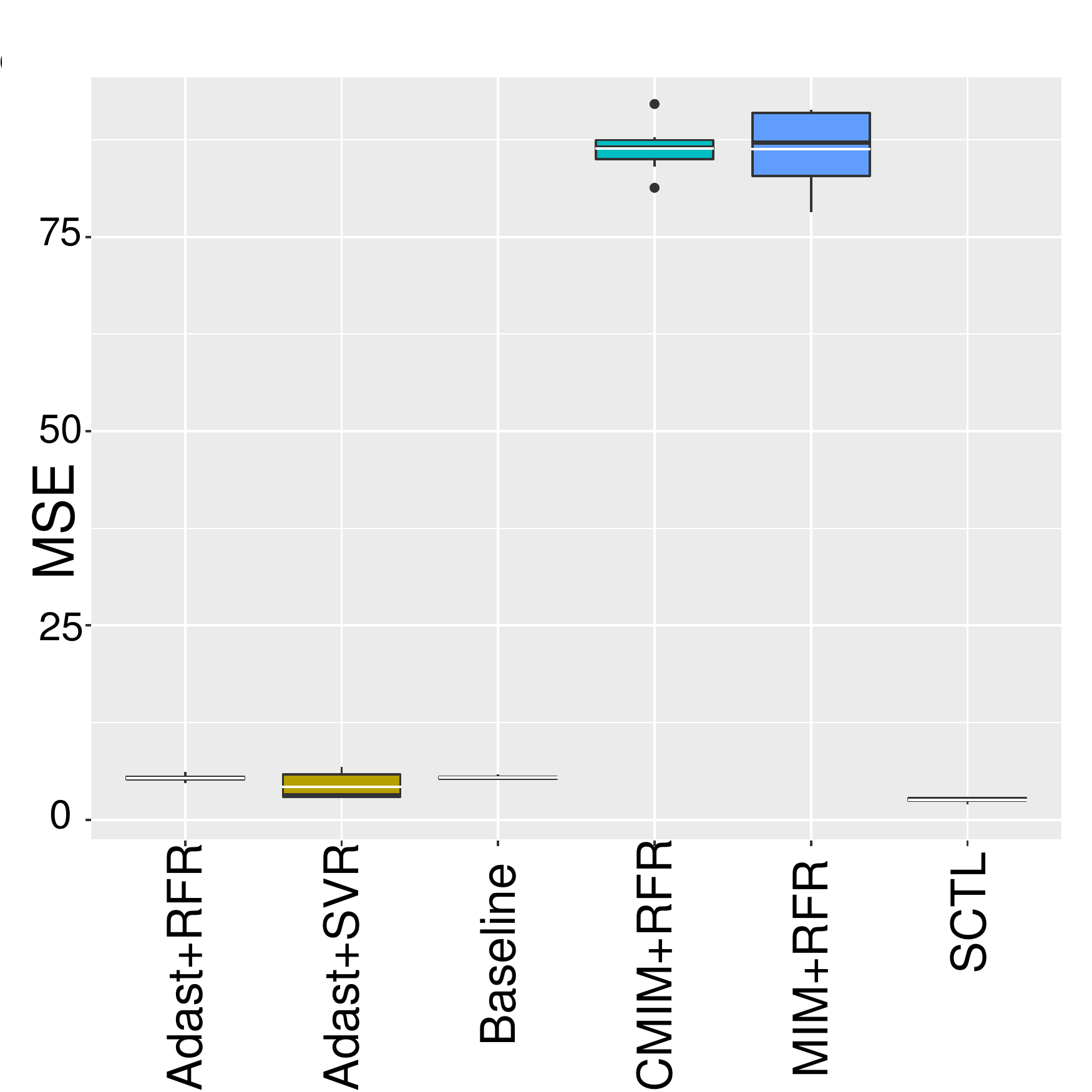}}\vspace{-.3cm}%
\caption{\footnotesize{(a) the ground truth graph, (b) prediction error. Data shift: \textcolor{green}{$C_1$}, target variable: \textcolor{orange}{$M$}, sample size = 1000, and data type: discrete.}}
    \label{fig:6}
    \end{figure}
\begin{figure}
\centering
\captionsetup[subfigure]{font=footnotesize,justification=centering}
        \vspace{-3cm}
        \subcaptionbox{}[.45\linewidth]{%
	    \begin{tikzpicture}[transform shape,scale=.45]
	\tikzset{vertex/.style = {shape=circle,align=center,draw=black, fill=white}}
\tikzset{edge/.style = {->,> = latex',thick}}
	\node[vertex,thick] (l) at  (-5.3,5) {L};
	\node[vertex,thick] (k) at  (-3.7,5) {K};
	\node[vertex,thick] (j) at  (-3.7,3.8) {J};
	\node[vertex,thick](m) at  (-5.3,3.8) {M};
	\node[vertex,thick] (n) at  (-4.5,2.6) {N};
    \node[vertex,thick,fill=green] (c1) at  (-2.5,5) {c1};
    \node[vertex,thick,fill= gray!40,dashed] (u) at  (-1.5,6.2) {U};
    \node[vertex,thick,fill = green] (c2) at  (-0.5,5) {c2};
    \node[vertex,thick](x) at  (-0.5,3.8) {X};
    \node[vertex,thick,fill=Dandelion ] (t) at  (-0.5,1.4) {T};
    \node[vertex,thick ] (y) at  (-1.5,.2) {Y};
	\node[vertex,thick] (p) at  (0.8,1.4) {P};
	\node[vertex,thick] (q) at  (0.8,0.2) {Q};
	\node[vertex,thick] (b) at  (0.8,3.8) {B};
    \node[vertex,thick] (d) at  (0.8,5) {D};
    \node[vertex,thick] (e) at  (2,5) {E};
	\node[vertex,thick] (i) at  (2,3.8) {I};
    \node[vertex,thick] (f) at  (-.5,6.2) {F};
    \node[vertex,thick] (g) at  (.8,6.2) {G};
    \node[vertex,thick] (h) at  (2,6.2) {H};
	\draw[edge] (k) to (l);
	\draw[edge] (k) to (m);
	\draw[edge] (k) to (j);
	\draw[edge] (m) to (n);
	\draw[edge] (j) to (n);
	\draw[edge] (l) to (m);
	\draw[edge,fill = Cyan] (u) to (c1);
	\draw[edge,fill = Cyan] (u) to (c2);
	\draw[edge] (c1) to (y);
	\draw[edge] (c2) to (x);
	\draw[edge] (x) to (t);
	\draw[edge] (t) to (y);
	\draw[edge] (t) to (p);
	\draw[edge] (p) to (q);
	\draw[edge] (d) to (b);
	\draw[edge] (c2) to (b);    
	\draw[edge] (e) to (i);
	\draw[edge] (f) to (d);
	\draw[edge] (g) to (e);
	\draw[edge] (e) to (i);
	\draw[edge] (h) to (e);
	\draw[edge] (j) to (x);
    \draw[edge] (b) to (p);
\end{tikzpicture}
}%
	\subcaptionbox{}[.33\textwidth]{%
	\includegraphics[scale=.225]{images/plots/MSE_C1_C2_graph1_discrete_.pdf}}\vspace{-.3cm}%
    \caption{\footnotesize{(a) the ground truth graph, (b) prediction error. Data shift: \textcolor{green}{$C_1$} \& \textcolor{green}{$C_2$}, target variable: \textcolor{orange}{$T$}, sample size = 1000, and data type: discrete.}}
    \label{fig:7}
\centering
\captionsetup[subfigure]{font=footnotesize,justification=centering}
\subcaptionbox{}[.45\linewidth]{%
	    \begin{tikzpicture}[transform shape,scale=.45]
	\tikzset{vertex/.style = {shape=circle,align=center,draw=black, fill=white}}
\tikzset{edge/.style = {->,> = latex',thick}}
	\node[vertex,thick] (l) at  (-5.3,5) {L};
	\node[vertex,thick] (k) at  (-3.7,5) {K};
	\node[vertex,thick] (j) at  (-3.7,3.8) {J};
	\node[vertex,thick](m) at  (-5.3,3.8) {M};
	\node[vertex,thick] (n) at  (-4.5,2.6) {N};
    \node[vertex,thick] (c1) at  (-2.5,5) {c1};
    \node[vertex,thick,fill= gray!40,dashed] (u) at  (-1.5,6.2) {U};
    \node[vertex,thick] (z) at  (-1.4,2.75) {Z};
    \node[vertex,thick] (c2) at  (-0.5,5) {c2};
    \node[vertex,thick](x) at  (-0.5,3.8) {X};
    \node[vertex,thick,fill=Dandelion ] (t) at  (-0.5,1.4) {T};
    \node[vertex,thick ] (y) at  (-1.5,.2) {Y};
	\node[vertex,thick] (p) at  (0.8,1.4) {P};
	\node[vertex,thick] (q) at  (0.8,0.2) {Q};
	\node[vertex,thick] (b) at  (0.8,3.8) {B};
    \node[vertex,thick] (d) at  (0.8,5) {D};
    \node[vertex,thick] (e) at  (2,5) {E};
	\node[vertex,thick] (i) at  (2,3.8) {I};
    \node[vertex,thick] (f) at  (-.5,6.2) {F};
    \node[vertex,thick] (g) at  (.8,6.2) {G};
    \node[vertex,thick] (h) at  (2,6.2) {H};
	\draw[edge] (k) to (l);
	\draw[edge] (k) to (j);
	\draw[edge] (m) to (n);
	\draw[edge] (j) to (n);
	\draw[edge] (l) to (m);
	\draw[edge,fill = Cyan] (u) to (c1);
	\draw[edge,fill = Cyan] (u) to (c2);
	\draw[edge] (c1) to (y);
	\draw[edge] (c2) to (x);
	\draw[edge] (x) to (t);
	\draw[edge] (t) to (y);
	\draw[edge] (t) to (p);
	\draw[edge] (p) to (q);
	\draw[edge] (d) to (b);
	\draw[edge] (c2) to (b);    
	\draw[edge] (e) to (b);
	\draw[edge] (e) to (i);
	\draw[edge] (f) to (d);
	\draw[edge] (g) to (e);
	\draw[edge] (e) to (i);
	\draw[edge] (h) to (e);
	\draw[edge] (j) to (x);
 	\draw[edge] (c1) to (z);
 	\draw[edge] (z) to (t);
\end{tikzpicture}
}%
	\subcaptionbox{}[.33\textwidth]{%
	\includegraphics[scale=.225]{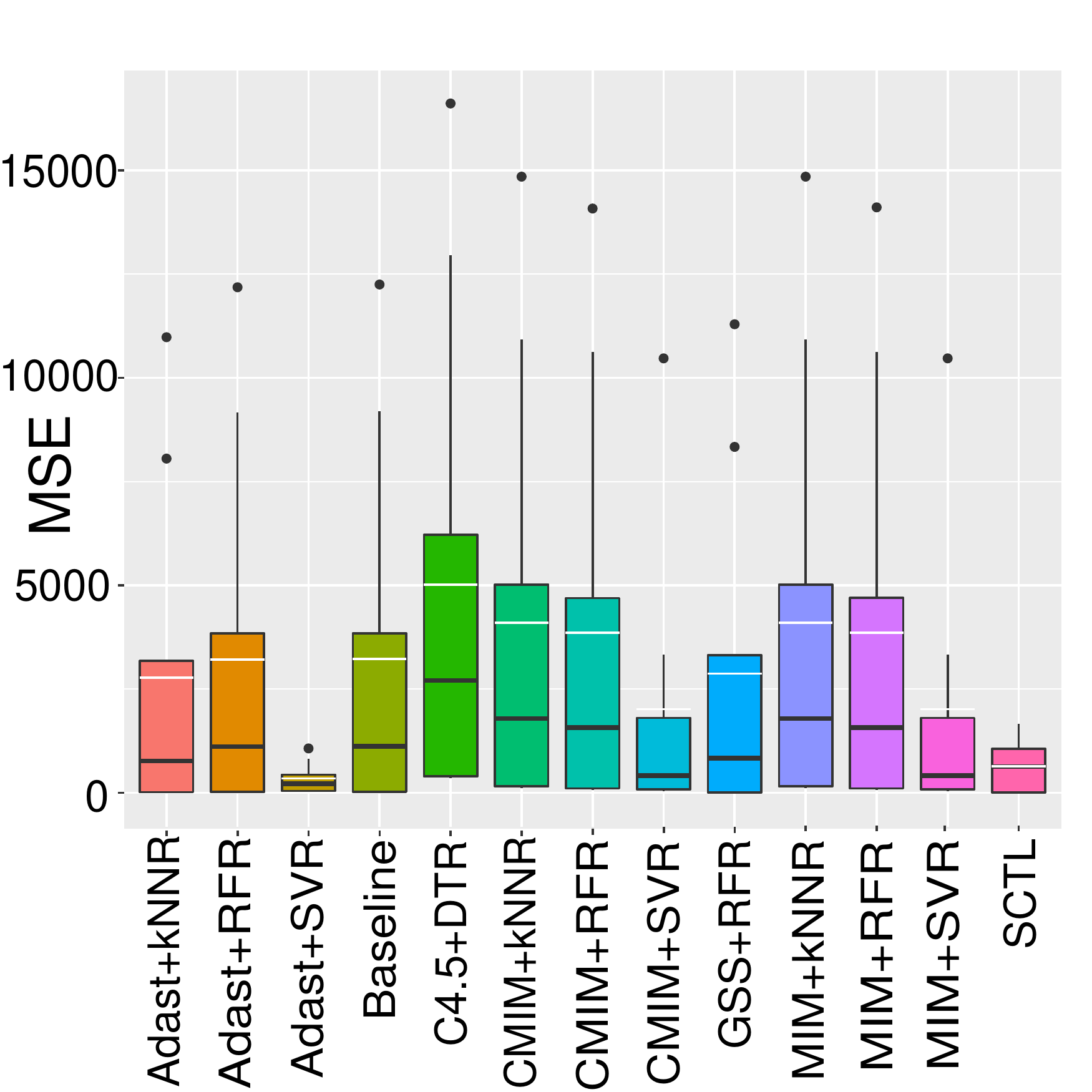}}\vspace{-.3cm}
    \caption{\footnotesize{(a) the ground truth graph, (b) prediction error. No data shift, target variable: \textcolor{orange}{$T$}, sample size = 1000, and data type: Gaussian.}}
    \label{fig:8}

\centering
\captionsetup[subfigure]{font=footnotesize,justification=centering}
\subcaptionbox{}[.45\linewidth]{%
	    \begin{tikzpicture}[transform shape,scale=.45]
	\tikzset{vertex/.style = {shape=circle,align=center,draw=black, fill=white}}
\tikzset{edge/.style = {->,> = latex',thick}}
	\node[vertex,thick] (l) at  (-5.3,5) {L};
	\node[vertex,thick] (k) at  (-3.7,5) {K};
	\node[vertex,thick] (j) at  (-3.7,3.8) {J};
	\node[vertex,thick](m) at  (-5.3,3.8) {M};
	\node[vertex,thick] (n) at  (-4.5,2.6) {N};
    \node[vertex,thick,fill=green] (c1) at  (-2.5,5) {c1};
    \node[vertex,thick,fill= gray!40,dashed] (u) at  (-1.5,6.2) {U};
    \node[vertex,thick,fill=green] (c2) at  (-0.5,5) {c2};
    \node[vertex,thick](x) at  (-0.5,3.8) {X};
    \node[vertex,thick,fill=Dandelion ] (t) at  (-0.5,1.4) {T};
    \node[vertex,thick ] (y) at  (-1.5,.2) {Y};
	\node[vertex,thick] (p) at  (0.8,1.4) {P};
	\node[vertex,thick] (q) at  (0.8,0.2) {Q};
	\node[vertex,thick] (b) at  (0.8,3.8) {B};
    \node[vertex,thick] (d) at  (0.8,5) {D};
    \node[vertex,thick] (e) at  (2,5) {E};
	\node[vertex,thick] (i) at  (2,3.8) {I};
    \node[vertex,thick] (f) at  (-.5,6.2) {F};
    \node[vertex,thick] (g) at  (.8,6.2) {G};
    \node[vertex,thick] (h) at  (2,6.2) {H};
	\draw[edge] (k) to (l);
	\draw[edge] (k) to (j);
	\draw[edge] (m) to (n);
	\draw[edge] (j) to (n);
	\draw[edge] (l) to (m);
	\draw[edge,fill = Cyan] (u) to (c1);
	\draw[edge,fill = Cyan] (u) to (c2);
	\draw[edge] (c1) to (y);
	\draw[edge] (c2) to (x);
	\draw[edge] (x) to (t);
	\draw[edge] (t) to (y);
	\draw[edge] (t) to (p);
	\draw[edge] (p) to (q);
	\draw[edge] (d) to (b);
	\draw[edge] (c2) to (b);    
	\draw[edge] (e) to (b);
	\draw[edge] (e) to (i);
	\draw[edge] (f) to (d);
	\draw[edge] (g) to (e);
	\draw[edge] (e) to (i);
	\draw[edge] (h) to (e);
	\draw[edge] (j) to (x);
\end{tikzpicture}
}%
	\subcaptionbox{}[.33\textwidth]{%
  	\includegraphics[scale=0.225]{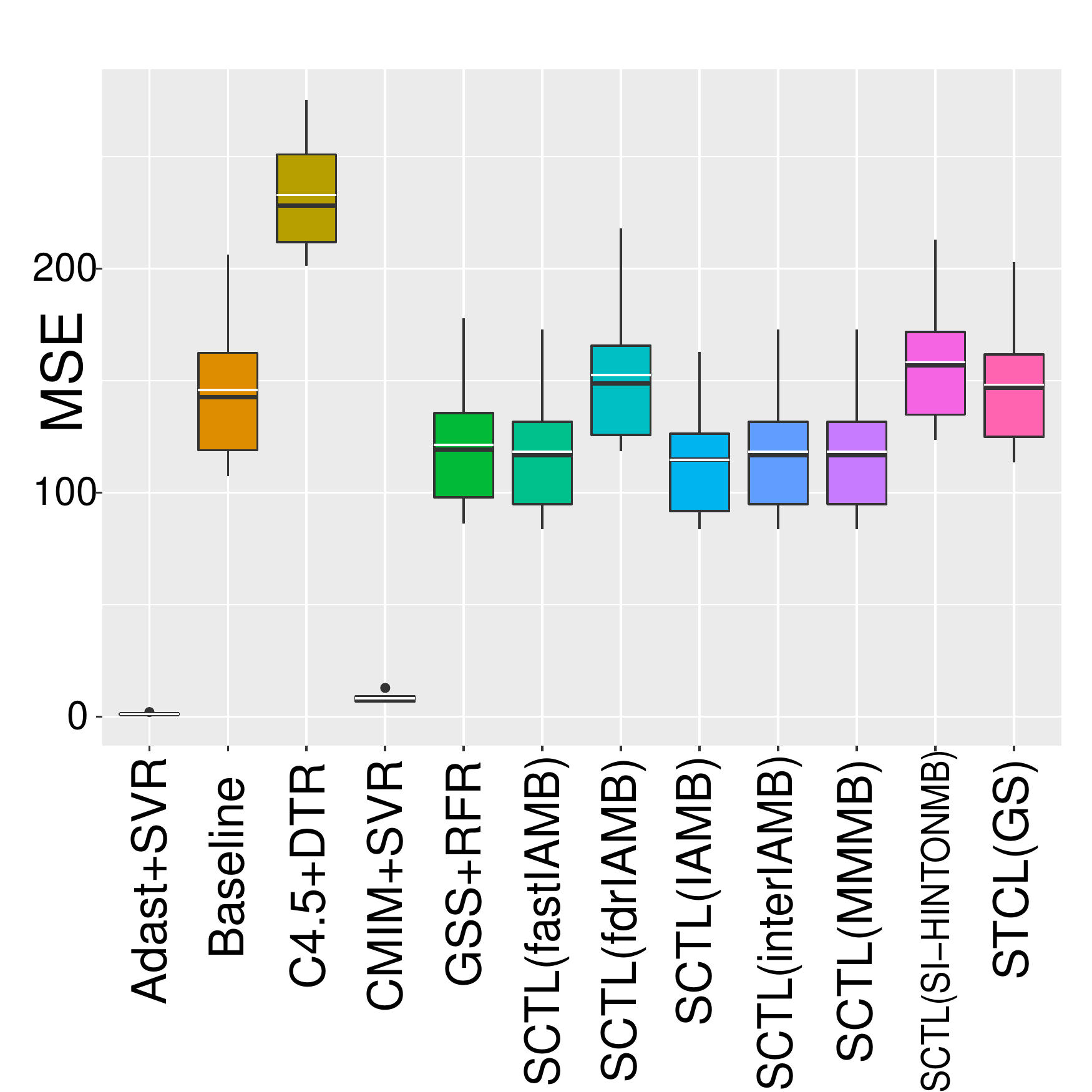}}\vspace{-.3cm}
	\caption{\footnotesize{(a) the ground truth graph, (b) prediction error. Data shift: \textcolor{green}{$C_1$} \& \textcolor{green}{$C_2$}, target variable: \textcolor{orange}{$T$}, sample size = 50, and data type:  Gaussian.}}
	\label{fig:9}
	\includegraphics[scale=0.225]{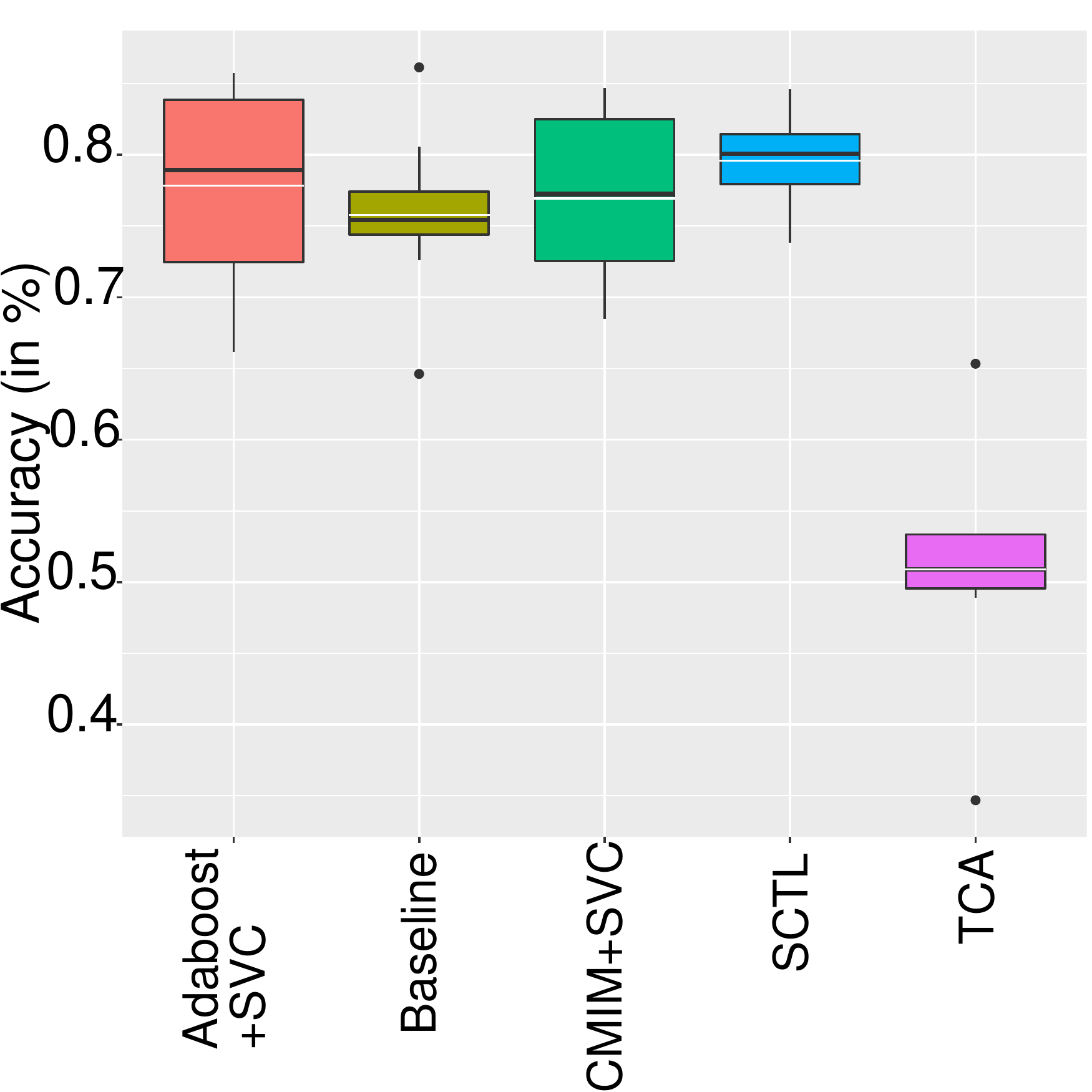}%
	\includegraphics[scale=0.225]{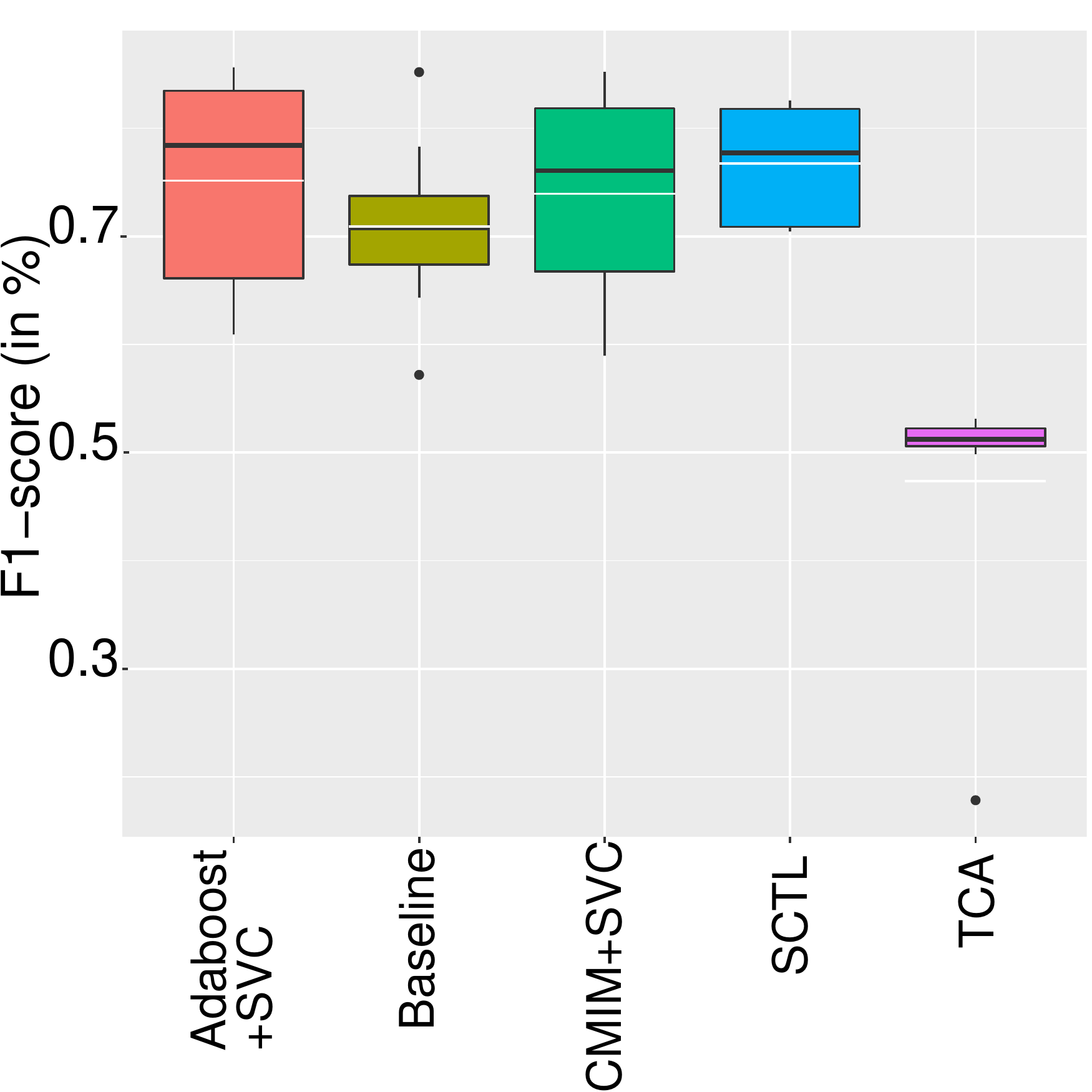}
	\caption{\footnotesize{Error rate comparison of feature selection approaches on the Diabetes dataset with \textit{Age Shift} (discussed in Section~\ref{sec:real_descp}).} }
	\label{fig:real_data}
\end{figure}
\restoregeometry
\onecolumn

(5)\label{key5} \textit{Choice of Markov Blanket Algorithms}: As shown in Figure~\ref{fig:9} (b), the Markov blanket algorithm used holds key significance. Their practical uses may show results different from the theoretical perception. For example, GS does not consider the ordering and the strength of the association of the candidate variable and the target variable $T$. On the other hand, IAMB  orders variables based on the strength of their association with the target variable $T$ first and then check their membership in the $\mb(T)$. Since choosing different p-values, the sample size of the data, and the maximum size of the conditioning sets have a big effect on the quality of learned $\mb(T)$, it is necessary to choose an appropriate Markov blanket approach depending on the given setting. For example, fdrIAMB \cite{Pena08Mb} is particularly suitable when dealing with more nodes than samples in Gaussian models.
\subsection{Real-world data: Diabetes Dataset}
Although we do not have the ground truth causal graph for the real world diabetes data~\cite{islam2020likelihood}, we assume it follows Assumption~\ref{assum:jci} \&~\ref{assum:cda}. We consider \textit{Age} as a context variable, because previous studies \cite{kirkman2012diabetes} have shown that older people are at higher risk of developing Type II diabetes. As this is a classification
problem, the commonly used metrics Accuracy and F1-score were used to measure the
algorithms' performance.
We compare \rctl~against Adaboost and CMIM as they showed the best performance in the synthetic case. 
Similar to synthetic data, we ran ~\ess~on the computer for 72 hours, after which it was crashed. The results in Figure~\ref{fig:real_data} show that although \rctl~uses fewer features for prediction, it provides higher accuracy and F1-score in average compared with other approaches. We observe a more significant variability for Adaboost+SVC and CMIM+SVC accuracy and F1-score compared to \rctl~as well as larger outliers. This can be considered an indication of the robustness of \rctl~in real-world scenarios. On the other hand, although there is a greater variability for \rctl~accuracy and F1-score compared to the baseline as well as larger outliers, \rctl~enjoys higher accuracy and F1-score in general. These observations  verify the importance of finding causally invariant features where data shift occurs.  Notice that the total sample size for both training and testing scenarios had 200 entries for each. This signifies that even for a moderate number of samples, the Markov blankets can learn the vicinity of the target well enough for giving low error rates in the target domain. We also compared \rctl~against Transfer Component Analysis (TCA), a state-of-the-art
algorithm in unsupervised domain adaptation that identifies transfer components
that remains invariant across domains. As shown in \cite{pan2010domain}, the performance of TCA is sensitive to the choice of kernel choice and hyperparameters of the kernel function. We, therefore,  performed hyperparameter tuning (details in Appendix) and reported the best result. We hypothesize that TCA cannot learn the transfer components correctly in this scenario and requires further investigations. 
\vspace{-.25cm}
\subsection{Real-world data: Cancer Dataset}
\label{sec:cancer}
\vspace{-.25cm}
To show the scalability of \rctl, we used the Colorectal Cancer Dataset. Similar to the diabetes case, we did not have the ground truth graph. For the same reasons, we compared \rctl~against Adaboost and CMIM using Accuracy and F1-score. For this experiment, we considered \textit{Gender} as the context variable. Firstly, this is a dataset with ~{400k} variables, making it impossible to scale under any circumstances using \cite{MagliacaneNIPS18}. 
\begin{wraptable}{r}{8cm}
\caption{Results for the cancer dataset (100k ranked features)} 
\centering 
\footnotesize\begin{tabular}{l c c} 
\hline\hline 
Methodology & Accuracy & F1-Score \\ [0.5ex] 
\hline 
Baseline &  91.25	& 90.94  \\
\rctl~(GS) (167 features)& 88.12	& 87.88  \\
\rctl~(IAMB) (80 features)& 96.25 &	96.22  \\
CMIM + SVC & 90.05 & 89.79	  \\
Adaboost & \textbf{96.87} &	\textbf{97.83}   \\ [1ex] 
\hline 
\end{tabular}
\label{table:cancer100k} 
\end{wraptable}
We conducted a sensitivity analysis to evaluate the robustness of algorithms in this high-dimensional and small sample size setting. For this purpose, we extracted 1000, 10k, and 100k  features using an Extra-Trees classifier \cite{pedregosa2011scikit} to rank features based on their importance. In addition to the sensitivity analysis, by this intervention to the dataset, we made the setting in favor
of feature selection algorithms that we compare by making the problem less dimensional and sorting the features based on their importance using the complete dataset
(details in Appendix \ref{sec:moreresults}). The results for each are shown in Tables  \ref{table:cancer100k}, \ref{table:cancer1000}, and \ref{table:cancer10k}. We found~\rctl~to be either as good as or better than other approaches except in comparison to Adaboost, which enjoys slightly better performance
than \rctl~in (almost) all settings. We conjecture that this is due to a misalignment between theory and dynamics of the domain. It also is well-known that Adaboost  is a successful classifier that takes advantage of boosting \cite{wyner2017explaining}. Although non-causal feature selection methods such as Adaboost are often used in
practice, they cannot be interpreted causally even
when they achieve excellent predictivity. Note that the size of the returned feature set for \rctl(IAMB)~in Table \ref{table:cancer100k} is considerably smaller than for \rctl(GS); the performance of \rctl(IAMB)~is better than \rctl(GS). This indicates  that IAMB is a better choice than GS in high-dimensional settings, as discussed in key highlight (5) for synthetic experiments.

\section*{Conclusion}
In this paper, we proposed a new algorithm, called \textbf{S}calable \textbf{C}ausal \textbf{T}ransfer \textbf{L}earning (\rctl), that identifies causal invariance in the presence of covariate shift and showed that it scales to high-dimensional and robust in low sample size settings in both synthetic and real-world scenarios. Weakening  Assumption \ref{assum:cda}(b) and relaxing Assumption \ref{assum:cda}(c) are interesting directions for future work. 

\section*{Acknowledgements} This work has been supported in part by NASA (Awards 80NSSC20K1720 and 521418-SC) and NSF (Awards 2007202 and 2107463). We are grateful to all who provided feedback on this work, including anonymous reviewers of NeurIPS WHY-21.

\bibliographystyle{plain}
\bibliography{references.bib}
\clearpage
\appendix

\section{Basic Definitions and Concepts}\label{sec:defs}
Assume that $G=(V, E)$ is a directed graph, where $ V = \{ v_1,v_2,v_3 \cdots v_n  \}$ is the set of nodes (variables), $n  \geq 1$, and $E$ is the set of directed or bidirected edges. We say $v_i$ is a \textit{parent} of $v_j$ and $v_j$ is a \textit{child} of $v_i$ if $v_i  \rightarrow v_j $ is an edge in $G$.  We denote the set of parents and  children of a variable $v$ by $ \pa(v)$ and $\ch(v)$, respectively. Any bidirected edge $ v_i \leftrightarrow v_j $ means that there exists a node $v_k\not\in V$ as a hidden confounder, such that $v_i\gets v_k\to v_j$. If $G$ is acyclic, then $G$ is an \textit{Acyclic Directed Mixed Graph} (ADMG). Formally, $\nbr(T)=\pa(T)\cup \ch(T)\cup\{v\in V|v\leftrightarrow T\textrm{ is a bidirected edge in } G\}$ refers to the \textit{neighbours} of $T$. We define \textit{spouses} of $v$ as $\spouse(v)=\{u\in V|\exists w\in V \textrm{ s.t. }u\to w\gets v \textrm{ in }G\}$. We define \textit{Markov blanket} of node $T$ as $\mb(T)=\pa(T) \cup \ch(T)\cup\spouse(T)$ when $G$ is a directed acyclic graph (DAG) and the set of children, parents, and spouses of $T$, and vertices connected with $T$ or
children of $T$ by a bidirected path (i.e., only with edges $\leftrightarrow$) and their respective parents is the
Markov blanket of $T$ when $G$ is an ADMG. 

A path of length $n$ from $x$ to $y$ in an ADMG $G=(V,E)$ is a sequence $x=a_0,\dots , a_n=y$ of distinct vertices such that $(a_i,a_{i+1})\in E$, for all $i=1,\dots ,n$. A vertex $\alpha$ is said to be an \emph{ancestor} of a vertex $\beta$ if either there is a directed path $\alpha \to \dots \to \beta$ from $\alpha$ to $\beta$, or $\alpha=\beta$. We apply this definition to sets: $\an(X) = \{\alpha | \alpha \textrm{ is an ancestor of } \beta \textrm{ for some } \beta \in X\}$. 
\begin{definition}\label{RS}
	A nonendpoint vertex $\zeta$ on a path is a \emph{collider} on the path if the edges preceding and succeeding $\zeta$ on the path have an arrowhead at $\zeta$, that is, $\to \zeta \gets, or \leftrightarrow \zeta \leftrightarrow, or\leftrightarrow \zeta \gets, or\to \zeta \leftrightarrow$. A nonendpoint vertex $\zeta$ on a path which is not a collider is a noncollider on the path. A path between vertices $\alpha$ and $\beta$ in an ADMG G is said to be \emph{m}-connecting given a set Z (possibly empty), with $\alpha, \beta \notin Z$, if: 
	
	\noindent (i) every noncollider on the path is not in Z, and 
	
	\noindent (ii) every collider on the path is in $\an_G(Z)$.
	
	If there is no path m-connecting $\alpha$ and $\beta$ given Z, then $\alpha$ and $\beta$ are said to be \emph{$m$-separated} given Z. Sets X and Y are \emph{m}-separated given Z, if for every pair $\alpha, \beta$, with $\alpha\in X$ and $\beta \in Y$, $\alpha$ and $\beta$ are \emph{m}-separated given $Z$ (X, Y, and Z are disjoint sets; X, Y are nonempty). This criterion is referred to as a \emph{global Markov property}. We denote the independence model resulting from applying the \emph{m}-separation criterion to G, by $\Im_m$(G). This is an extension of Pearl's $d$-separation criterion to mixed graphs in that in a DAG D, a path is $d$-connecting if and only if it is m-connecting.
\end{definition}
\begin{definition}
	Let $G_A$ denote the induced subgraph of $G$ on the vertex set $A$, formed by removing from $G$ all vertices that are not in $A$, and all edges that do not have both endpoints in $A$. Two vertices $x$ and $y$ in an ADMG $G$ are said to be collider connected if there is a path from $x$ to $y$ in $G$ on which every non-endpoint vertex is a collider; such a path is called a collider path. (Note that a single edge trivially forms a collider path, so if $x$ and $y$ are adjacent in an ADMG then they are collider connected.) The \emph{augmented graph} derived from $G$, denoted $(G)^a$, is an undirected graph with the same vertex set as $G$ such that $c\--d \textrm{ in } (G)^a \Leftrightarrow c \textrm{ and } d \textrm{ are collider connected in } G.$
\end{definition}

\begin{definition}\label{RS2}
	Disjoint sets $X, Y\ne \emptyset,$ and $Z$ ($Z$ may be empty) are said to be
	\emph{$m^\ast$-separated} if $X$ and $Y$ are separated by Z in $(G_{\an(X\cup Y\cup Z)})^a$. Otherwise, $X$ and $Y$ are said to be $m^\ast$-connected
	given $Z$. The resulting independence model is denoted by $\Im_{m^\ast}(G)$.
\end{definition}
Richardson  in \cite[Theorem 1]{richardson2003markov} shows that for an ADMG $G$, $\Im_m(G)=\Im_{m^\ast}(G)$.

The \textit{Markov condition} is said to hold for $G = (V,E)$ and a probability distribution $P(V)$ if $\langle G,P\rangle$ satisfies the following implication:
$\forall X,Y\in V, \forall Z\subseteq V\setminus\{X,Y\}:(X{\!\perp\!\!\!\perp}_{m} Y |Z \Longrightarrow X{\!\perp\!\!\!\perp}_{p} Y |Z)$.
The \textit{faithfulness condition} states that the only conditional independencies to hold are those specified by the Markov condition, formally:
$\forall X,Y\in V, \forall Z\subseteq V\setminus\{X,Y\}: (X{\not\!\perp\!\!\!\perp}_{m} Y |Z \Longrightarrow X{\not\!\perp\!\!\!\perp}_{p} Y |Z)$. 

\section{Proofs of Theoretical Results}\label{app:A}
To prove Theorem \ref{thm:rctl}, we need the following definitions and propositions.
\begin{definition}[District and Induced Markov Blanket \cite{richardson2003markov}]\label{def:imb}
Let $G$ be an acyclic directed mixed graph. We first specify a total ordering $(\prec)$ on the vertices of $G$, such that $x\prec y \Rightarrow y\not\in \an(x)$; such an ordering is said to be consistent with $G$. Let $pre_{G,\prec}=\{v|v\prec x\textrm{ or }v=x\}$. The \emph{district} of $x$ in $G$ is $x$ and the set of vertices connected to
$x$ by a path on which every edge is of the form $\leftrightarrow$, denoted $dis_G(x)$. So,
$dis_G(x)=\{v|v\leftrightarrow\cdots\leftrightarrow x\in G\textrm{ or }v=x\}$. A set $A$ is said to be \emph{ancestral} if it is closed under the ancestor relation, i.e., if $\an(A)=A$. If $A$ is an ancestral set in an ADMG $G$, and $x$ is a vertex in $A$ that has \textbf{no} children in $A$ then we
define the Markov blanket of a vertex $x$ with respect to the \emph{induced subgraph} on $A$, called \emph{induced Markov blanket}, as the following:
$imb_A(x)=\pa_{G_A}(dist_{G_A}(x))\cup(dist_{G_A}(x)\setminus\{x\})$. It is not difficult to see that if $A=\an_G(T)$, then $imb_A(T)\subseteq \mb(T)$ by the definition of $\mb(T)$ in an ADMG.
\end{definition}
\begin{proposition}\label{prop:ancestor}
Given two nodes $X$ and $Y$ in an ADMG $G$ and a set $S$ of nodes not	containing $X$ and $Y$, there exists some subset of $S$ which $m$-separates $X$ and $Y$ if only if the set $S'=S\cap \an(X\cup Y)$ $m$-separates $X$ and $Y$.
\end{proposition}
\begin{proof}
	($\Rightarrow$) Proof by contradiction. Let $S'=S\cap \an(X\cup Y)$ and ${\langle X, Y \not| S'\rangle}$, i.e., $S'$ does not $m$-separates $X$ and $Y$ in $G$. Since $S'\subseteq \an(X\cup Y)$, it is obvious that $\an(X\cup Y\cup S')=\an(X\cup Y)$. So, $X$ and $Y$ are not separated by $S'$ in $(G_{\an(X\cup Y)})^a$, hence there is an undirected path $C$ between $X$ and $Y$ in $(G_{\an(X\cup Y)})^a$ that bypasses $S'$ i.e., the path $C$ is formed from nodes in $\an(X\cup Y)$ that are outside of $S$. Since $\an(X\cup Y)\subseteq \an(X\cup Y\cup S''), \forall S''\subseteq S$, then $(G_{\an(X\cup Y)})^a$ is a subgraph of $(G_{\an(X\cup Y\cup S)})^a$. Then, the previously found path $C$ is also a path in $(G_{\an(X\cup Y\cup S'')})^a$ that bypasses $S''$, which means that $X$ and $Y$ are not separated by any $S''\subseteq S$ in $(G_{\an(X\cup Y\cup S'')})^a$, which is a contradiction.
	
	\noindent ($\Leftarrow$) It is obvious.	
\end{proof}

Now, we prove Theorem \ref{thm:rctl}.
\begin{proof}[Proof of Theorem \ref{thm:rctl}] 
	Let $G$ be a causal graph with variable set $V$ consisting of system variables $X_{j \in J}$ and context variables $C_{i \in I}$. Assume that $C\in C_{i \in I}$, $T$ is the target variable and $\mb(T)$ is the Markov blanket of $T$.
	We have two cases: 
	\begin{itemize}
	    \item $C\not\in\mb(T)$: In this case, $C{\!\perp\!\!\!\perp}T|\mb(T)$ due to the definition of  Markov blanket.
	    \item$C\in\mb(T)$. In this case we show that considering the CDA assumptions, every path between $C$ and $T$ is blocked by $imb_A(T)$, where $A=\an_G(T)$, or there is no subset of variables that $m$-separates $C$ from $T$. Proposition \ref{prop:ancestor} implies that in order to find an $m$-separating set between $C$ and $T$ it is enough to restrict our search to the set $\an_G(C\cup T)=\{C\}\cup \an_G(T)$. Assume $A=\an_G(T)$, we have the following subcases:
	\begin{itemize}
		\item[(a)] $C\in A$ and $C\not\in imb_A(T)$. Using the ordered local Markov property for $G$ and Theorem 2 in \cite{richardson2003markov} implies that 
		$C{\!\perp\!\!\!\perp}_mT|imb_A(T)$. Note that $imb_A(T)\subseteq \mb(T)$ by the definition of $\mb(T)$.
		\item[(b)] $C\not\in A$ and $C\not\in imb_A(T)$. Since $imb_A(T)\subseteq A$ and $A$ is an ancestral set, then $\an_G(imb_A(T))=A$ and  $\an_G(C\cup T\cup imb_A(T))=\{C\}\cup \an_G(T)$. So, if there is an undirected path between $T$ and $C$ in $(G_{\an_G(C\cup T\cup imb_A(T))})^a$, Lemma 4 in \cite{richardson2003markov} implies that this path intersects with the set $imb_A(T)$. In other words, $C{\!\perp\!\!\!\perp}_mT|imb_A(T)$.
		\item[(c)] $C\in imb_A(T)$. Using Definition \ref{def:imb} implies that $C\in A$. In this case, there is no subset of variables that $m$-separates $C$ from $T$. Proof by contradiction: Assume that there is a subset of variables $Z$ that $m$-separates $C$ from $T$ in $G$. Proposition \ref{prop:ancestor} implies that $S=Z\cap\an_G(T\cup C)=Z\cap (\{G\}\cup\an_G(T))=Z\cap\an_G(T)$ must $m$-separate $T$ from $C$ in $G$. Since $C\in imb_A(T)$ and $imb_A(T)\subseteq \an_G(T)$, then it is not difficult to see that in $(G_{\an(T\cup C\cup S)})^a$, $C$ is directly connected to the vertex $T$. This means that $S$ does not $m$-separate $T$ from $C$ in $G$, which is a contradiction. 
		The following simple example illustrates such situations:
		\begin{figure}[h]
		\centering
			\begin{tikzpicture}[transform shape]
			\tikzset{vertex/.style = {shape=circle,inner sep=0pt,
					text width=5mm,align=center,
					draw=black,
					fill=white}}
			\tikzset{edge/.style = {->,> = latex',thick}}
			\node[vertex,thick] (c) at  (2,3) {$C$};
			\node[vertex,thick] (x) at  (2,1.5) {$Y$};
			\node[vertex,thick] (t) at  (2,0) {$T$};
			\node[vertex,thick] (y) at  (0,0) {$X$};
			\draw[thick,edge] (c) to (x);
			\draw[thick,edge] (x) to (t);
			\draw[thick,edge] (y) to (t);
			\draw[thick,edge] (y) to (x);
			\draw[thick,edge] (c) to (y);
			\draw[thick,edge] (t) to (y);
			\end{tikzpicture}
		\end{figure}
	\end{itemize}
	\end{itemize}

Note that in all cases, the Markov condition and faithful assumptions guarantee the correctness of independence relationships. As we have shown in all cases under the CDA assumptions, to find a separating set of features that $m$-separates $C$ from the target variable $T$ in the causal graph $G$, it is enough to restrict our search to the set of Markov blanket  $\mb(T)$ of the target variable $T$. In the case that $C_{i, i \in I}$ has more than one element, a similar argument can be used to prove the theorem.

Using any subset $S$ for prediction that
satisfies the $m$-separating set property, implies zero transfer bias. So, the best predictions
are then obtained by selecting a separating subset that also minimizes the source domain's risk (i.e.,
minimizes the incomplete information bias).
\end{proof}

\begin{proof}[Correctness of Algorithm \ref{alg:rctl}]
	The correctness of Algorithm \ref{alg:rctl} follows from Theorem \ref{thm:rctl}.
\end{proof}

\begin{proof}[Proof of Theorem \ref{thm:admg}]
	It is enough to show that for any $A\in V\setminus\{T,\textbf{Mb}(T)\}$, $T{\!\perp\!\!\!\perp}_mA|\textbf{Mb}(T)$. For this purpose, we prove that any path between $A$ and $T$ in $G$ is blocked by $\textbf{Mb}(T)$. In the following cases ($A-{\!\!\!*}B$, where means $A-B$ or $A\to B$ and $A{{*\!\!\!}-{\!\!\!*}}B$ means $A-B$, $A\to B$, or $A\gets B$) we have:
	\begin{enumerate}
		\item The path $\rho$ between $A$ and $T$ is of the form $A{{*\!\!\!}-{\!\!\!*}}\cdots {{*\!\!\!}-{\!\!\!*}}B\to T$. Clearly, $B\in \mb(T)$ blocks the path $\rho$.
		\item The path $\rho$ between $A$ and $T$ is of the form $A{{*\!\!\!}-{\!\!\!*}}\cdots{{*\!\!\!}-{\!\!\!*}} C\gets B\gets T$. Clearly, $B\in \mb(T)$ blocks the path $\rho$.
		\item The path $\rho$ between $A$ and $T$ is of the form $A{{*\!\!\!}-{\!\!\!*}}\cdots{{*\!\!\!}-{\!\!\!*}} C\to B\gets T$. $C\in \mb(T)$ blocks the path $\rho$.
		\item The path $\rho$ between $A$ and $T$ is of the form $A{{*\!\!\!}-{\!\!\!*}}\cdots{{*\!\!\!}-{\!\!\!*}}C\to D\leftrightarrow \cdots\leftrightarrow B\gets T$, where $\omega=C\to D\leftrightarrow \cdots\leftrightarrow B\gets T$ is the largest collider path between $T$ and a node on the path $\rho$. Since all nodes of $\omega$ are in the Markov blanket of $T$, $\forall X\in \omega, X\ne A$. So, $C\in \mb(T)$ blocks the path $\rho$.
		\item The path $\rho$ between $A$ and $T$ is of the form $A{{*\!\!\!}-{\!\!\!*}}\cdots{{*\!\!\!}-{\!\!\!*}}C\to D\leftrightarrow \cdots\leftrightarrow B\leftrightarrow T$, where $\omega=C\to D\leftrightarrow \cdots\leftrightarrow B\gets T$ is the largest collider path between $T$ and a node on the path $\rho$. Since all nodes of $\omega$ are in the Markov blanket of $T$, $\forall X\in \omega, X\ne A$. So, $C\in \mb(T)$ blocks the path $\rho$.
	\end{enumerate}
From the global Markov property, it follows that every $m$-separation relation in
$G$ implies conditional independence in every joint probability distribution $P$ that satisfies the
global Markov property for $G$. Thus, we have $T{\!\perp\!\!\!\perp}_p V\setminus\{T,\textbf{Mb}(T)\}|\textbf{Mb}(T)$.
\end{proof}

Now, we are ready to prove Theorem \ref{thm:Mbalgs}.
\begin{proof}[Sketch of proof of Theorem \ref{thm:Mbalgs}]
If a variable belongs to $\textbf{Mb}(T)$, then it will be admitted in the first step (Growing phase) at some point, since it will be dependent on $T$ given the candidate set of  $\textbf{Mb}(T)$.  This holds because of the causal faithfulness and because the set $\textbf{Mb}(T)$
is the minimal set with that property. If a variable $X$ is not a
member of $\textbf{Mb}(T)$, then conditioned on $\textbf{Mb}(T)\setminus\{X\}$, it will be independent of $T$ and thus will
be removed from the candidate set of $\textbf{Mb}(T)$ in the second phase (Shrinking phase) because the causal Markov condition entails that independencies in the distribution are represented in the graph. Since the causal faithfulness condition entails dependencies in the distribution from the graph, we never remove any variable $X$ from the candidate set of $\textbf{Mb}(T)$ if $X\in\textbf{Mb}(T)$.  Using this
argument inductively, we will end up with the $\textbf{Mb}(T)$.
\end{proof}

In order to show the details of the proof of the Theorem \ref{thm:Mbalgs}, we prove only the correctness of the Grow-Shrink Markov blanket (GSMB) algorithm without causal sufficiency assumption in details as following (for the other algorithms listed in Theorem \ref{thm:Mbalgs}, a similar argument can be used):
\begin{proof}[Proof of Correctness of the GSMB Algorithm]
By “correctness” we mean that GSMB is able to produce the true Markov blanket of any variable in the
ground truth ADMG under Markov condition and the faithfulness assumption if all conditional independence tests done during its course are assumed to be correct.

\begin{algorithm}[ht]
	\caption{The GS Markov Blanket Algorithm \cite{Margaritis2003}.}\label{alg:gsmb}
	\footnotesize\KwIn{a set $V$ of nodes, a target variable $T$, and a probability distribution $p$ faithful to an unknown ADMG $G$.}
	\KwOut{The Markov blanket of $T$ i.e., $\mb(T)$.}
	Set $S=\emptyset$\;
	\tcc{Grow Phase:}
	\While{$\exists X \in V\setminus\{T\} \textrm{ such that } X{\not\!\perp\!\!\!\perp}_pT|S$}{
		$S\gets S\cup\{X\}$;
	}
	\tcc{Shrink Phase:}
	\While{$\exists X \in S \textrm{ such that } X{\!\perp\!\!\!\perp}_pT|S\setminus\{X\}$}{
		$S\gets S\setminus\{X\}$;
	}
	return ($\mb(T)\gets S$)\;
\end{algorithm}	

We first prove that there does not exist any variable $X\in\mb(T)$ at the end of the growing phase that is not in $S$. The proof is by induction (a semi-induction approach on a finite subset of natural numbers)  on the \textit{length} of the collider path(s) between $X$ and $T$. We define the length of a collider path between $X$ and $T$ as the number of edges between them. Let $s$ be the length of a largest collider path between $X$ and $T$.
\begin{itemize}
    \item (\textit{Base case}) For the base of induction, consider the set of adjacent of $T$ i.e., $\adj(T)=\{X\in V|X\to T,X\gets T,\textrm{ or } X\leftrightarrow T\}$. In this case, $\nexists S\subseteq V\setminus\{T,X\} \textrm{ such that } X{\!\perp\!\!\!\perp}_mT|S$ in $G$. The faithfulness assumption implies that $\nexists S\subseteq V\setminus\{T,X\} \textrm{ such that } X{\!\perp\!\!\!\perp}_pT|S$. So, at the end of the growth phase, all the adjacent of $T$ are in the candidate set for the Markov blanket of $T$.
    \item (\textit{Induction hypothesis}) For all $1\le n<s$, if there is a collider path between $X$ and $T$ of length, $n$ then $X\in\mb(T)$ at the end of the growing phase.
    \item (\textit{Induction step}) To prove the inductive step, we assume the induction hypothesis for $n=s-1$ and then use this assumption to prove that the statement holds for $s$. Assume that $\rho=(v_0=T){*\!\!\!\!}\to v_1\leftrightarrow\cdots\leftrightarrow v_{s-1}\gets{\!\!\!\!*}v_s$ is a collider path of length $s$. From the induction hypothesis, we know that $v_i\in\mb(T), \forall 1\le i<s$ at the end of the grow phase. Using Definition \ref{RS} implies that $\rho$ is $m$-connected to $T$. This means that at some point of the grow phase $v_s$ falls into the set $S$. The faithfulness assumption implies that $\exists S\subseteq V\setminus\{T,v_s\} \textrm{ such that } v_s{\not\!\perp\!\!\!\perp}_pT|S$. So, at the end of the grow phase, all of $X$'s that there is a collider path between $X$ and $T$ fall into the candidate set for the Markov blanket of $T$. 
\end{itemize}
For the correctness of the shrinking phase, we have to prove two things: (1) we never remove any variable $X$ from $S$ if $X\in\mb(T)$, and (2) if $X\not\in\mb(T)$, $X$ is removed in the shrink phase. 

Now, we prove case (1) by contradiction. Assume that $X\in\mb(T)$,  $X\in S$ at the end of the grow phase, and we remove $X$ from $S$ in the shrink phase. This means $X\!\perp\!\!\!\perp_p T|S\setminus\{X\}$. Using the faithful assumption implies that $X\!\perp\!\!\!\perp_m T|S\setminus\{X\}$ i.e., $S\setminus\{X\}$ $m$-separates $X$ from $T$ in $G$, which is a contradiction because there is a collider path between $X$ and $T$ and $\mb(T)\setminus\{X\}\subseteq S$. In other word, the collider path between $X$ and $T$ is $m$-connected by $S\setminus\{X\}$.

To prove the case (2), assume that $X\not\in\mb(T)$, $\textbf{Y}=S\setminus\{\mb(T)\}$, and $X\in \textbf{Y}$. Due to the Markov blanket property, $\mb(T)$ $m$-separates $\textbf{Y}$ from $T$ in $G$. Using the Markov condition implies that $\textbf{Y}\!\perp\!\!\!\perp_p T|\mb(T)$. Since the probability distribution $p$ satisfies the faithfulness assumption, it satisfies the \textit{weak union} condition \cite{sadeghi2017faithfulness}. We recall the weak union property here: $A\!\perp\!\!\!\perp_p BD | C \Rightarrow (A\!\perp\!\!\!\perp_p B | DC \textrm{ and }  A\!\perp\!\!\!\perp_p D | BC)$. Using the weak union property for $\textbf{Y}\!\perp\!\!\!\perp_p T|\mb(T)$ implies that $X\!\perp\!\!\!\perp_p T|(\mb(T)\cup(\textbf{Y}\setminus\{X\}))$, which is the same as $X\!\perp\!\!\!\perp_p T|(S\setminus\{X\})$. This means, $X\not\in \mb(T)$ will be removed at the end of the shrink phase.
\end{proof}

\section{Experimental Evaluation}\label{sec:eval}
We empirically validated the robustness and scalability of~\rctl~on both synthetic (10-20 variables) and real-world high-dimensional data (400k variables) by comparing with various feature selection approaches including Conditional Mutual Information Maximization (CMIM)~\cite{fleuret2004fast}, Mutual Information Maximization (MIM)~\cite{lewis-1992-feature}, Adaboost (Adast)~\cite{freund1999short}, Greedy Subset Search (GSS)~\cite{rojas2018invariant}, C4.5 using Decision Trees (C4.5 + DTR)~\cite{quinlan1986induction} and Exhaustive Subset Search (\ess)~\cite{MagliacaneNIPS18}.

\subsection{Experimental Settings}\label{sec:setting}
In this section, we mention the different criteria used for  data generation, parameter-tuning and validation. Based on changes in parameters and user dependent variables, we divide this section into three key parts. In each subsection we describe the combinations of hyperparameters used, why we used them and how they were implemented in practice.

\subsubsection{Synthetic Data Generation}
Let us consider an ADMG $G$ as shown in Figure~\ref{fig:Graph}. The basic synthetic dataset based on $G$ consists of 19 randomly generated variables, with 16 system variables, 2 context variables ($C_1$ \& $C_2$) and 1 unknown confounder ($U$) (to be removed while using data).  The data were generated in Gaussian as well as Discrete distributions.  In Figure~\ref{fig:Graph} the dashed lines and circles annotate the supplementary additions to the central structure which is represented by solid lines.  We generate a dataset (see Appendix \ref{sec:evalsupp} for details) with desired number of samples and faithful w.r.t $G$. This data (not the graph structure) is used further for our experiments. 
We generated multiple datasets using different configurations to simulate the behaviour across different domain changes. These configurations were carefully generated to exploit characteristics of real-world scenarios:\\
(1) \textbf{Changes to the sample size:} We generated distributions with high disparity in the “amount” of data. The generated datasets contained 50, 1000 and 10000 cases. This would help us to see the difference in accuracy and robustness of the feature selection algorithms under extreme data-sensitive conditions (i.e., low sample size, moderate size, big data). \\
(2) \textbf{Changes to the network size:} We change the size of the graph, by increasing or dropping the number of nodes to monitor change in performance, to ensure scalability of the approach to changes in the number of environment variable. For this, we consider 3 different network sizes 20, 12 and 8 nodes, as shown in Figure~\ref{fig:Graph}
(further instances in Appendix). 
We also experimented by making changes to and on the number of context variables. This was done by using variables that are either partly, or completely unaffected by them (e.g., in~Figure \ref{fig:Graph}, we can use the whole or parts of graph with nodes $M$, $N$, $E$ as target). \\
(3) \textbf{Changes to the complexity:} The difference in complexity of the dataset can be further divided into \textit{structural} change i.e., changing the edge connections or relations among the nodes (e.g., in~Figure~\ref{fig:Graph}, we can add \textit{Z} to induce affect of $C_1$ on the target $T$, or connect \textit{B} and \textit{P}, to affect the Markov blanket of $T$), and \textit{domain-distribution} change i.e., changes between the source and target domain. Since we consider two context variables, domain-distribution change can be brought about by changes in either one or both. 
The changes to the domain can be classified further into three sub-settings: smooth (very little change to context variable),  mild (small change in context variables) and severe (drastic change in context variables between source and target) by varying the mean and variance for Gaussian distributions and changing the probabilities associated with the variables in discrete distributions.
To ensure consistency in our results, for each setting, we generate 8 test datasets, by making slight changes to any one of the system variables at random. The slight change is made so that the algorithms used for prediction do not report the same error each time, instead an error range serves a better depiction of practical robustness. 

\subsubsection{Real-world Dataset: Low-dimensional}\label{sec:real_descp}
We used the diabetes dataset~\cite{islam2020likelihood} that reports common diabetic symptoms of 520 persons and contains 17 features, with 320 cases of diabetes Type II positive and 200 of diabetes negative patients. In the dataset, we consider two context variables, Age and Gender. We perform 3 sets of data shifts between training and testing domains using these context variables: (1) \textit{Gender Shift}: We train on the Male and test on the Female. (2) \textit{Age Shift}: We split the data using an arbitrary threshold (we choose 50 as threshold for our experimentation as it gives a fair distribution of samples between training and testing data), train on young (less than 50) and test on the old (greater than 50). (3) \textit{Double context shift}: In this case, we train on young male patients and test on old female patients. We trained by sampling 100 samples from the source domain, keeping the samples in target domain constant to generates conditions for multiple source domains and single target domain. Such conditions help further test the robustness of the feature selection algorithm.

\subsubsection{Real-world Dataset: High-dimensional}\label{cancerdata}
To assess ~\rctl~in a high-dimension, low-sample size real world setting, we experiment on the Colorectal Cancer Dataset \cite{wang2020dysfunctional}. It contains 334 samples assigned to three cancer risk groups (low, medium, and high) based on their personal adenoma or cancer history. The entire dataset contains $\sim$ 400k features. We perform 2 sets of possible data shifts between training and testing domains using these context variables: (1) \textit{Gender Shift}: We train on the Male and test on the Female. (2) \textit{Age Shift}: We split the data using an arbitrary threshold, in our experiments, train on young (less than 50) and test on the older (greater than 50). 

\subsection{Evaluation Metrics}
We tested~\rctl~and other algorithms using different regression techniques such as RFR, SVR, and kNNR, in order to mitigate doubts about approach specific biases. We found, for \rctl, the same error more or less in (almost) all cases, but we chose Random Forest Regressor as it not only performed well, but also, gave consistent results. Here, the  \textit{Baseline} indicates using all available features and feeding it to the Random Forest Regressor. We reported accuracy with Mean-Square-Error (MSE) and time taken in second. 


\section{Details of Experimental Evaluation}\label{sec:evalsupp}
The data generation processes, conditional independence tests and Markov blanket Algorithms were implemented in \textit{R} by extending the \textit{bnlearn}~\cite{Scutari17} and \textit{pclag}~\cite{kalisch2020package} packages. We used Python 3.6 with \textit{scikit-learn}~\cite{pedregosa2011scikit} library for implementing the above-mentioned feature selection and machine learning algorithms to compare our approach against. The library is also used for prediction over subset of features selected by~\rctl~using \textit{RandomForestRegressor} function. The generated predictions from each of the algorithms are compared to the actual results for finding the error and all experiments have been reported using various comparison metrics. 
\subsection*{Synthetic Data Generation}
For Gaussian setting, a model graph as shown in Figure \ref{fig:Graph} was generated using \textit{model2network} function from the \textit{graphviz} package.
The obtained DAG from Figure \ref{fig:Graph} is passed to the \textit{custom.fit} function from \textit{bnlearn} which defines the mean, variance and relationship (edges values) between the nodes. This process ensures that the generated dataset will be faithful w.r.t $G$. Once the graph is set and all variables have been properly defined, we use the \textit{rbn} function from \textit{bnlearn} to generate the required number of samples.

For the Discrete setting, the same process is followed, but instead of mean, variance and edge values, we use a matrix of random probabilities. For each node, the size of the matrix would be the number of discrete values for the node, plus the number of discrete values for each of its connecting nodes multiplied by the number of modes connected to it.

\subsubsection*{Implementation Parameters}
\rctl~requires the use of Markov-blanket and neighborhood for feature selection. We use multiple Markov blanket discovery-algorithms to evaluate the effect the algorithm-choice can have on final prediction. For this purpose, we select, GS, IAMB, interIAMB, fastIAMB, fdrIAMB, MMMB, SI.HINTON.MB. These algorithms were implemented in \textit{R} using the~\textit{bnlearn} package. The variable \textit{subs} mentioned on line 12 of Algorithm \ref{alg:rctl} contains all possible combination of subsets from the feature set $S_2$. To find all possible subsets of a given feature set we use Python 3.5, one or more of these subsets will act as the separating set. To find the separating set, we sort the subsets based on $p_{value}$ of conditional independence tests and choosing the subset/s with the highest score. We use a significance level $\alpha$ of 0.05 as threshold for the conditional independence tests. The tests for different configurations are based on the type of data being used. For Gaussian data, we used \textit{gaussCItest} from the \textit{pclag} package, which uses Fisher’s z-transformation of the partial correlation, for testing correlation for sets of normally distributed random variables. For discrete data, we use the mutual information (\textit{mi}) test (an information-theoretic distance measure, which is somewhat proportional to the log-likelihood ratio) from the \textit{bnlearn} package.
\paragraph{Change to the complexity}
As an example, for a central structure consist of 8 nodes as shown in Figure~\ref{fig:Graph} having 10000 samples with severe changes in context, we train by generating one dataset on the base setting. Then make severe changes to the domain-distribution (i.e., by drastically changing the values of the context variables) and then generate a new dataset which is used for testing. 
\subsection*{Evaluation Metrics}
We used the RandomForestRegressor function from \textit{scikit-learn} package, using the out-of-bag (OOB) scoring approach on the default parameters, for this process. 

\section{More Experimental Results}\label{sec:moreresults}

In this section, we include the remaining results that we got during experimentation, along with a subsequent ground truth graph for representing the various settings we experiment on  synthetic data. In the later part of this section we also include t-test tables as mentioned in Section~\ref{sec:results}. More experimental results regarding synthetic data can be found at the end of this appendix.

\subsection{Real-world Dataset: Low-dimensional}\label{appendix:diabetes}
We used the diabetes dataset~\cite{islam2020likelihood} that contains reports of common diabetic symptoms of 520 persons. This includes data about symptoms that may cause or are potentially caused by diabetes. The dataset has been created from a direct questionnaire to people who have either recently been diagnosed as diabetic, or who are still non-diabetic but having show few or more symptoms of diabetes. The diabetes dataset contains a total of 17 features, with 320 cases of diabetes Type II positive and 200 of diabetes negative patients. The data has been collected from the patients of the Sylhet Diabetes Hospital, Sylhet, Bangladesh. In the dataset, we consider two context variables, Age, and Gender. We perform 3 sets of data shifts between training and testing domains using these context variables: (1) \textit{Gender Shift}: We train on the Male and test on the Female. (2) \textit{Age Shift}: We split the data using an arbitrary threshold (we choose 50 as threshold for our experimentation as it gives a fair distribution of samples between training and testing data), train on young (less than 50) and test on the older (greater than 50). (3) \textit{Double context shift}: We do a dataset shift for both context variables, in our case, we train on young male patients and test on old female patients. In practice, while experimenting on a context shift setting we train by sampling 100 samples from the source domain, keeping the samples in the target domain constant. This allows us to generate conditions similar to having multiple source domains and single target domain. Such conditions will help further test the robustness of the feature selection algorithm.

\subsection{Real-world Dataset: High-dimensional}\label{appendix:cancerdata}
To assess ~\rctl~ in a high-dimension, low-sample size real world setting, we experiment on the Colorectal Cancer Dataset \cite{wang2020dysfunctional}. Chronological age is a prominent risk factor for many types of cancers, including Colorectal Cancer, on which this dataset is based. The dataset is basically a genome-wide DNA methylation study on samples of normal colon mucosa, it contains 334 samples assigned to three cancer risk groups (low, medium, and high) based on their personal adenoma or cancer history. Our target variable for classification was the `Diagnosis' feature, a binary variable which indicates whether or not the patient was diagnosed with cancer, the entire dataset contains \textasciitilde 400k features. We again perform 2 sets of data shifts between training and testing domains using these context variables: (1) \textit{Gender Shift}: We train on the Male and test on the Female. (2) \textit{Age Shift}: We split the data using an arbitrary threshold (we choose 50 as threshold for our experimentation as it gives a fair distribution of samples between training and testing data), train on young (less than 50) and test on the older (greater than 50).

\begin{table}[ht]
\caption{Cancer dataset (1000 ranked features) result} 
\centering 
\begin{tabular}{l c c} 
\hline\hline 
Methodology & Accuracy & F1-Score \\ [0.5ex] 
\hline 
Baseline &  91.87	& 91.70  \\
~\rctl~(GS) (167 features) & 87.50	& 87.18  \\
~\rctl~(IAMB) (13 features) & 96.82 &	97.22  \\
CMIM + SVC & 87.5 & 87.06	  \\
Adaboost & \textbf{96.87} &	\textbf{97.83}   \\ [1ex] 
\hline 
\end{tabular}
\label{table:cancer1000} 
\end{table}

\begin{table}[ht]
\caption{Cancer dataset (10k ranked features) result} 
\centering 
\begin{tabular}{l c c} 
\hline\hline 
Methodology & Accuracy & F1-Score \\ [0.5ex] 
\hline 
Baseline &  91.25	& 90.94  \\
~\rctl~(GS) (167 features)& 88.125	& 87.88  \\
~\rctl~(IAMB) (79 features)& \textbf{97.5} &	97.65  \\
CMIM + SVC & 90.15 & 89.79	  \\
Adaboost & 96.87 &	\textbf{97.83}   \\ [1ex] 
\hline 
\end{tabular}
\label{table:cancer10k}
\end{table}


\subsection{\rctl~vs Unsupervised Domain Adaptation TCA}
\textbf{T}ransfer \textbf{C}omponent \textbf{A}nalysis (\textbf{TCA}) is an unsupervised technique for domain adaptation proposed by Pan et al. \cite{pan2010domain} that discovers common latent features, called \textit{transfer components}, that have the same marginal distribution across the source and target domains while maintaining the essential structure of the source domain data \cite{weiss2016survey}. TCA learns these transfer components across domains in a reproducing kernel Hilbert space \cite{steinwart2001influence} using maximum mean miscrepancy \cite{borgwardt2006integrating}. Once the transfer components are found, a traditional machine learning technique is used to train the final target classifier. In short, TCA is a two-step
domain adaptation technique where the first step reduces the marginal distributions between the source and target domains and the second step trains a classifier with the adapted domain data. Here, we compare \rctl~against TCA, which is a state-of-the-art
algorithm in unsupervised domain adaptation, and report the results. Note that, in TCA the focus is to identify substructures (i.e., transfer components)
that remains invariant across domains, which is different from the problem of identifying invariant features used
for training a predictive model, in which we do not need to have prior knowledge about the causal structure of the model that we are interested in. For implementation of the TCA we used, the \textit{rbf} kernel, and the \textit{logistic} classifier to predict on the extracted components. We also tune the hyperparameters of the algorithm by using Support Vector Classifier and increasing the number of components from 1 to 5, which lead to further decrease in the accuracy of the model. We hypothesize that TCA cannot learn the transfer components correctly in this setting. 
Further investigation by comparing with semi-supervised versions \cite{pan2010domain} as well as strategies that also incorporate sample reweighing are left for future work.

\begin{figure}[ht]
	\includegraphics[width=.5\linewidth]{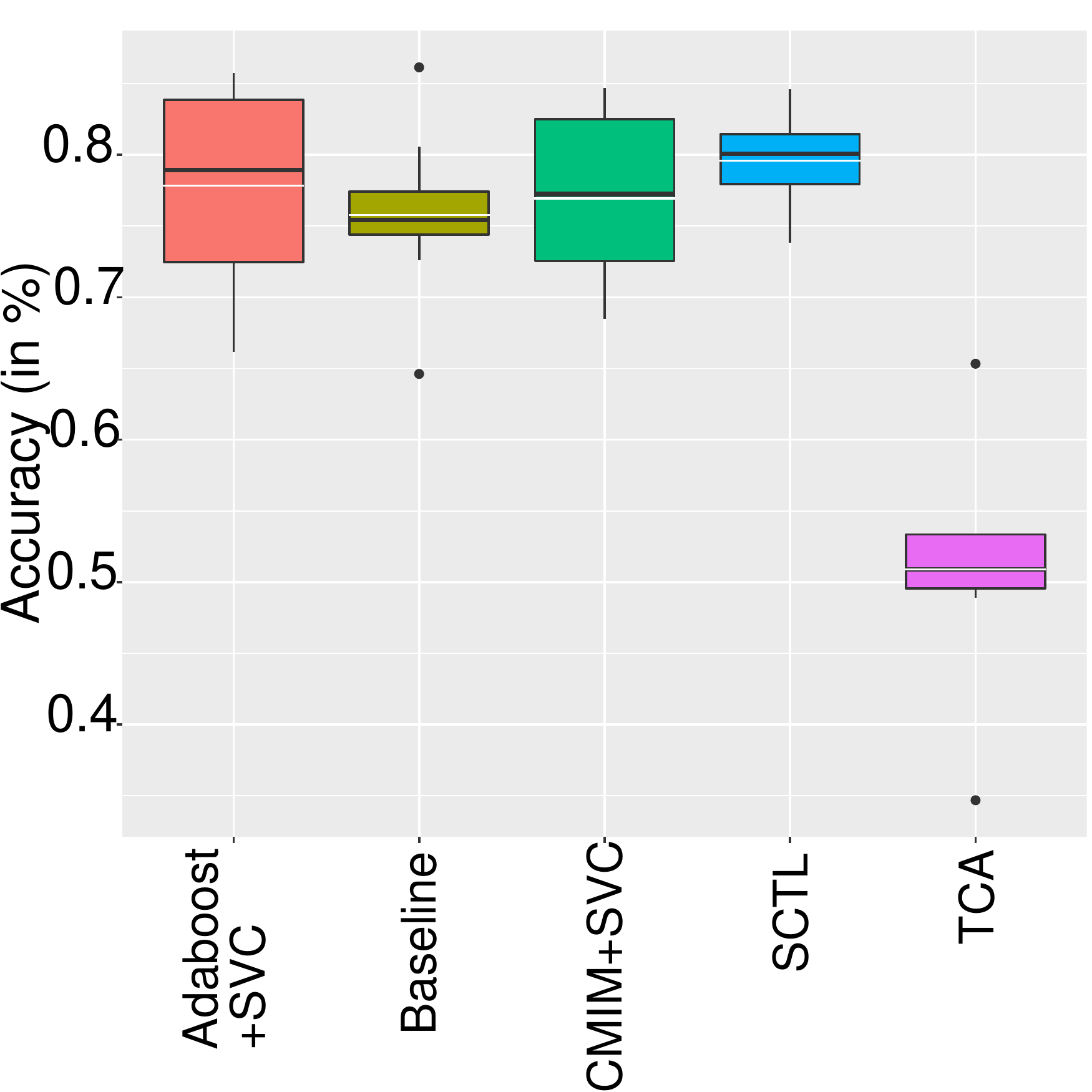}%
	\includegraphics[width=.5\linewidth]{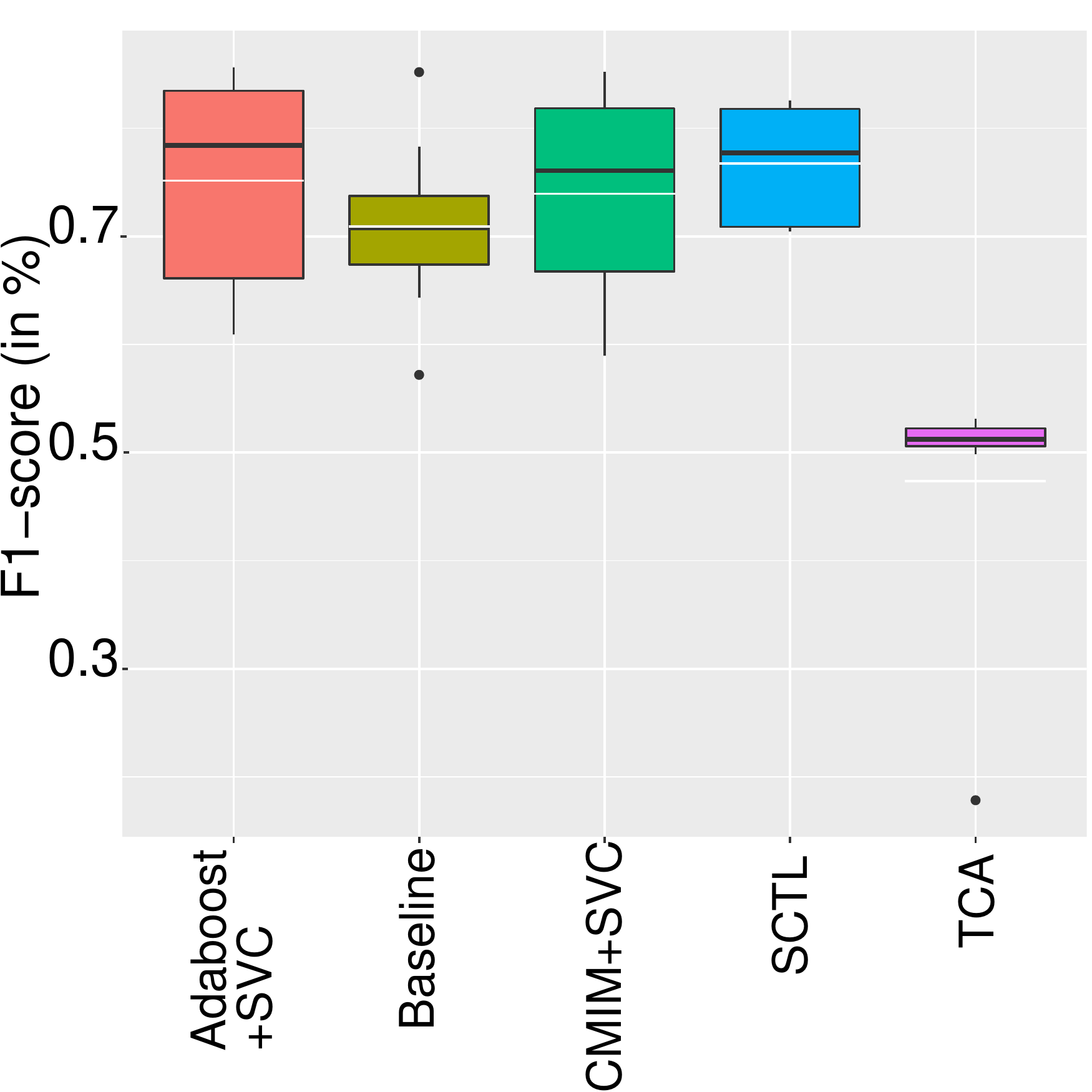}
	\caption{ Error rate comparison of multiple feature selection approaches (supervised and unsupervised) on a simulation of the Diabetes dataset with \textit{Age Shift} (discussed in Section~\ref{sec:real_descp}).}
	\label{fig:RCTLvsTCA}
\end{figure}

\begin{table}[t]
\caption{Markov blanket and separating feature sets extracted by various algorithms on the target variable $T$ for Figure \ref{fig:targetM}(a).} 
\centering 
\begin{tabular}{l c c} 
\hline\hline 
Methodology & Markov blanket & Separating set \\ [0.5ex] 
\hline 
GS & ["K", "L",  "N" , "Y"]& ["K",  "L","N", "Y" ] \\ 
IAMB & ["K", "L",  "N"]& ["K",  "L","N" ] \\
interIAMB & ["K", "L",  "N"]& ["K",  "L","N" ] \\
fastIAMB & ["K", "L",  "N", "Y"]& ["K",  "L","N","Y" ]\\
MMMB & ["K", "L",  "N"]& ["K",  "L","N" ] \\
fdrIAMB & ["K", "L",  "N"]& ["K",  "L","N" ] \\ [1ex] 
\hline 
\end{tabular}
\label{table:robust1} 
\end{table}

\begin{table}[!ht]
\caption{Markov blanket and separating feature sets extracted by various algorithms on the target variable $T$ for Figure \ref{fig:smallg}(a).} 
\centering 
\begin{tabular}{l c c} 
\hline\hline 
Methodology & Markov blanket & Separating set \\ [0.5ex] 
\hline 
GS & ["C1", "P",  "X" , "Y" ]& ["P",  "X" ] \\ 
IAMB & ["C1", "P" , "X"  ,"Y"]&  ["P",  "X" ]\\
interIAMB & ["C1", "P",  "X" , "Y" ]&  ["P",  "X" ] \\
fastIAMB & ["C1", "P",  "X" , "Y" ] & ["P",  "X" ] \\
MMMB & ["C1", "P",  "X" , "Y" ]  & ["P",  "X" ] \\
fdrIAMB & ["C1", "P",  "X" , "Y" ]  & ["P",  "X" ] \\
ESS & NA & ["Q", "X", "Y"] \\  [1ex] 
\hline 
\end{tabular}
\label{table:robust2} 
\end{table}

\paragraph{Computational complexity of \rctl.}
Assume that the “learning Markov
blankets” phase (Step 1) uses the grow-shrink (GS) algorithm \cite{Margaritis99}, and $n=|V|$, $m=|E|$, where
$G=(V,E)$ is the unknown true causal structure. Since the Markov blanket algorithm
involves $O(n)$ conditional independence (CI) tests, if the cardinality of the Markov
blanket of the target variable is $k$ (in many real world cases $k << n$, see for example real world datasets and their properties in \cite{bnlearn}) then the total
number of needed Cl tests for finding the best separating set is O$(n+2^k)$, see \cite{Margaritis2003} for more details. On the
other hand, the brute force algorithm used in (Magliacane et al., 2018) needs O$(2^n)$
tests in the worst case scenario, which is infeasible in practice as we have shown in our
experiments. The same argument is valid when we have multiple target variables, say $c$, where $c<<n$.

\begin{figure*}[ht]
    \subcaptionbox{}[.33\textwidth]{%
	    \begin{tikzpicture}[transform shape,scale=.56]
	\tikzset{vertex/.style = {shape=circle,align=center,draw=black, fill=white}}
\tikzset{edge/.style = {->,> = latex',thick}}
	\node[vertex,thick] (l) at  (-5.3,5) {L};
	\node[vertex,thick] (k) at  (-3.7,5) {K};
	\node[vertex,thick] (j) at  (-3.7,3.8) {J};
	\node[vertex,thick](m) at  (-5.3,3.8) {M};
	\node[vertex,thick] (n) at  (-4.5,2.6) {N};
    \node[vertex,thick,fill=green] (c1) at  (-2.5,5) {c1};
    \node[vertex,thick,fill= gray!40,dashed] (u) at  (-1.5,6.2) {U};
    \node[vertex,thick] (c2) at  (-0.5,5) {c2};
    \node[vertex,thick](x) at  (-0.5,3.8) {X};
    \node[vertex,thick,fill=Dandelion ] (t) at  (-0.5,1.4) {T};
    \node[vertex,thick ] (y) at  (-1.5,.2) {Y};
	\node[vertex,thick] (p) at  (0.8,1.4) {P};
	\node[vertex,thick] (q) at  (0.8,0.2) {Q};
	\node[vertex,thick] (b) at  (0.8,3.8) {B};
    \node[vertex,thick] (d) at  (0.8,5) {D};
    \node[vertex,thick] (e) at  (2,5) {E};
	\node[vertex,thick] (i) at  (2,3.8) {I};
    \node[vertex,thick] (f) at  (-0.5,6.2) {F};
    \node[vertex,thick] (g) at  (.8,6.2) {G};
    \node[vertex,thick] (h) at  (2,6.2) {H};
	\draw[edge] (k) to (l);
	\draw[edge] (k) to (m);
	\draw[edge] (k) to (j);
	\draw[edge] (m) to (n);
	\draw[edge] (j) to (n);
	\draw[edge] (l) to (m);
	\draw[edge,fill = Cyan] (u) to (c1);
	\draw[edge,fill = Cyan] (u) to (c2);
	\draw[edge] (c1) to (y);
	\draw[edge] (c2) to (x);
	\draw[edge] (x) to (t);
	\draw[edge] (t) to (y);
	\draw[edge] (t) to (p);
	\draw[edge] (p) to (q);
	\draw[edge] (d) to (b);
	\draw[edge] (c2) to (b);    
	\draw[edge] (e) to (b);
	\draw[edge] (e) to (i);
	\draw[edge] (f) to (d);
	\draw[edge] (g) to (e);
	\draw[edge] (e) to (i);
	\draw[edge] (h) to (e);
	\draw[edge] (j) to (x);
\end{tikzpicture}

}%
\subcaptionbox{}[.33\textwidth]{%
	\includegraphics[scale= 0.3]{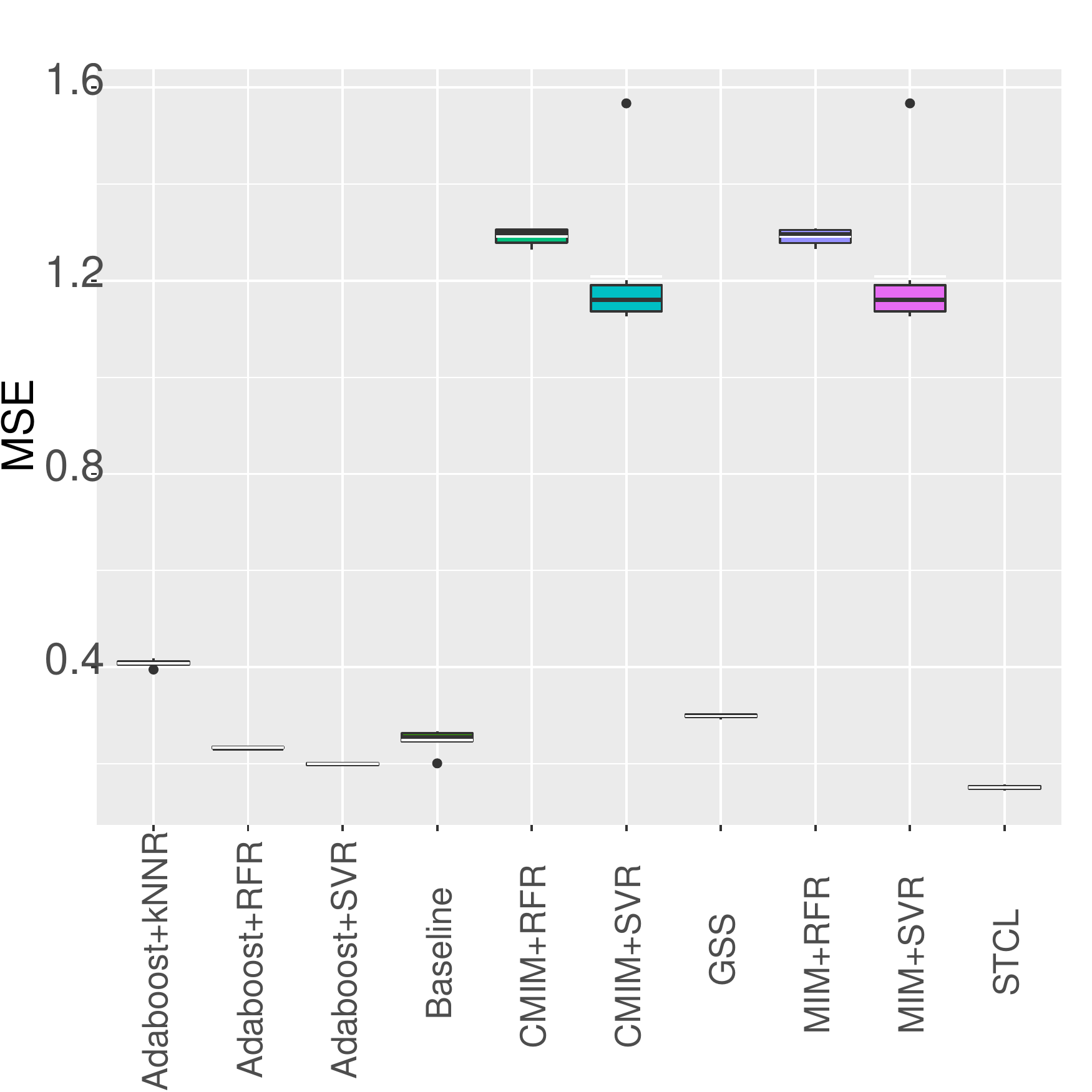}
    }%
	\subcaptionbox{}[.33\textwidth]{%
	\includegraphics[scale= 0.3]{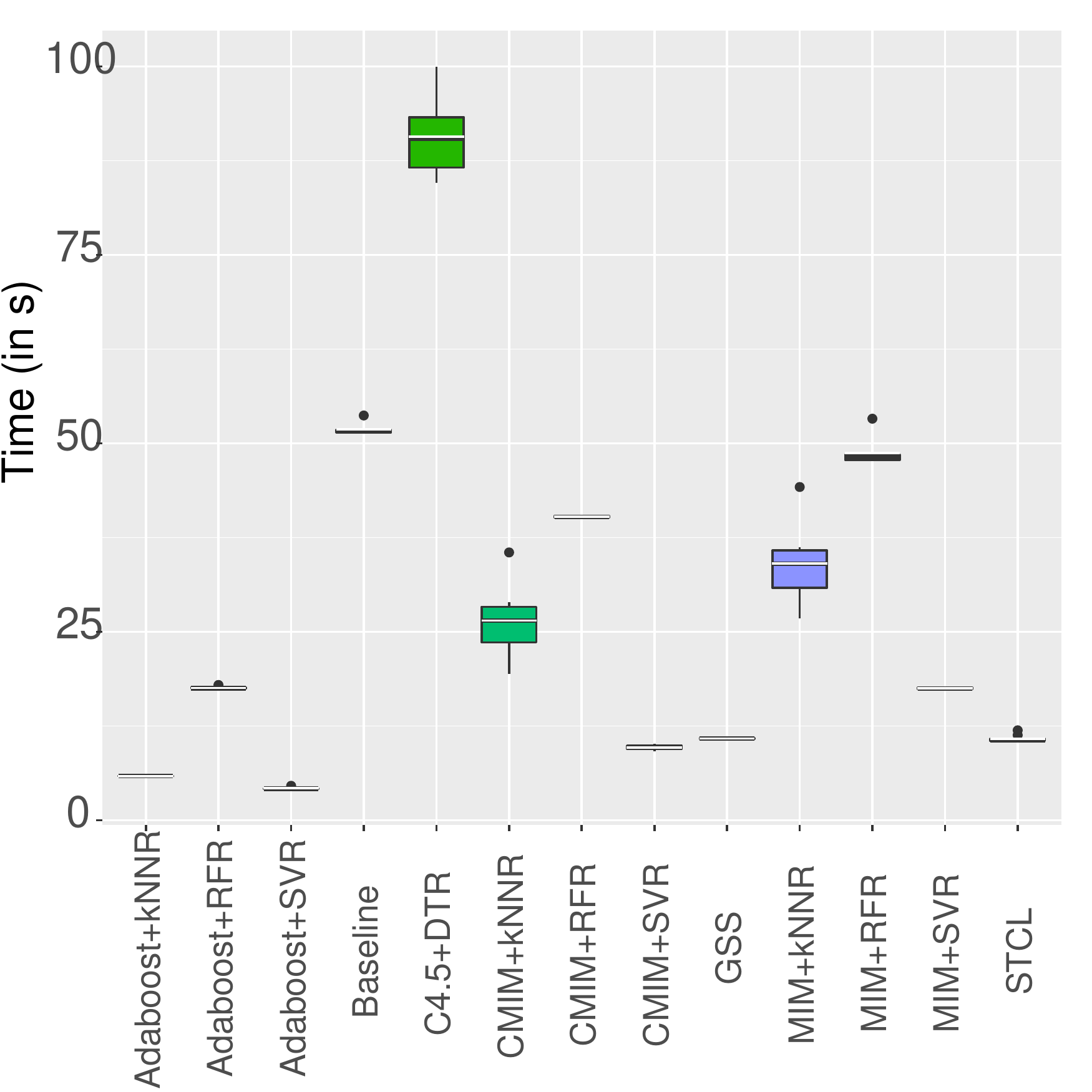}}
\caption{Results on (b), (c) are generated over the ground truth graph (a) for the target variable \textcolor{orange}{$T$}. Data shift between source and target occurs by change in the probability distribution of the context variable \textcolor{green}{$C_1$}, where sample size = 10000 for a Gaussian distribution.}	
\end{figure*}
\begin{figure*}
    \subcaptionbox{}[.33\textwidth]{%
	    \begin{tikzpicture}[transform shape,scale=.56]
	\tikzset{vertex/.style = {shape=circle,align=center,draw=black, fill=white}}
\tikzset{edge/.style = {->,> = latex',thick}}
	\node[vertex,thick] (l) at  (-5.3,5) {L};
	\node[vertex,thick] (k) at  (-3.7,5) {K};
	\node[vertex,thick] (j) at  (-3.7,3.8) {J};
	\node[vertex,thick](m) at  (-5.3,3.8) {M};
	\node[vertex,thick] (n) at  (-4.5,2.6) {N};
    \node[vertex,thick,fill=green] (c1) at  (-2.5,5) {c1};
    \node[vertex,thick,fill= gray!40,dashed] (u) at  (-1.5,6.2) {U};
    \node[vertex,thick] (c2) at  (-0.5,5) {c2};
    \node[vertex,thick](x) at  (-0.5,3.8) {X};
    \node[vertex,thick,fill=Dandelion ] (t) at  (-0.5,1.4) {T};
    \node[vertex,thick ] (y) at  (-1.5,.2) {Y};
	\node[vertex,thick] (p) at  (0.8,1.4) {P};
	\node[vertex,thick] (q) at  (0.8,0.2) {Q};
	\node[vertex,thick] (b) at  (0.8,3.8) {B};
    \node[vertex,thick] (d) at  (0.8,5) {D};
    \node[vertex,thick] (e) at  (2,5) {E};
	\node[vertex,thick] (i) at  (2,3.8) {I};
    \node[vertex,thick] (f) at  (-0.5,6.2) {F};
    \node[vertex,thick] (g) at  (.8,6.2) {G};
    \node[vertex,thick] (h) at  (2,6.2) {H};
	\draw[edge] (k) to (l);
	\draw[edge] (k) to (m);
	\draw[edge] (k) to (j);
	\draw[edge] (m) to (n);
	\draw[edge] (j) to (n);
	\draw[edge] (l) to (m);
	\draw[edge,fill = Cyan] (u) to (c1);
	\draw[edge,fill = Cyan] (u) to (c2);
	\draw[edge] (c1) to (y);
	\draw[edge] (c2) to (x);
	\draw[edge] (x) to (t);
	\draw[edge] (t) to (y);
	\draw[edge] (t) to (p);
	\draw[edge] (p) to (q);
	\draw[edge] (d) to (b);
	\draw[edge] (c2) to (b);    
	\draw[edge] (e) to (b);
	\draw[edge] (e) to (i);
	\draw[edge] (f) to (d);
	\draw[edge] (g) to (e);
	\draw[edge] (e) to (i);
	\draw[edge] (h) to (e);
	\draw[edge] (j) to (x);
\end{tikzpicture}

}%
\subcaptionbox{}[.33\textwidth]{%
	\includegraphics[scale= 0.3]{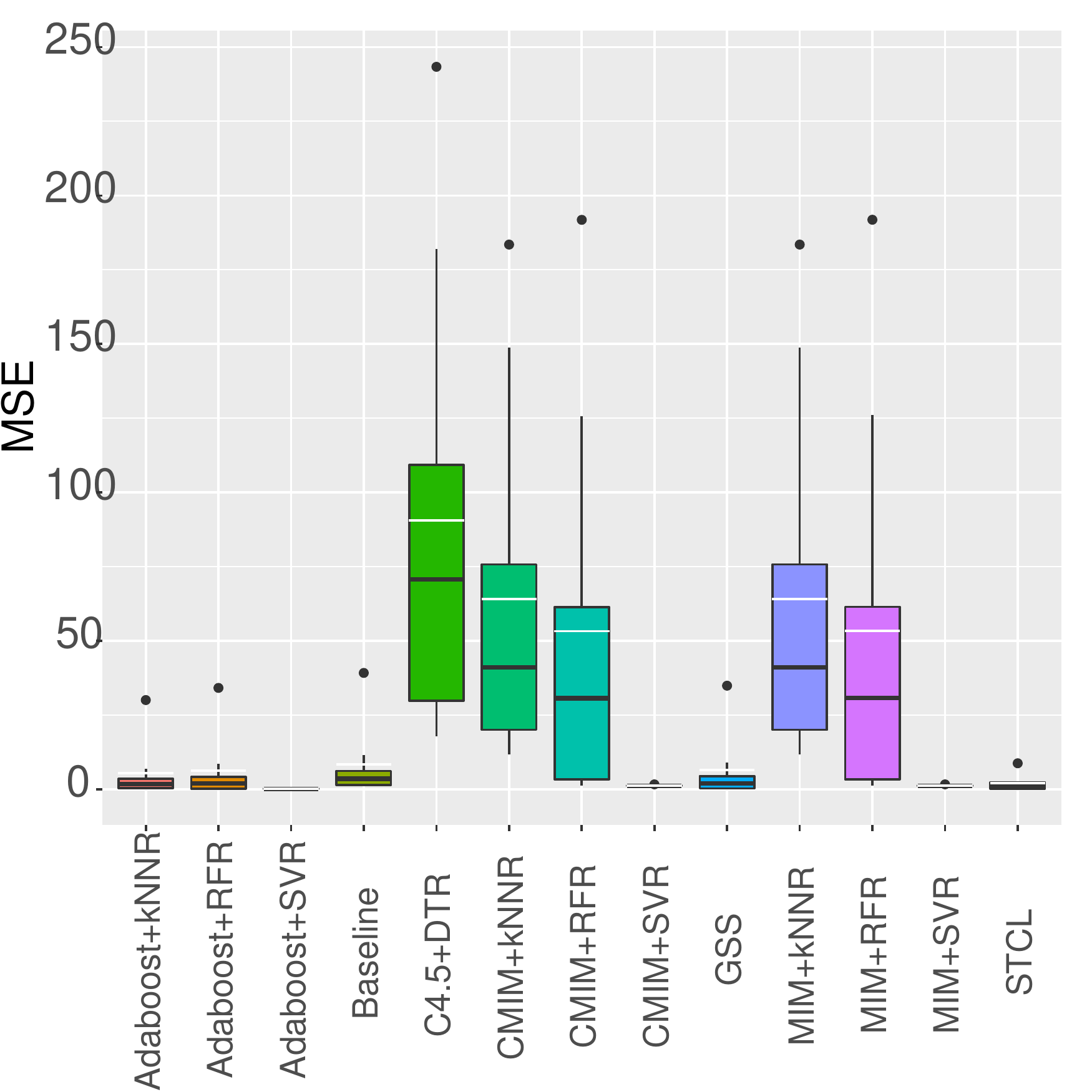}
    }%
	\subcaptionbox{}[.33\textwidth]{%
	\includegraphics[scale= 0.3]{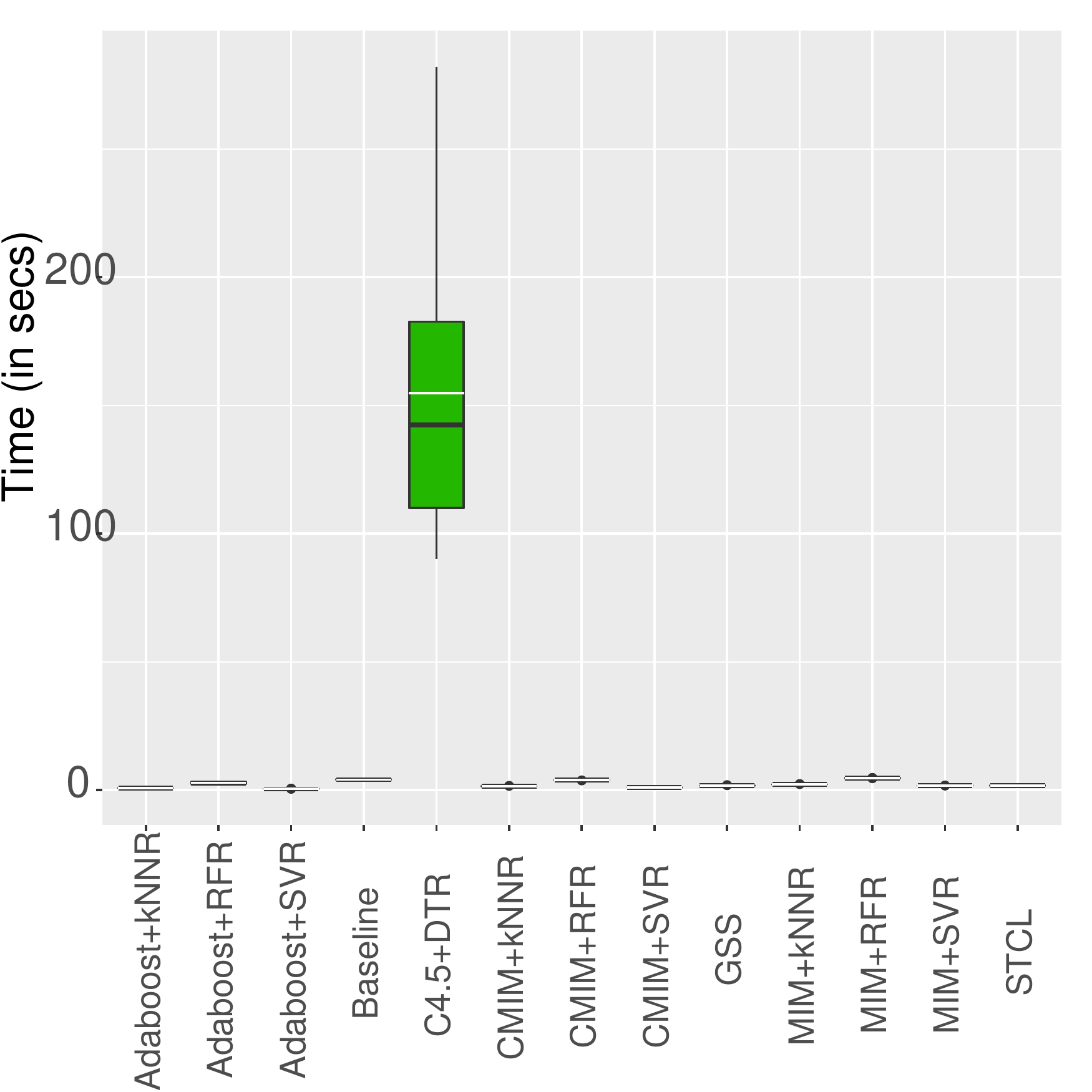}}
\caption{Results on (b), (c), (d) are generated over the ground truth graph (a) for the target variable \textcolor{orange}{$T$}. Data shift between source and target occurs by change in the probability distribution of the context variable \textcolor{green}{$C_1$}, where sample size = 1000 for a Gaussian distribution.}	
\end{figure*}
\begin{figure*}
    \subcaptionbox{}[.33\textwidth]{%
	    \begin{tikzpicture}[transform shape,scale=.56]
	\tikzset{vertex/.style = {shape=circle,align=center,draw=black, fill=white}}
\tikzset{edge/.style = {->,> = latex',thick}}
	\node[vertex,thick] (l) at  (-5.3,5) {L};
	\node[vertex,thick] (k) at  (-3.7,5) {K};
	\node[vertex,thick] (j) at  (-3.7,3.8) {J};
	\node[vertex,thick](m) at  (-5.3,3.8) {M};
	\node[vertex,thick] (n) at  (-4.5,2.6) {N};
    \node[vertex,thick,fill=green] (c1) at  (-2.5,5) {c1};
    \node[vertex,thick,fill= gray!40,dashed] (u) at  (-1.5,6.2) {U};
    \node[vertex,thick] (c2) at  (-0.5,5) {c2};
    \node[vertex,thick](x) at  (-0.5,3.8) {X};
    \node[vertex,thick,fill=Dandelion ] (t) at  (-0.5,1.4) {T};
    \node[vertex,thick ] (y) at  (-1.5,.2) {Y};
	\node[vertex,thick] (p) at  (0.8,1.4) {P};
	\node[vertex,thick] (q) at  (0.8,0.2) {Q};
	\node[vertex,thick] (b) at  (0.8,3.8) {B};
    \node[vertex,thick] (d) at  (0.8,5) {D};
    \node[vertex,thick] (e) at  (2,5) {E};
	\node[vertex,thick] (i) at  (2,3.8) {I};
    \node[vertex,thick] (f) at  (-0.5,6.2) {F};
    \node[vertex,thick] (g) at  (.8,6.2) {G};
    \node[vertex,thick] (h) at  (2,6.2) {H};
	\draw[edge] (k) to (l);
	\draw[edge] (k) to (m);
	\draw[edge] (k) to (j);
	\draw[edge] (m) to (n);
	\draw[edge] (j) to (n);
	\draw[edge] (l) to (m);
	\draw[edge,fill = Cyan] (u) to (c1);
	\draw[edge,fill = Cyan] (u) to (c2);
	\draw[edge] (c1) to (y);
	\draw[edge] (c2) to (x);
	\draw[edge] (x) to (t);
	\draw[edge] (t) to (y);
	\draw[edge] (t) to (p);
	\draw[edge] (p) to (q);
	\draw[edge] (d) to (b);
	\draw[edge] (c2) to (b);    
	\draw[edge] (e) to (b);
	\draw[edge] (e) to (i);
	\draw[edge] (f) to (d);
	\draw[edge] (g) to (e);
	\draw[edge] (e) to (i);
	\draw[edge] (h) to (e);
	\draw[edge] (j) to (x);
\end{tikzpicture}

}%
\subcaptionbox{}[.33\textwidth]{%
	\includegraphics[scale= 0.3]{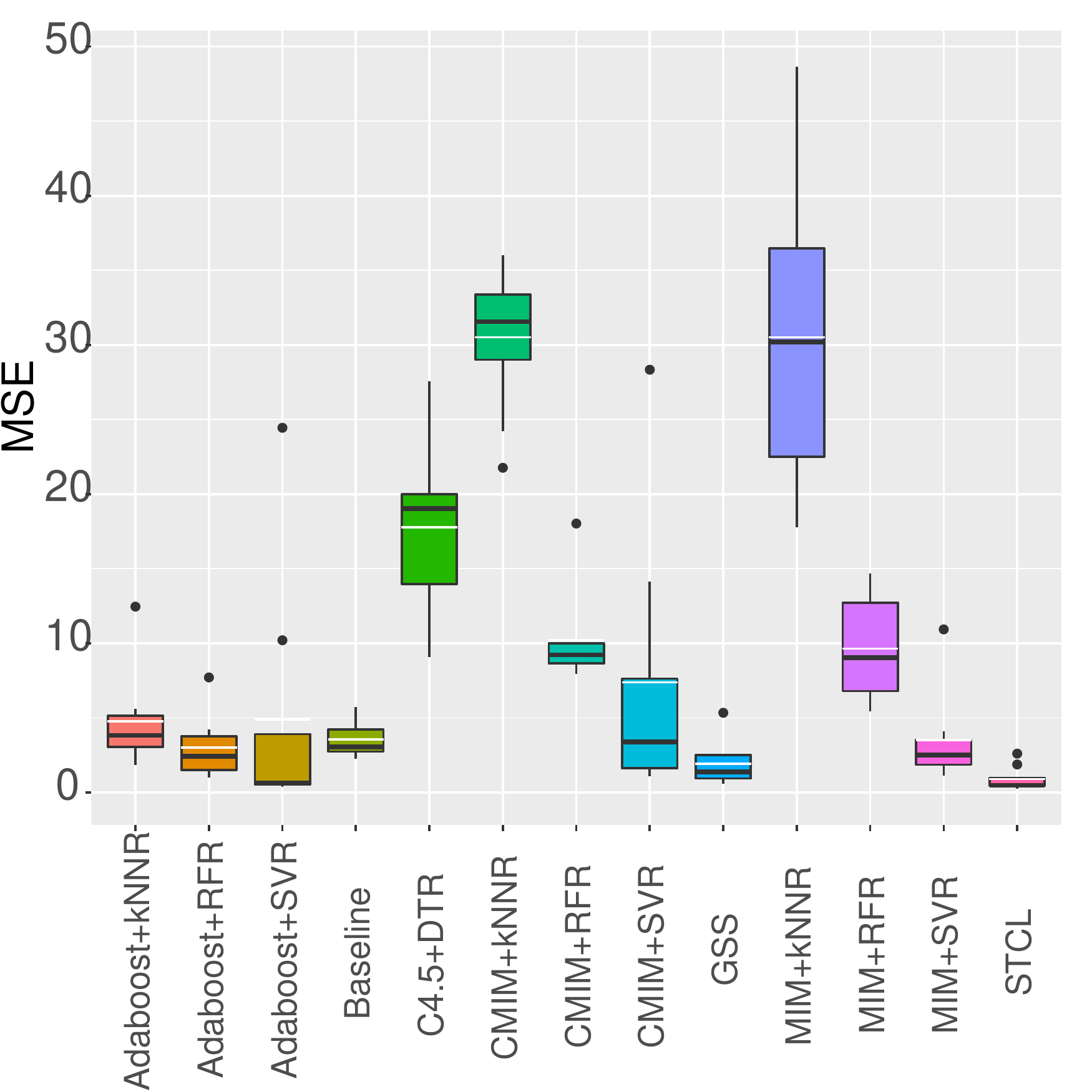}
    }%
	\subcaptionbox{}[.33\textwidth]{%
	\includegraphics[scale= 0.3]{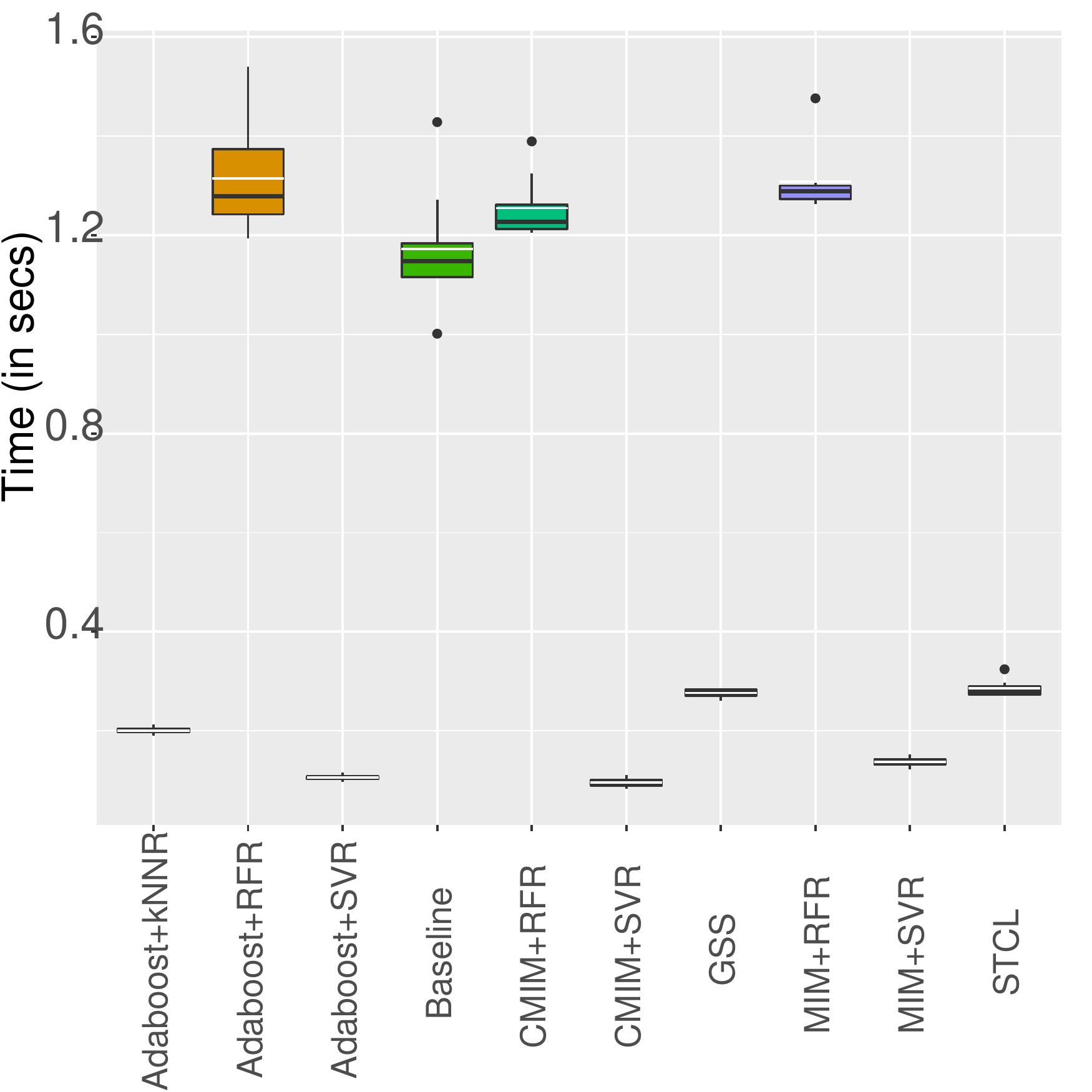}}
\caption{Results on (b), (c), (d) are generated over the ground truth graph (a) for the target variable \textcolor{orange}{$T$}. Data shift between source and target occurs by change in the probability distribution of the context variable \textcolor{green}{$C_1$}, where sample size = 50 for a Gaussian distribution.}	
\end{figure*}

\begin{figure*}
    \subcaptionbox{}[.33\textwidth]{%
	    \begin{tikzpicture}[transform shape,scale=.56]
	\tikzset{vertex/.style = {shape=circle,align=center,draw=black, fill=white}}
\tikzset{edge/.style = {->,> = latex',thick}}
	\node[vertex,thick] (l) at  (-5.3,5) {L};
	\node[vertex,thick] (k) at  (-3.7,5) {K};
	\node[vertex,thick] (j) at  (-3.7,3.8) {J};
	\node[vertex,thick](m) at  (-5.3,3.8) {M};
	\node[vertex,thick] (n) at  (-4.5,2.6) {N};
    \node[vertex,thick,fill=green] (c1) at  (-2.5,5) {c1};
    \node[vertex,thick,fill= gray!40,dashed] (u) at  (-1.5,6.2) {U};
    \node[vertex,thick,fill=green] (c2) at  (-0.5,5) {c2};
    \node[vertex,thick](x) at  (-0.5,3.8) {X};
    \node[vertex,thick,fill=Dandelion ] (t) at  (-0.5,1.4) {T};
    \node[vertex,thick ] (y) at  (-1.5,.2) {Y};
	\node[vertex,thick] (p) at  (0.8,1.4) {P};
	\node[vertex,thick] (q) at  (0.8,0.2) {Q};
	\node[vertex,thick] (b) at  (0.8,3.8) {B};
    \node[vertex,thick] (d) at  (0.8,5) {D};
    \node[vertex,thick] (e) at  (2,5) {E};
	\node[vertex,thick] (i) at  (2,3.8) {I};
    \node[vertex,thick] (f) at  (-0.5,6.2) {F};
    \node[vertex,thick] (g) at  (.8,6.2) {G};
    \node[vertex,thick] (h) at  (2,6.2) {H};
	\draw[edge] (k) to (l);
	\draw[edge] (k) to (m);
	\draw[edge] (k) to (j);
	\draw[edge] (m) to (n);
	\draw[edge] (j) to (n);
	\draw[edge] (l) to (m);
	\draw[edge,fill = Cyan] (u) to (c1);
	\draw[edge,fill = Cyan] (u) to (c2);
	\draw[edge] (c1) to (y);
	\draw[edge] (c2) to (x);
	\draw[edge] (x) to (t);
	\draw[edge] (t) to (y);
	\draw[edge] (t) to (p);
	\draw[edge] (p) to (q);
	\draw[edge] (d) to (b);
	\draw[edge] (c2) to (b);    
	\draw[edge] (e) to (b);
	\draw[edge] (e) to (i);
	\draw[edge] (f) to (d);
	\draw[edge] (g) to (e);
	\draw[edge] (e) to (i);
	\draw[edge] (h) to (e);
	\draw[edge] (j) to (x);
\end{tikzpicture}

}%
\subcaptionbox{}[.33\textwidth]{%
	\includegraphics[scale= 0.3]{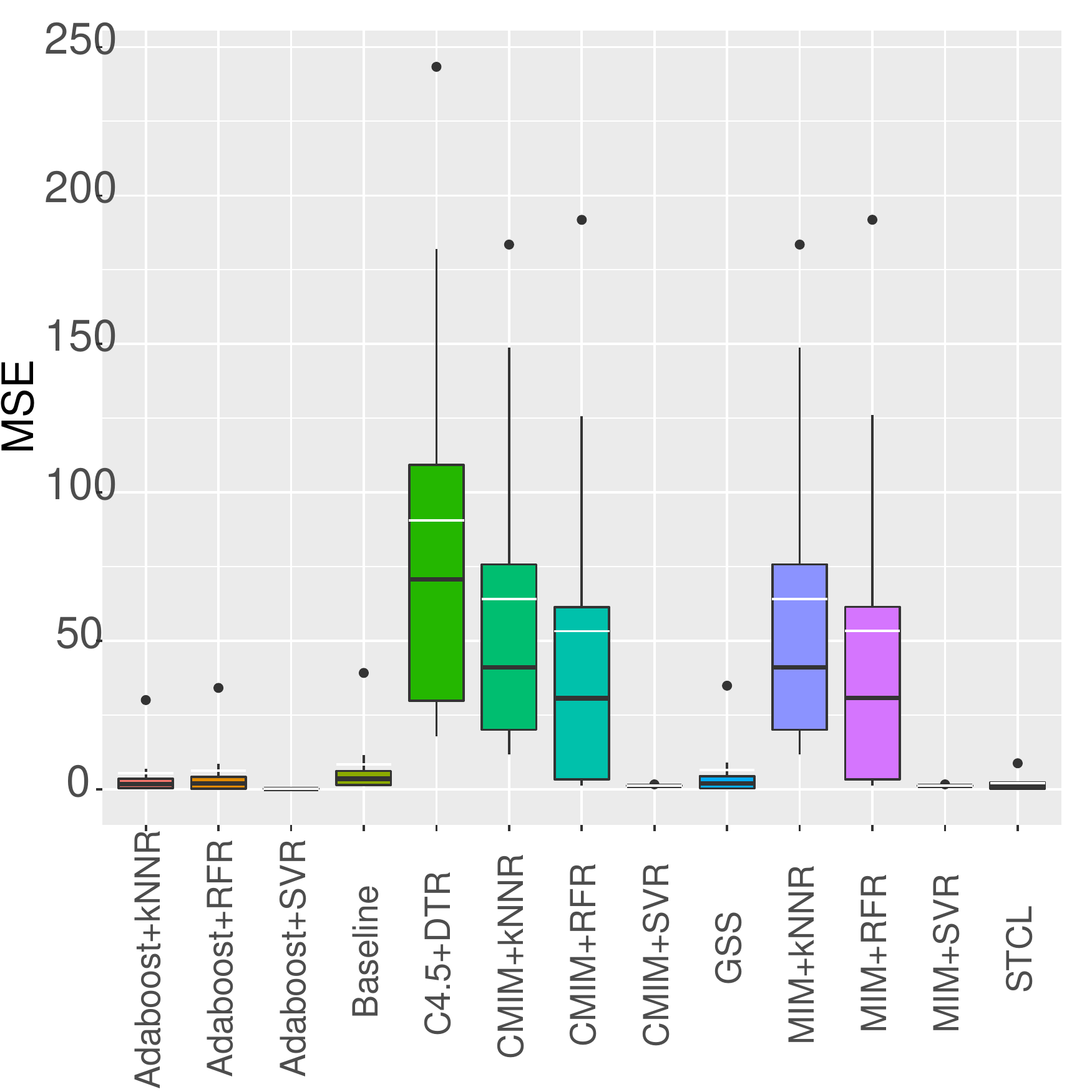}
    }%
	\subcaptionbox{}[.33\textwidth]{%
	\includegraphics[scale= 0.3]{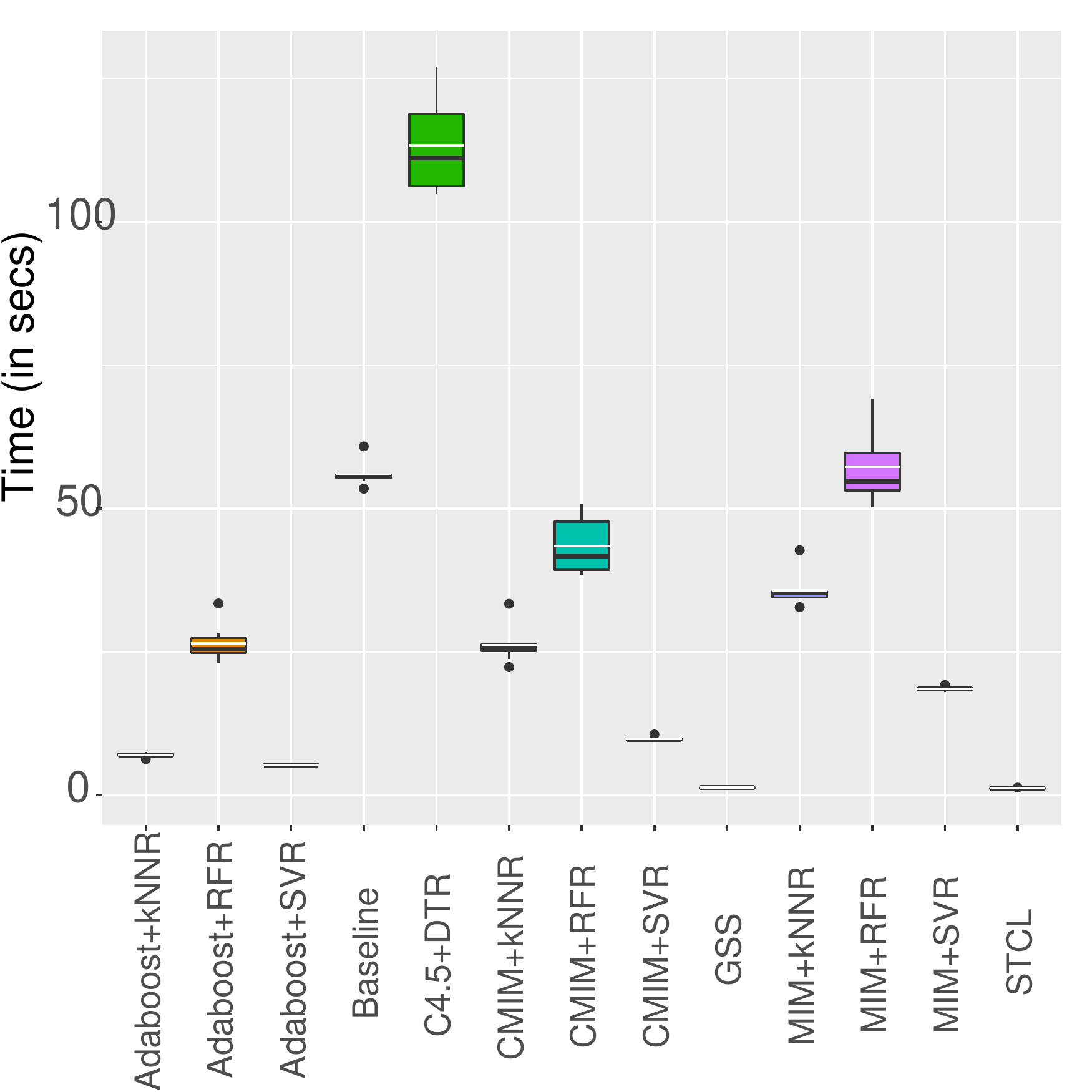}}
\caption{Results on (b), (c), (d) are generated over the ground truth graph (a) for the target variable \textcolor{orange}{$T$}. Data shift between source and target occurs by change in the probability distribution of the context variable \textcolor{green}{$C_1$}\& \textcolor{green}{$C_2$}, where sample size = 10000 for a Gaussian distribution.}	
\end{figure*}

\begin{figure*}
    \subcaptionbox{}[.33\textwidth]{%
	    \begin{tikzpicture}[transform shape,scale=.56]
	\tikzset{vertex/.style = {shape=circle,align=center,draw=black, fill=white}}
\tikzset{edge/.style = {->,> = latex',thick}}
	\node[vertex,thick] (l) at  (-5.3,5) {L};
	\node[vertex,thick] (k) at  (-3.7,5) {K};
	\node[vertex,thick] (j) at  (-3.7,3.8) {J};
	\node[vertex,thick](m) at  (-5.3,3.8) {M};
	\node[vertex,thick] (n) at  (-4.5,2.6) {N};
    \node[vertex,thick,fill=green] (c1) at  (-2.5,5) {c1};
    \node[vertex,thick,fill= gray!40,dashed] (u) at  (-1.5,6.2) {U};
    \node[vertex,thick,fill=green] (c2) at  (-0.5,5) {c2};
    \node[vertex,thick](x) at  (-0.5,3.8) {X};
    \node[vertex,thick,fill=Dandelion ] (t) at  (-0.5,1.4) {T};
    \node[vertex,thick ] (y) at  (-1.5,.2) {Y};
	\node[vertex,thick] (p) at  (0.8,1.4) {P};
	\node[vertex,thick] (q) at  (0.8,0.2) {Q};
	\node[vertex,thick] (b) at  (0.8,3.8) {B};
    \node[vertex,thick] (d) at  (0.8,5) {D};
    \node[vertex,thick] (e) at  (2,5) {E};
	\node[vertex,thick] (i) at  (2,3.8) {I};
    \node[vertex,thick] (f) at  (-0.5,6.2) {F};
    \node[vertex,thick] (g) at  (.8,6.2) {G};
    \node[vertex,thick] (h) at  (2,6.2) {H};
	\draw[edge] (k) to (l);
	\draw[edge] (k) to (m);
	\draw[edge] (k) to (j);
	\draw[edge] (m) to (n);
	\draw[edge] (j) to (n);
	\draw[edge] (l) to (m);
	\draw[edge,fill = Cyan] (u) to (c1);
	\draw[edge,fill = Cyan] (u) to (c2);
	\draw[edge] (c1) to (y);
	\draw[edge] (c2) to (x);
	\draw[edge] (x) to (t);
	\draw[edge] (t) to (y);
	\draw[edge] (t) to (p);
	\draw[edge,] (p) to (q);
	\draw[edge] (d) to (b);
	\draw[edge] (c2) to (b);    
	\draw[edge] (e) to (b);
	\draw[edge] (e) to (i);
	\draw[edge] (f) to (d);
	\draw[edge] (g) to (e);
	\draw[edge] (e) to (i);
	\draw[edge] (h) to (e);
	\draw[edge] (j) to (x);
\end{tikzpicture}
}%
\subcaptionbox{}[.33\textwidth]{%
	\includegraphics[scale= 0.3]{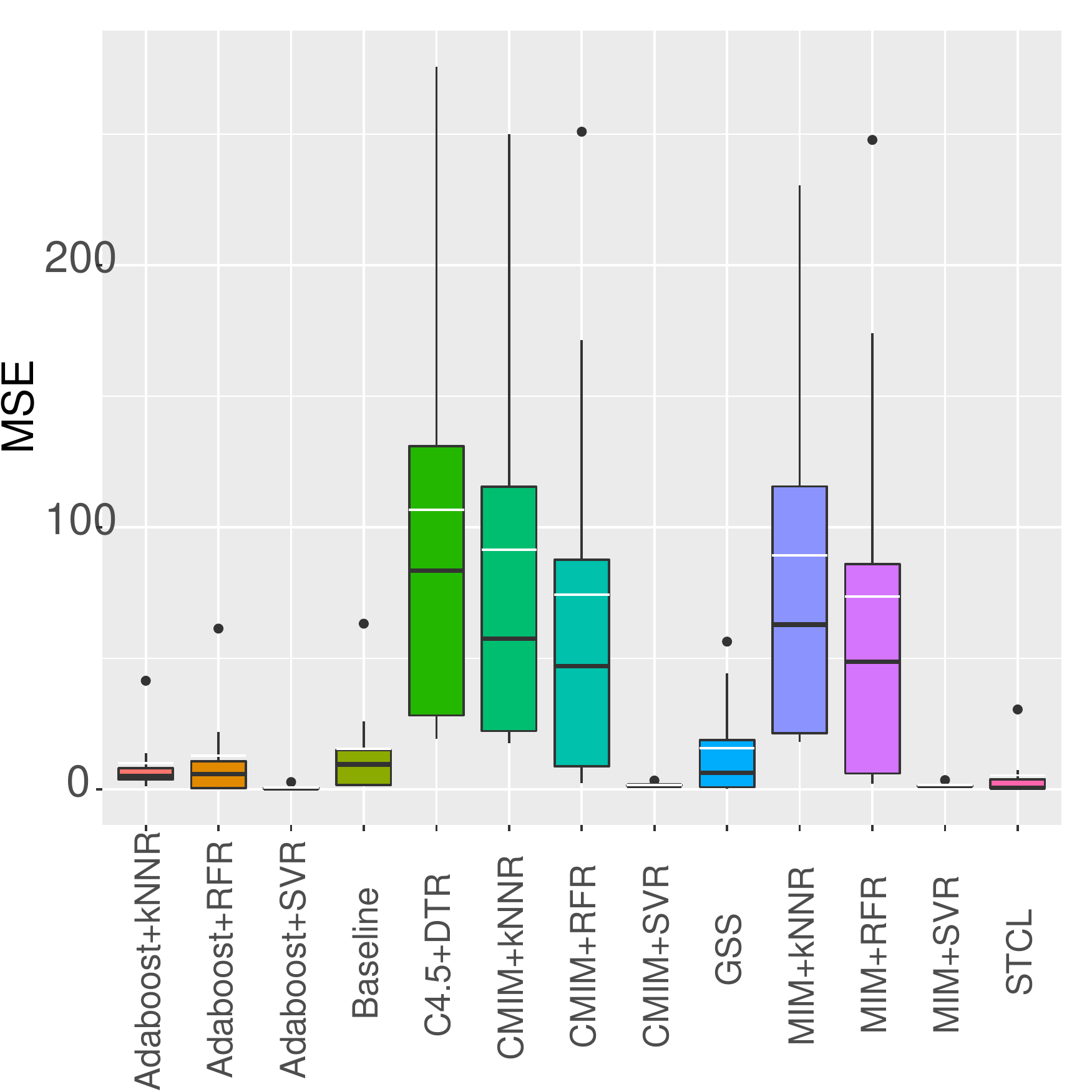}}
	\subcaptionbox{}[.33\textwidth]{%
	\includegraphics[scale= 0.3]{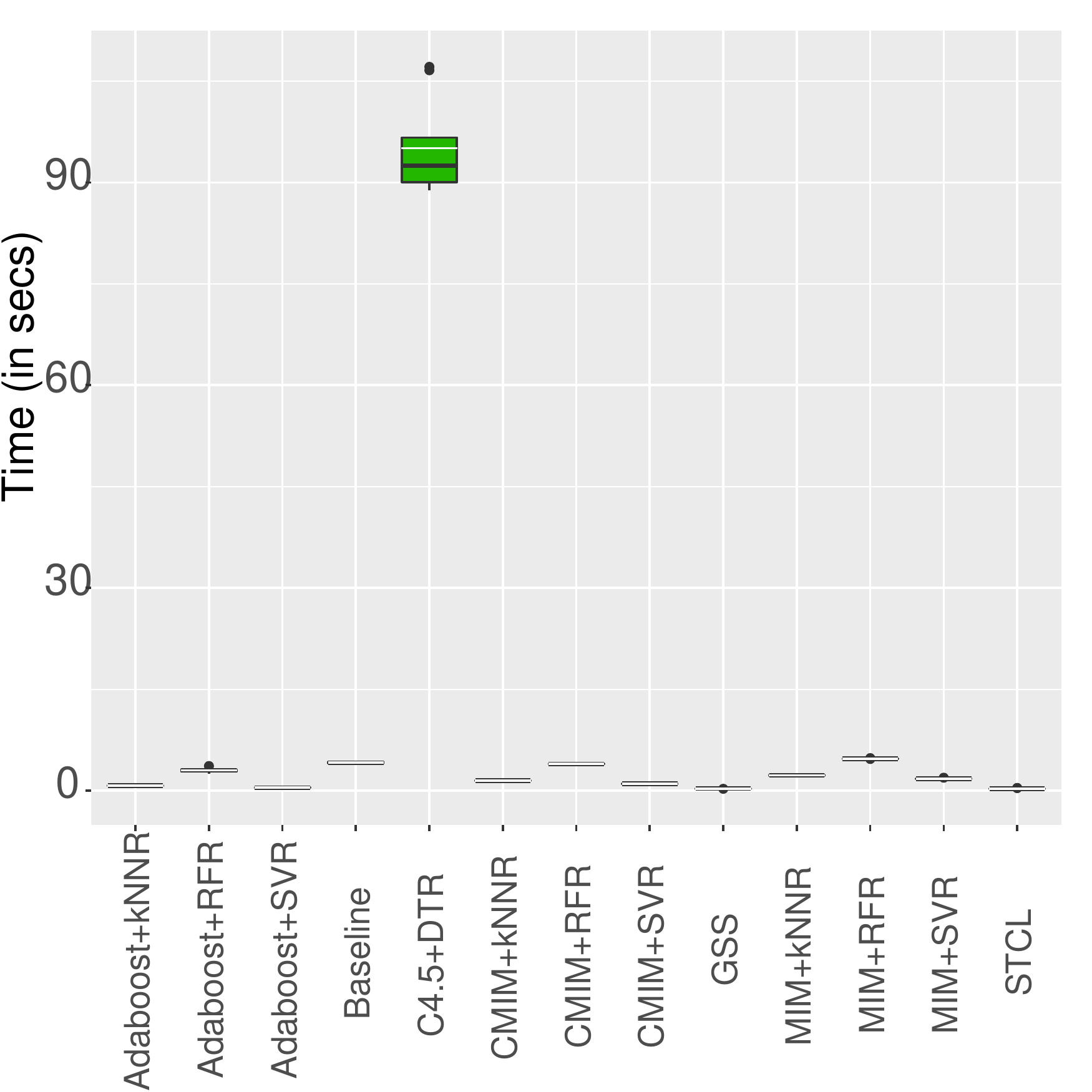}}
\caption{Results on (b), (c), (d) are generated over the ground truth graph (a) for the target variable \textcolor{orange}{$T$}. Data shift between source and target occurs by change in the probability distribution of the context variable \textcolor{green}{$C_1$} \& \textcolor{green}{$C_2$}, where sample size = 1000 for a Gaussian distribution.}	
\end{figure*}

\begin{figure*}
    \subcaptionbox{}[.33\textwidth]{%
	    \begin{tikzpicture}[transform shape,scale=.56]
	\tikzset{vertex/.style = {shape=circle,align=center,draw=black, fill=white}}
\tikzset{edge/.style = {->,> = latex',thick}}
	\node[vertex,thick] (l) at  (-5.3,5) {L};
	\node[vertex,thick] (k) at  (-3.7,5) {K};
	\node[vertex,thick] (j) at  (-3.7,3.8) {J};
	\node[vertex,thick](m) at  (-5.3,3.8) {M};
	\node[vertex,thick] (n) at  (-4.5,2.6) {N};
    \node[vertex,thick,fill=green] (c1) at  (-2.5,5) {c1};
    \node[vertex,thick,fill= gray!40,dashed] (u) at  (-1.5,6.2) {U};
    \node[vertex,thick,fill=green] (c2) at  (-0.5,5) {c2};
    \node[vertex,thick](x) at  (-0.5,3.8) {X};
    \node[vertex,thick,fill=Dandelion ] (t) at  (-0.5,1.4) {T};
    \node[vertex,thick ] (y) at  (-1.5,.2) {Y};
	\node[vertex,thick] (p) at  (0.8,1.4) {P};
	\node[vertex,thick] (q) at  (0.8,0.2) {Q};
	\node[vertex,thick] (b) at  (0.8,3.8) {B};
    \node[vertex,thick] (d) at  (0.8,5) {D};
    \node[vertex,thick] (e) at  (2,5) {E};
	\node[vertex,thick] (i) at  (2,3.8) {I};
    \node[vertex,thick] (f) at  (-0.5,6.2) {F};
    \node[vertex,thick] (g) at  (.8,6.2) {G};
    \node[vertex,thick] (h) at  (2,6.2) {H};
	\draw[edge] (k) to (l);
	\draw[edge] (k) to (m);
	\draw[edge] (k) to (j);
	\draw[edge] (m) to (n);
	\draw[edge] (j) to (n);
	\draw[edge] (l) to (m);
	\draw[edge,fill = Cyan] (u) to (c1);
	\draw[edge,fill = Cyan] (u) to (c2);
	\draw[edge] (c1) to (y);
	\draw[edge] (c2) to (x);
	\draw[edge] (x) to (t);
	\draw[edge] (t) to (y);
	\draw[edge] (t) to (p);
	\draw[edge,] (p) to (q);
	\draw[edge] (d) to (b);
	\draw[edge] (c2) to (b);    
	\draw[edge] (e) to (b);
	\draw[edge] (e) to (i);
	\draw[edge] (f) to (d);
	\draw[edge] (g) to (e);
	\draw[edge] (e) to (i);
	\draw[edge] (h) to (e);
	\draw[edge] (j) to (x);
\end{tikzpicture}

}%
\subcaptionbox{}[.33\textwidth]{%
	\includegraphics[scale= 0.3]{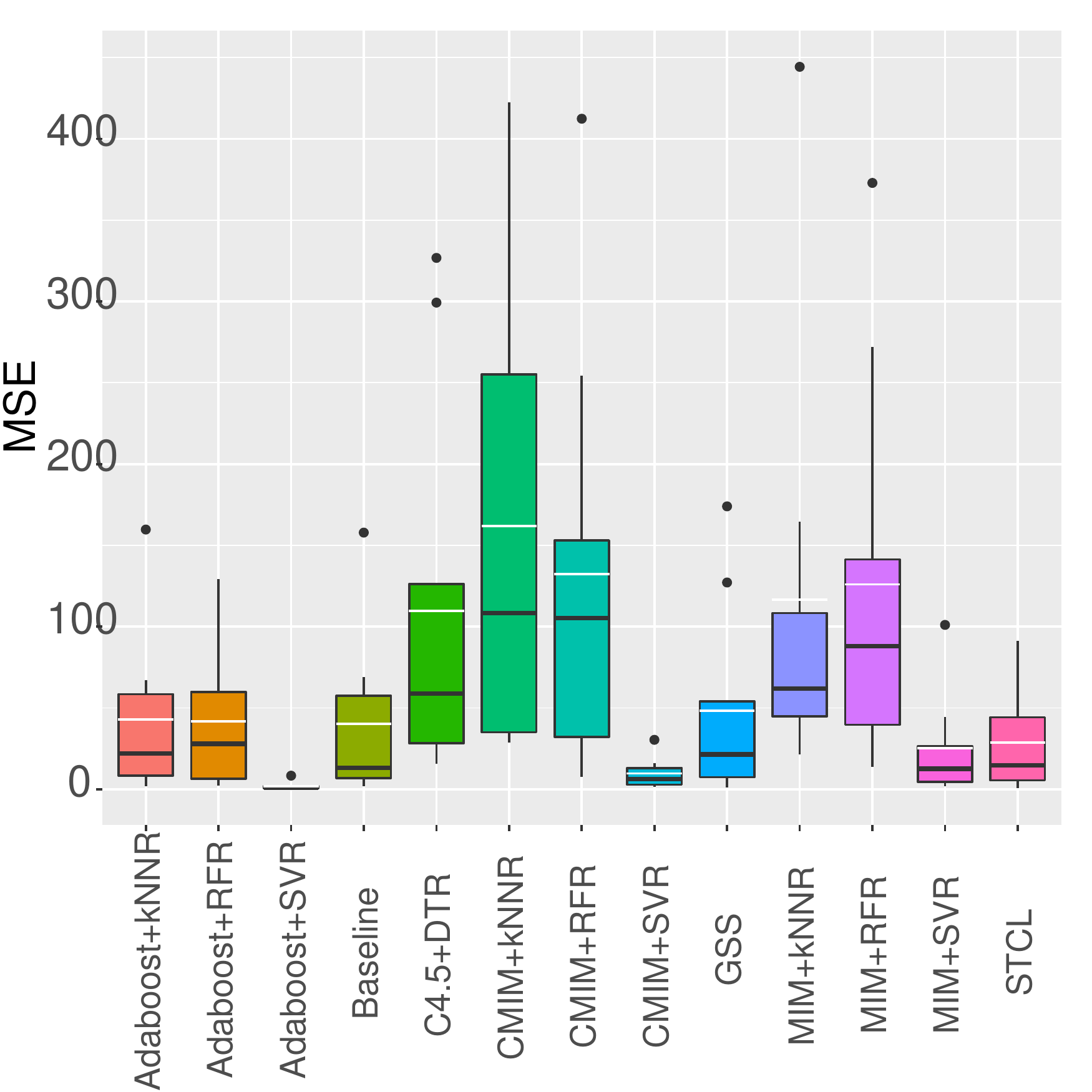}
    }%
	\subcaptionbox{}[.33\textwidth]{%
	\includegraphics[scale= 0.3]{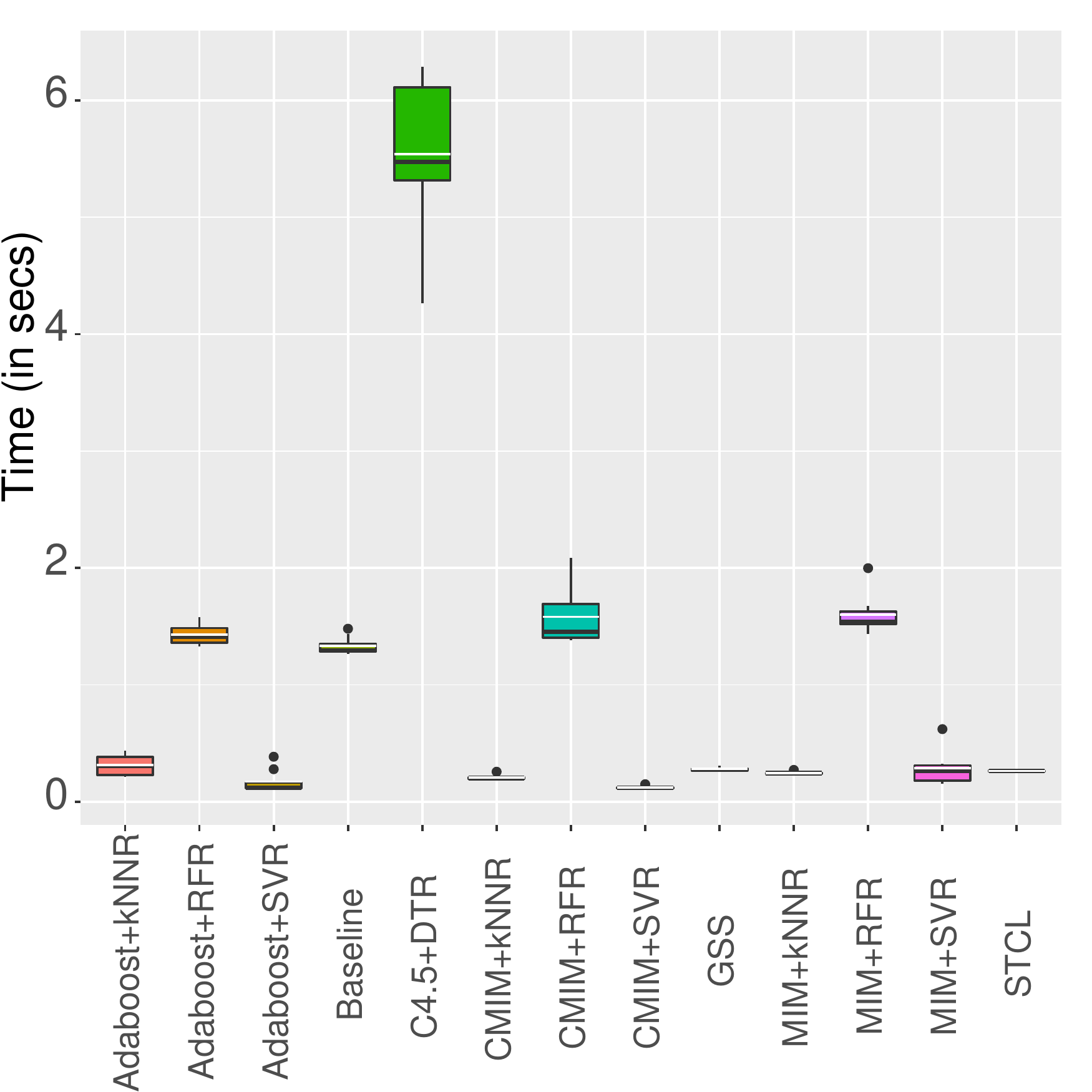}}
\caption{Results on (b), (c), (d) are generated over the ground truth graph (a) for the target variable \textcolor{orange}{$T$}. Data shift between source and target occurs by change in the distribution of the context variable \textcolor{green}{$C_1$} \& \textcolor{green}{$C_2$}, where sample size = 50 for a Gaussian distribution.}	
\end{figure*}

\begin{figure*}
    \subcaptionbox{}[.33\textwidth]{%
	    \begin{tikzpicture}[transform shape,scale=.56]
	\tikzset{vertex/.style = {shape=circle,align=center,draw=black, fill=white}}
\tikzset{edge/.style = {->,> = latex',thick}}
    \node[vertex,thick,fill=green] (c1) at  (-2.5,5) {c1};
    \node[vertex,thick,fill= gray!40,dashed] (u) at  (-1.5,6.2) {U};
    \node[vertex,thick] (c2) at  (-0.5,5) {c2};
    \node[vertex,thick](x) at  (-0.5,3.8) {X};
    \node[vertex,thick,fill=Dandelion ] (t) at  (-0.5,1.4) {T};
    \node[vertex,thick ] (y) at  (-1.5,.2) {Y};
	\node[vertex,thick] (p) at  (0.8,1.4) {P};
	\node[vertex,thick] (q) at  (0.8,0.2) {Q};
	\node[vertex,thick] (b) at  (0.8,3.8) {B};
    \node[vertex,thick] (d) at  (0.8,5) {D};
	\draw[edge,fill = Cyan] (u) to (c1);
	\draw[edge,fill = Cyan] (u) to (c2);
	\draw[edge] (c1) to (y);
	\draw[edge] (c2) to (x);
	\draw[edge] (x) to (t);
	\draw[edge] (t) to (y);
	\draw[edge] (t) to (p);
	\draw[edge,] (p) to (q);
	\draw[edge] (d) to (b);
	\draw[edge] (c2) to (b);    
\end{tikzpicture}

}%
\subcaptionbox{}[.33\textwidth]{%
	\includegraphics[scale= 0.3]{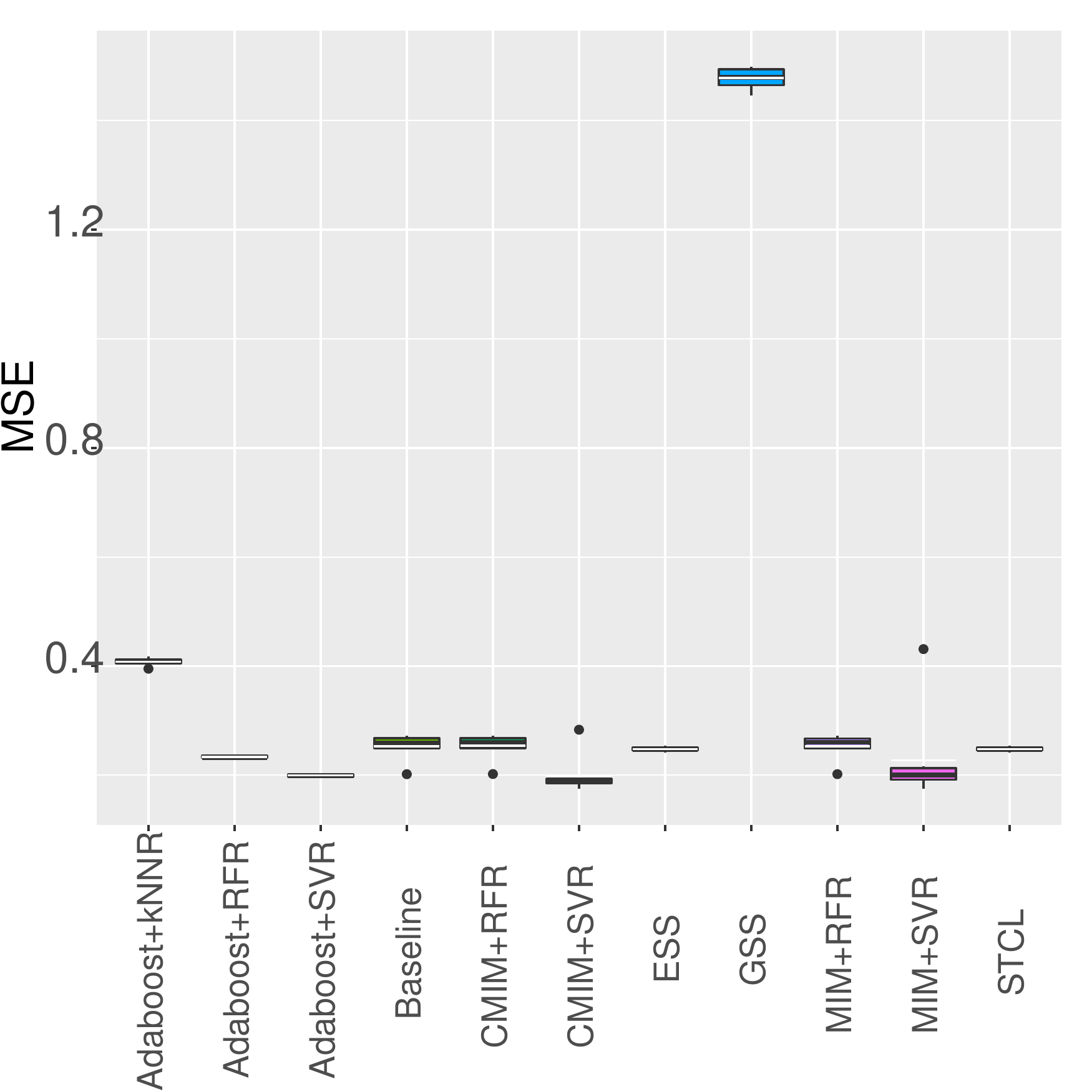}
    }%
	\subcaptionbox{}[.33\textwidth]{%
	\includegraphics[scale= 0.3]{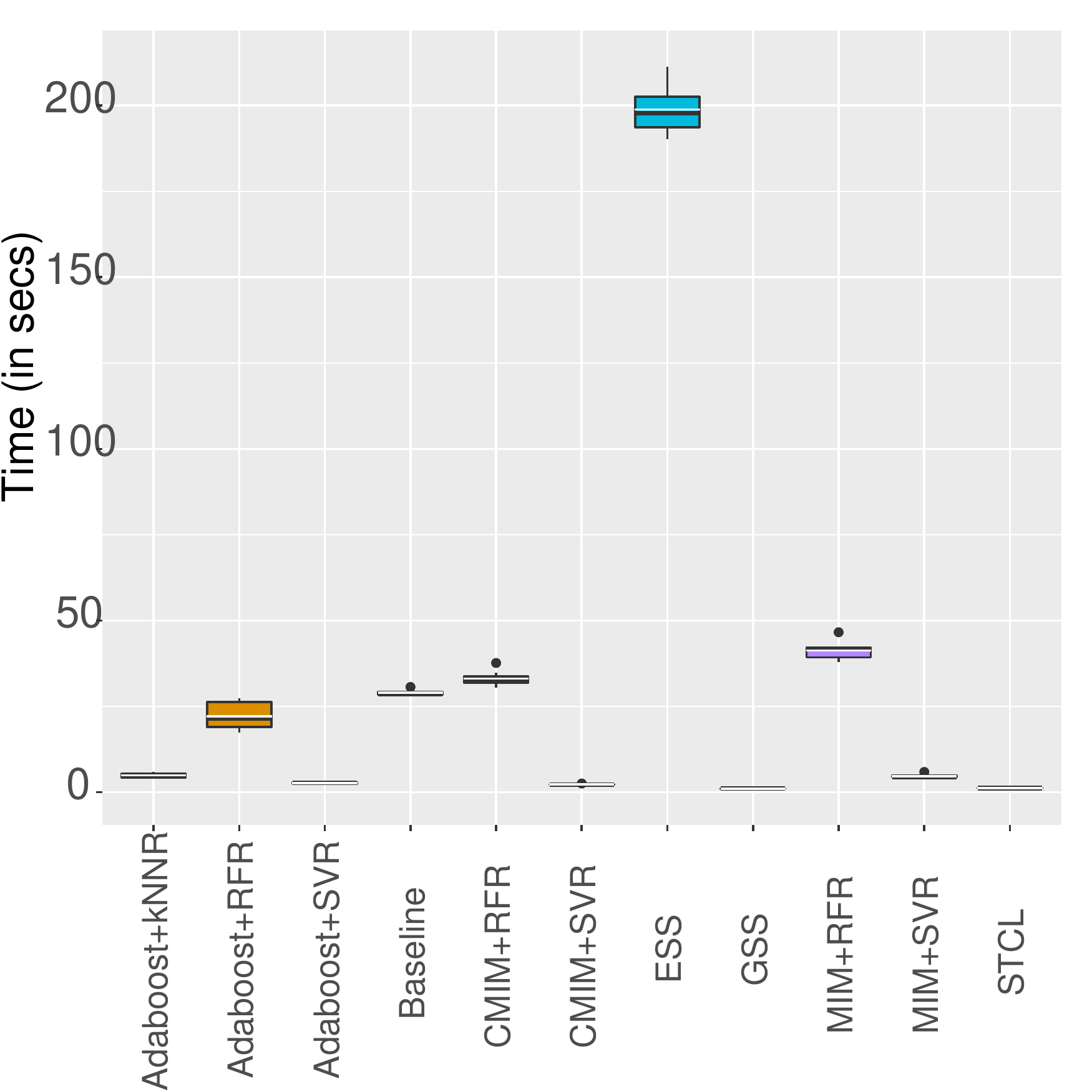}}
\caption{Results on (b), (c), (d) are generated over the ground truth graph (a) for the target variable \textcolor{orange}{$T$}. Data shift between source and target occurs by change in the distribution of the context variable \textcolor{green}{$C_1$}, where sample size = 10000 for a Gaussian distribution.}\label{fig:smallg}
\end{figure*}

\begin{figure*}
    \subcaptionbox{}[.33\textwidth]{%
	    \begin{tikzpicture}[transform shape,scale=.56]
	\tikzset{vertex/.style = {shape=circle,align=center,draw=black, fill=white}}
\tikzset{edge/.style = {->,> = latex',thick}}
    \node[vertex,thick,fill=green] (c1) at  (-2.5,5) {c1};
    \node[vertex,thick,fill= gray!40,dashed] (u) at  (-1.5,6.2) {U};
    \node[vertex,thick] (c2) at  (-0.5,5) {c2};
    \node[vertex,thick](x) at  (-0.5,3.8) {X};
    \node[vertex,thick,fill=Dandelion ] (t) at  (-0.5,1.4) {T};
    \node[vertex,thick ] (y) at  (-1.5,.2) {Y};
	\node[vertex,thick] (p) at  (0.8,1.4) {P};
	\node[vertex,thick] (q) at  (0.8,0.2) {Q};
	\node[vertex,thick] (b) at  (0.8,3.8) {B};
    \node[vertex,thick] (d) at  (0.8,5) {D};
	\draw[edge,fill = Cyan] (u) to (c1);
	\draw[edge,fill = Cyan] (u) to (c2);
	\draw[edge] (c1) to (y);
	\draw[edge] (c2) to (x);
	\draw[edge] (x) to (t);
	\draw[edge] (t) to (y);
	\draw[edge] (t) to (p);
	\draw[edge,] (p) to (q);
	\draw[edge] (d) to (b);
	\draw[edge] (c2) to (b);    
\end{tikzpicture}

}%
\subcaptionbox{}[.33\textwidth]{%
	\includegraphics[scale= 0.3]{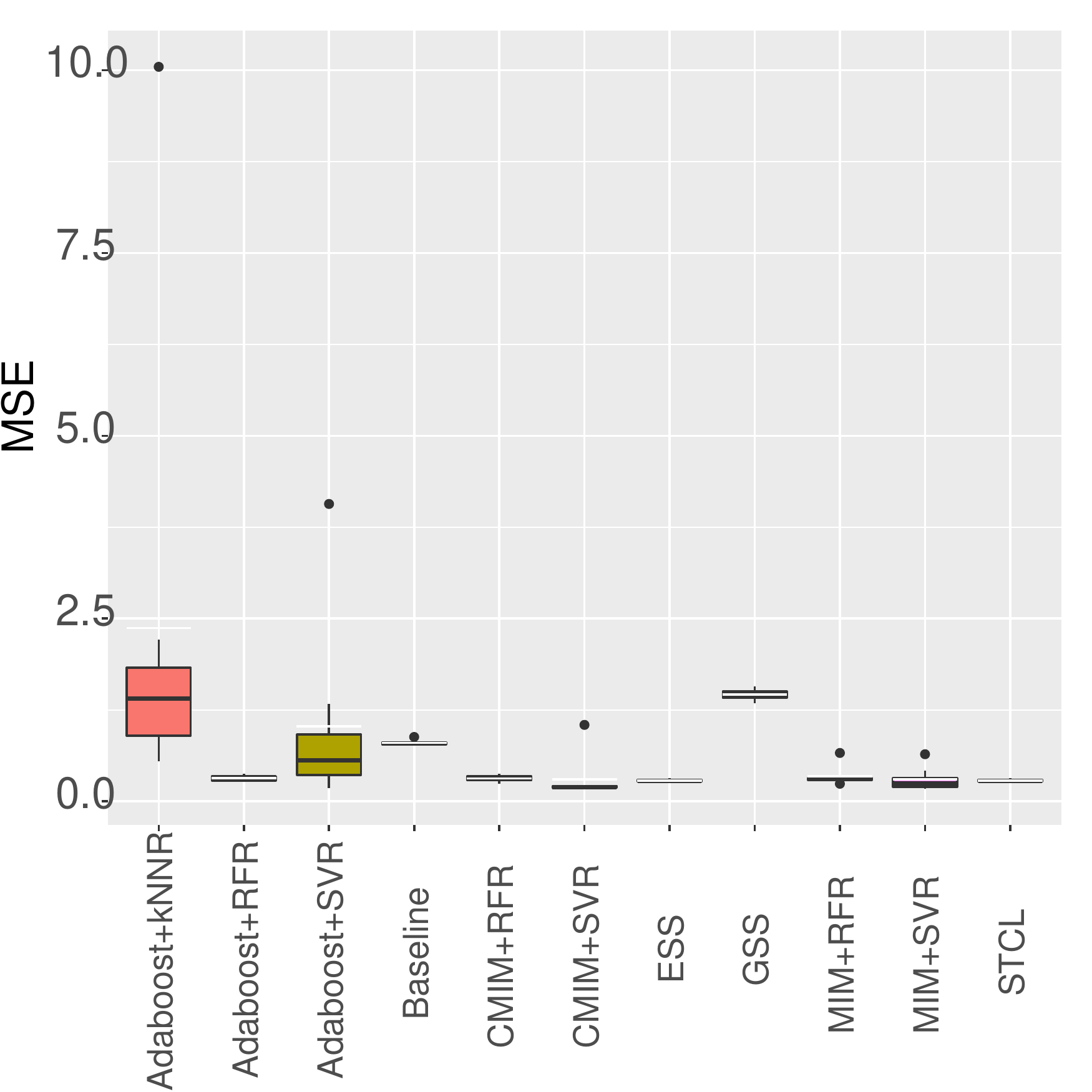}
    }
	\subcaptionbox{}[.33\textwidth]{%
	\includegraphics[scale= 0.3]{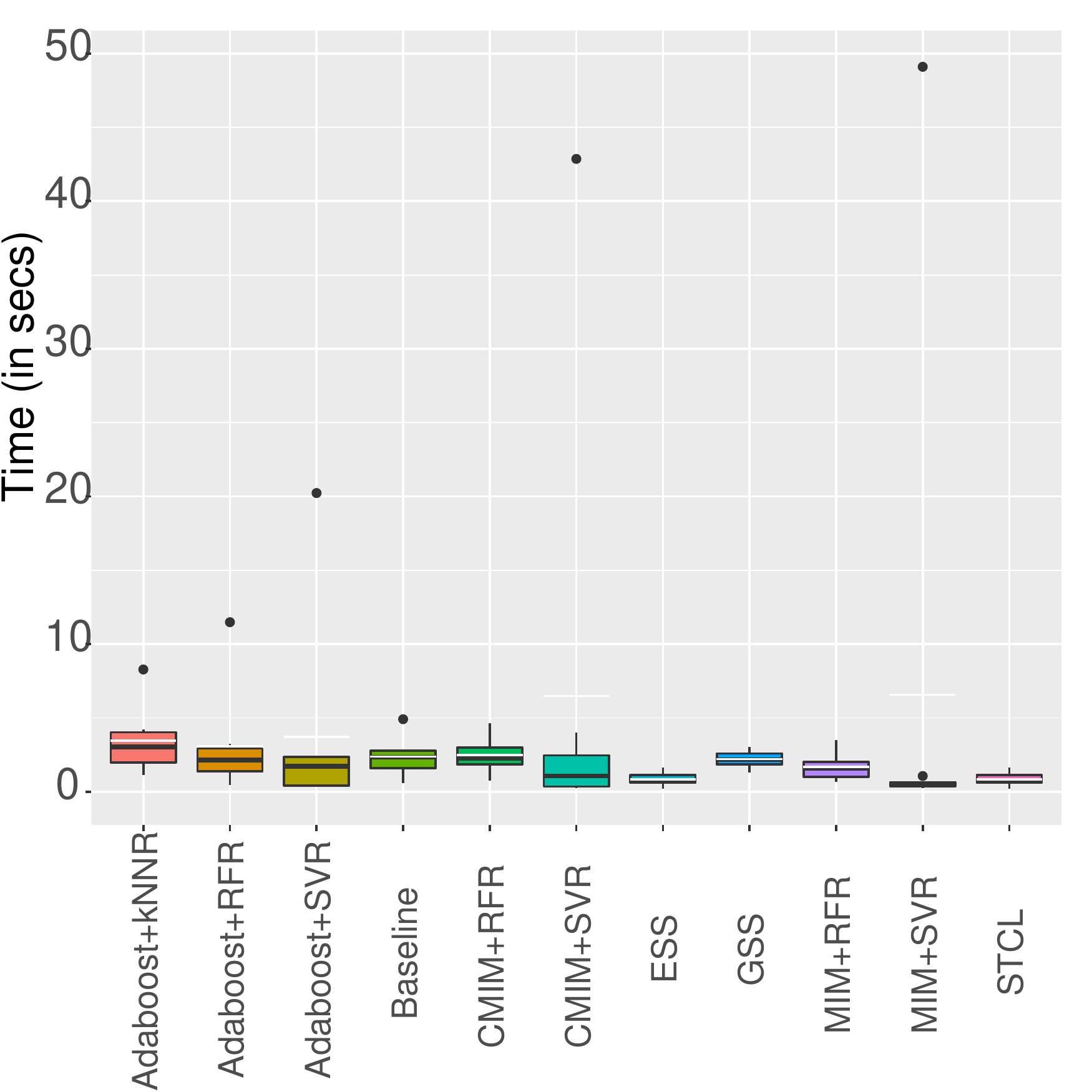}}
\caption{Results on (b), (c), (d) are generated over the ground truth graph (a) for the target variable \textcolor{orange}{$T$}. Data shift between source and target occurs by change in the distribution of the context variable \textcolor{green}{$C_1$}, where sample size = 1000 for a Gaussian distribution.}	
\end{figure*}

\begin{figure*}
    \subcaptionbox{}[.33\textwidth]{%
	    \begin{tikzpicture}[transform shape,scale=.56]
	\tikzset{vertex/.style = {shape=circle,align=center,draw=black, fill=white}}
\tikzset{edge/.style = {->,> = latex',thick}}
    \node[vertex,thick,fill=green] (c1) at  (-2.5,5) {c1};
    \node[vertex,thick,fill= gray!40,dashed] (u) at  (-1.5,6.2) {U};
    \node[vertex,thick] (c2) at  (-0.5,5) {c2};
    \node[vertex,thick](x) at  (-0.5,3.8) {X};
    \node[vertex,thick,fill=Dandelion ] (t) at  (-0.5,1.4) {T};
    \node[vertex,thick ] (y) at  (-1.5,.2) {Y};
	\node[vertex,thick] (p) at  (0.8,1.4) {P};
	\node[vertex,thick] (q) at  (0.8,0.2) {Q};
	\node[vertex,thick] (b) at  (0.8,3.8) {B};
    \node[vertex,thick] (d) at  (0.8,5) {D};
	\draw[edge,fill = Cyan] (u) to (c1);
	\draw[edge,fill = Cyan] (u) to (c2);
	\draw[edge] (c1) to (y);
	\draw[edge] (c2) to (x);
	\draw[edge] (x) to (t);
	\draw[edge] (t) to (y);
	\draw[edge] (t) to (p);
	\draw[edge,] (p) to (q);
	\draw[edge] (d) to (b);
	\draw[edge] (c2) to (b);    
\end{tikzpicture}

}%
\subcaptionbox{}[.33\textwidth]{%
	\includegraphics[scale= 0.3]{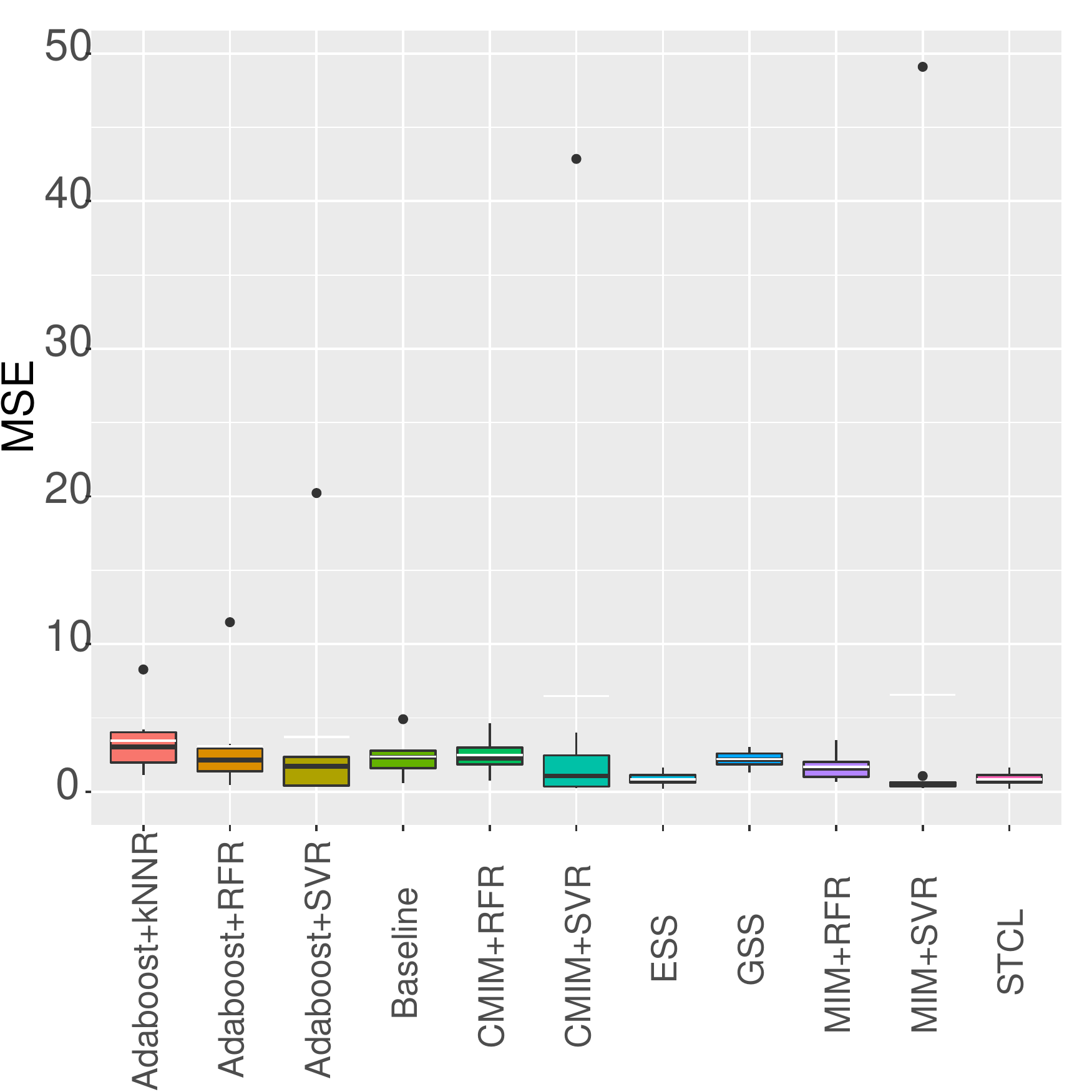}
    }%
	\subcaptionbox{}[.33\textwidth]{%
	\includegraphics[scale= 0.3]{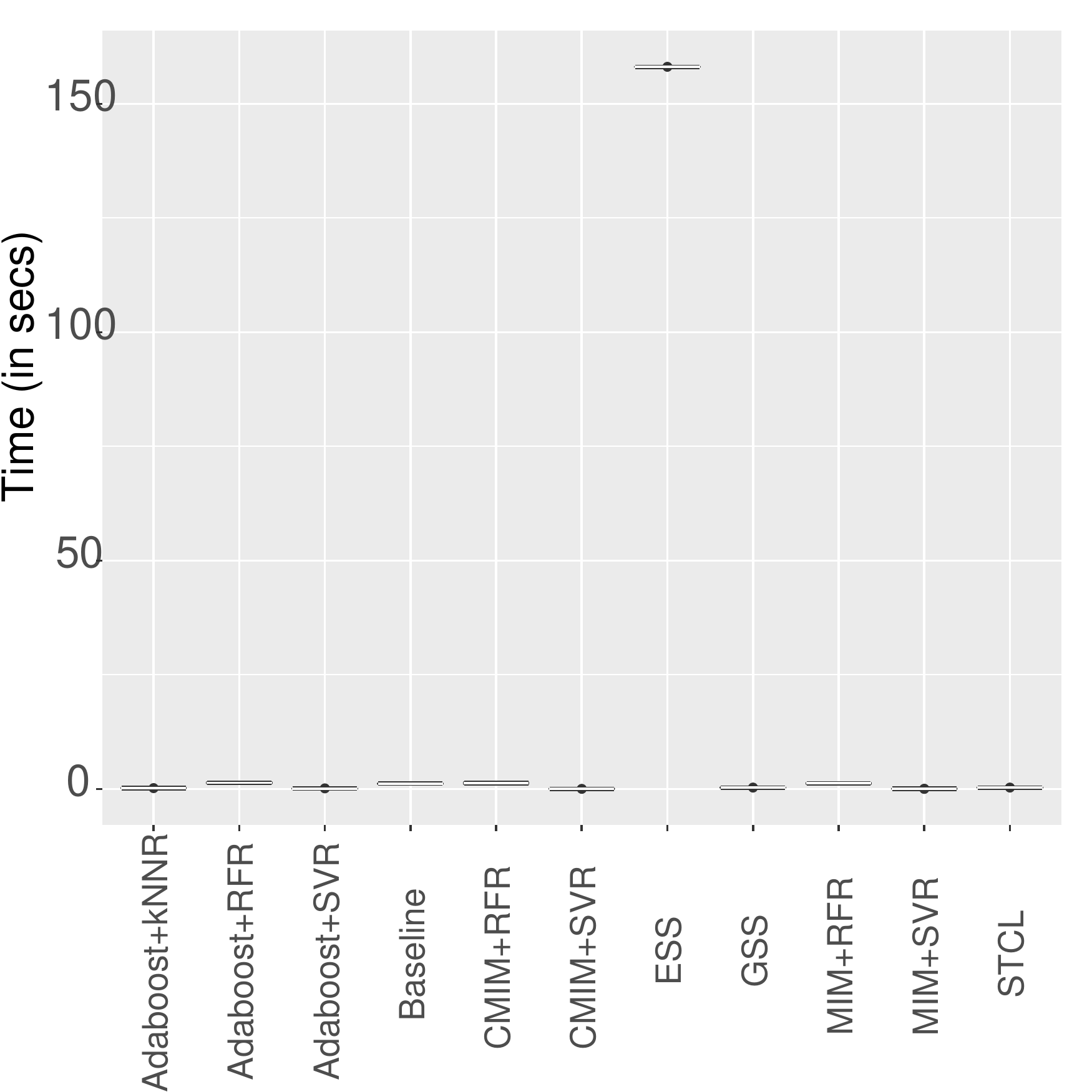}}
\caption{Results on (b), (c), (d) are generated over the ground truth graph (a) for the target variable \textcolor{orange}{$T$}. Data shift between source and target occurs by change in the distribution of the context variable \textcolor{green}{$C_1$}, where sample size = 50 for a Gaussian distribution.}	
\end{figure*}

\begin{figure*}
    \subcaptionbox{}[.33\textwidth]{%
	    \begin{tikzpicture}[transform shape,scale=.56]
	\tikzset{vertex/.style = {shape=circle,align=center,draw=black, fill=white}}
\tikzset{edge/.style = {->,> = latex',thick}}
	\node[vertex,thick] (l) at  (-5.3,5) {L};
	\node[vertex,thick] (k) at  (-3.7,5) {K};
	\node[vertex,thick] (j) at  (-3.7,3.8) {J};
	\node[vertex,thick](m) at  (-5.3,3.8) {M};
	\node[vertex,thick] (n) at  (-4.5,2.6) {N};
    \node[vertex,thick,fill=green] (c1) at  (-2.5,5) {c1};
    \node[vertex,thick,fill= gray!40,dashed] (u) at  (-1.5,6.2) {U};
    \node[vertex,thick] (z) at  (-1.4,2.75) {Z};
    \node[vertex,thick,fill=green] (c2) at  (-0.5,5) {c2};
    \node[vertex,thick](x) at  (-0.5,3.8) {X};
    \node[vertex,thick,fill=Dandelion ] (t) at  (-0.5,1.4) {T};
    \node[vertex,thick ] (y) at  (-1.5,.2) {Y};
	\node[vertex,thick] (p) at  (0.8,1.4) {P};
	\node[vertex,thick] (q) at  (0.8,0.2) {Q};
	\node[vertex,thick] (b) at  (0.8,3.8) {B};
    \node[vertex,thick] (d) at  (0.8,5) {D};
    \node[vertex,thick] (e) at  (2,5) {E};
	\node[vertex,thick] (i) at  (2,3.8) {I};
    \node[vertex,thick] (f) at  (-0.5,6.2) {F};
    \node[vertex,thick] (g) at  (.8,6.2) {G};
    \node[vertex,thick] (h) at  (2,6.2) {H};
	\draw[edge] (k) to (l);
	\draw[edge] (k) to (m);
	\draw[edge] (k) to (j);
	\draw[edge] (m) to (n);
	\draw[edge] (j) to (n);
	\draw[edge] (l) to (m);
	\draw[edge,fill = Cyan] (u) to (c1);
	\draw[edge,fill = Cyan] (u) to (c2);
	\draw[edge] (c1) to (y);
	\draw[edge] (c2) to (x);
	\draw[edge] (x) to (t);
	\draw[edge] (t) to (y);
	\draw[edge] (t) to (p);
	\draw[edge,] (p) to (q);
	\draw[edge] (d) to (b);
	\draw[edge] (c2) to (b);    
	\draw[edge] (e) to (b);
	\draw[edge] (e) to (i);
	\draw[edge] (f) to (d);
	\draw[edge] (g) to (e);
	\draw[edge] (e) to (i);
	\draw[edge] (h) to (e);
	\draw[edge] (j) to (x);
 	\draw[edge] (c1) to (z);
 	\draw[edge] (z) to (t);
\end{tikzpicture}

}%
\subcaptionbox{}[.33\textwidth]{%
	\includegraphics[scale= 0.3]{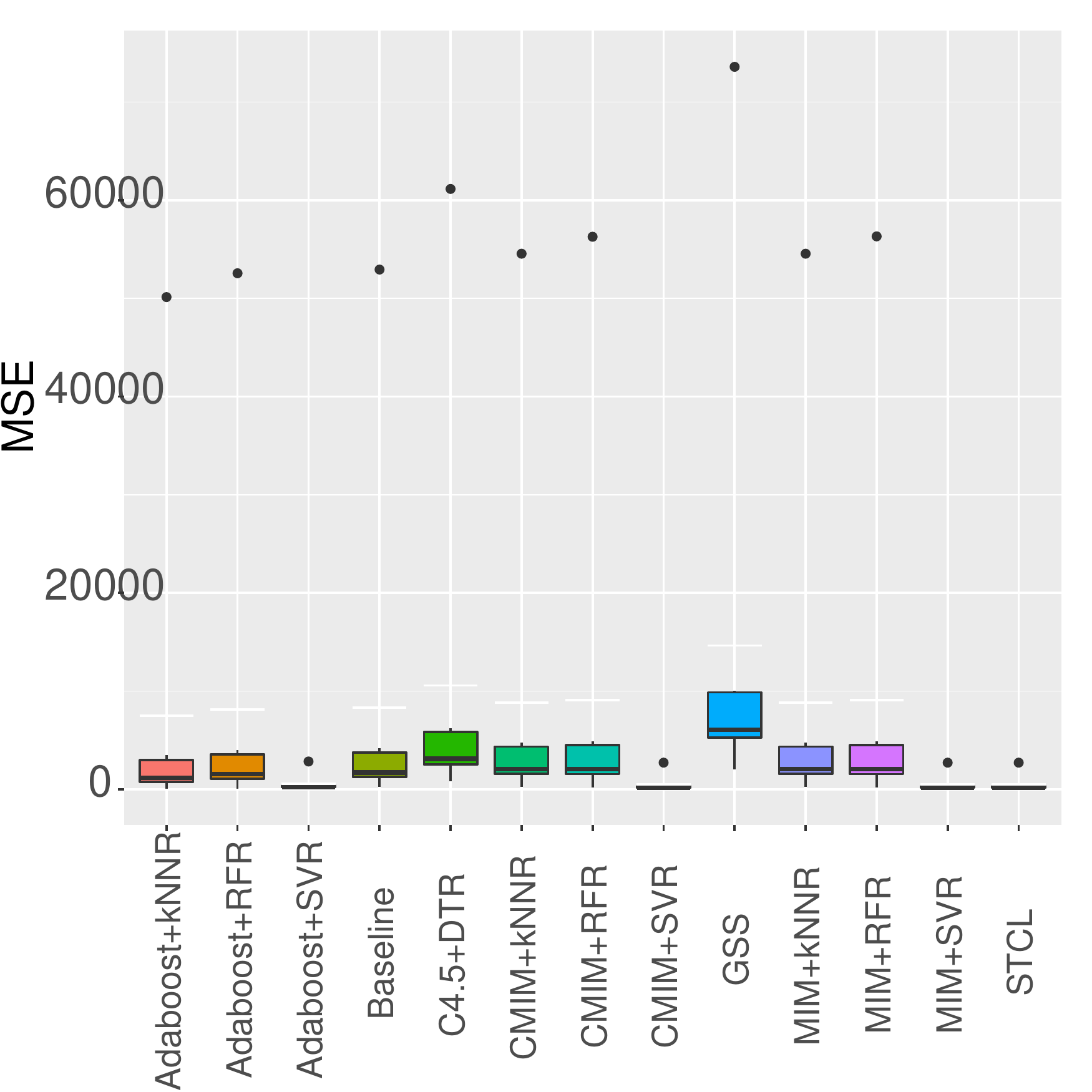}
    }%
	\subcaptionbox{}[.33\textwidth]{%
	\includegraphics[scale= 0.3]{images/Graph2_plots/Plot_Graph_G2_in_c1_change_Gaussian_medium_sample_size_mse.pdf}}
\caption{Results on (b), (c), (d) are generated over the ground truth graph (a) for the target variable \textcolor{orange}{$T$}. Data shift between source and target occurs by change in the distribution of the context variable \textcolor{green}{$C_1$} , where sample size = 1000 for a Gaussian distribution.}	
\end{figure*}

\begin{figure*}
    \subcaptionbox{}[.33\textwidth]{%
	    \begin{tikzpicture}[transform shape,scale=.56]
	\tikzset{vertex/.style = {shape=circle,align=center,draw=black, fill=white}}
\tikzset{edge/.style = {->,> = latex',thick}}
	\node[vertex,thick] (l) at  (-5.3,5) {L};
	\node[vertex,thick] (k) at  (-3.7,5) {K};
	\node[vertex,thick] (j) at  (-3.7,3.8) {J};
	\node[vertex,thick](m) at  (-5.3,3.8) {M};
	\node[vertex,thick] (n) at  (-4.5,2.6) {N};
    \node[vertex,thick,fill=green] (c1) at  (-2.5,5) {c1};
    \node[vertex,thick,fill= gray!40,dashed] (u) at  (-1.5,6.2) {U};
    \node[vertex,thick] (z) at  (-1.4,2.75) {Z};
    \node[vertex,thick,fill=green] (c2) at  (-0.5,5) {c2};
    \node[vertex,thick](x) at  (-0.5,3.8) {X};
    \node[vertex,thick,fill=Dandelion ] (t) at  (-0.5,1.4) {T};
    \node[vertex,thick ] (y) at  (-1.5,.2) {Y};
	\node[vertex,thick] (p) at  (0.8,1.4) {P};
	\node[vertex,thick] (q) at  (0.8,0.2) {Q};
	\node[vertex,thick] (b) at  (0.8,3.8) {B};
    \node[vertex,thick] (d) at  (0.8,5) {D};
    \node[vertex,thick] (e) at  (2,5) {E};
	\node[vertex,thick] (i) at  (2,3.8) {I};
    \node[vertex,thick] (f) at  (-0.5,6.2) {F};
    \node[vertex,thick] (g) at  (.8,6.2) {G};
    \node[vertex,thick] (h) at  (2,6.2) {H};
	\draw[edge] (k) to (l);
	\draw[edge] (k) to (m);
	\draw[edge] (k) to (j);
	\draw[edge] (m) to (n);
	\draw[edge] (j) to (n);
	\draw[edge] (l) to (m);
	\draw[edge,fill = Cyan] (u) to (c1);
	\draw[edge,fill = Cyan] (u) to (c2);
	\draw[edge] (c1) to (y);
	\draw[edge] (c2) to (x);
	\draw[edge] (x) to (t);
	\draw[edge] (t) to (y);
	\draw[edge] (t) to (p);
	\draw[edge,] (p) to (q);
	\draw[edge] (d) to (b);
	\draw[edge] (c2) to (b);    
	\draw[edge] (e) to (b);
	\draw[edge] (e) to (i);
	\draw[edge] (f) to (d);
	\draw[edge] (g) to (e);
	\draw[edge] (e) to (i);
	\draw[edge] (h) to (e);
	\draw[edge] (j) to (x);
 	\draw[edge] (c1) to (z);
 	\draw[edge] (z) to (t);
\end{tikzpicture}
    \vspace{2\baselineskip}

}%
\subcaptionbox{}[.33\textwidth]{%
	\includegraphics[scale= 0.3]{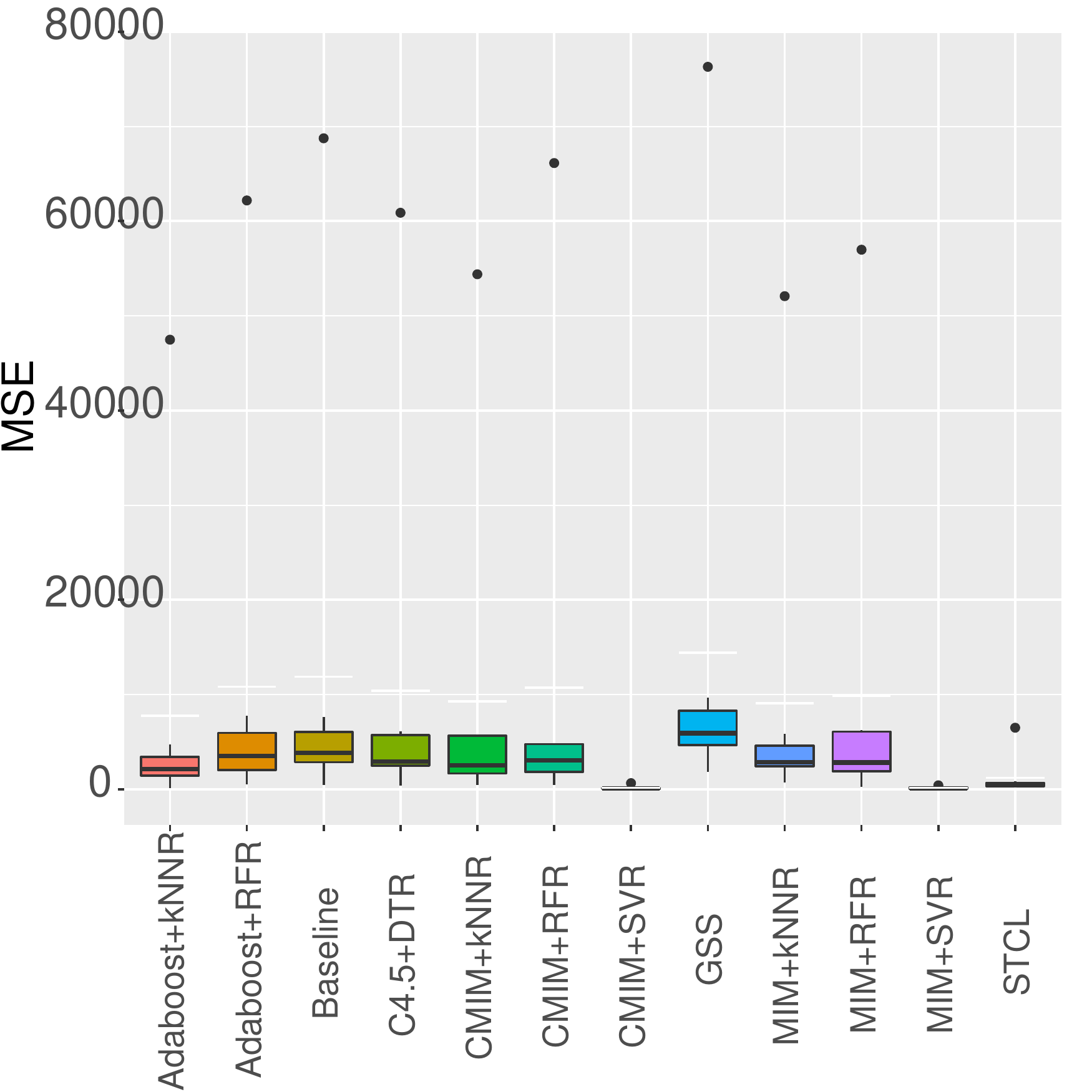}
    }%
	\subcaptionbox{}[.33\textwidth]{%
	\includegraphics[scale= 0.3]{images/Graph2_plots/Plot_Graph_G2_in_c1_change_Gaussian_low_sample_size_mse.pdf}}
\caption{Results on (b), (c), (d) are generated over the ground truth graph (a) for the target variable \textcolor{orange}{$T$}. Data shift between source and target occurs by change in the distribution of the context variable \textcolor{green}{$C_1$} , where sample size = 50 for a Gaussian distribution.}	
\end{figure*}

\begin{figure*}
    \centering
    \begin{minipage}{\textwidth}
  \begin{minipage}[b]{0.49\textwidth}
    \centering
    \begin{tikzpicture}[transform shape,scale=.60]
	\tikzset{vertex/.style = {shape=circle,align=center,draw=black, fill=white}}
\tikzset{edge/.style = {->,> = latex',thick}}
	\node[vertex,thick] (l) at  (-5.3,5) {L};
	\node[vertex,thick] (k) at  (-3.7,5) {K};
	\node[vertex,thick] (j) at  (-3.7,3.8) {J};
	\node[vertex,thick](m) at  (-5.3,3.8) {M};
	\node[vertex,thick] (n) at  (-4.5,2.6) {N};
    \node[vertex,thick,fill=green] (c1) at  (-2.5,5) {c1};
    \node[vertex,thick,fill= gray!40,dashed] (u) at  (-1.5,6.2) {U};
    \node[vertex,thick] (z) at  (-1.4,2.75) {Z};
    \node[vertex,thick] (c2) at  (-0.5,5) {c2};
    \node[vertex,thick](x) at  (-0.5,3.8) {X};
    \node[vertex,thick,fill=Dandelion ] (t) at  (-0.5,1.4) {T};
    \node[vertex,thick ] (y) at  (-1.5,.2) {Y};
	\node[vertex,thick] (p) at  (0.8,1.4) {P};
	\node[vertex,thick] (q) at  (0.8,0.2) {Q};
	\node[vertex,thick] (b) at  (0.8,3.8) {B};
    \node[vertex,thick] (d) at  (0.8,5) {D};
    \node[vertex,thick] (e) at  (2,5) {E};
	\node[vertex,thick] (i) at  (2,3.8) {I};
    \node[vertex,thick] (f) at  (0.8,6.2) {F};
    \node[vertex,thick] (g) at  (2,6.2) {G};
    \node[vertex,thick] (h) at  (3.2,6.2) {H};
	\draw[edge] (k) to (l);
	\draw[edge] (k) to (m);
	\draw[edge] (k) to (j);
	\draw[edge] (m) to (n);
	\draw[edge] (j) to (n);
	\draw[edge] (l) to (m);
	\draw[edge,fill = Cyan] (u) to (c1);
	\draw[edge,fill = Cyan] (u) to (c2);
	\draw[edge] (c1) to (y);
	\draw[edge] (c2) to (x);
	\draw[edge] (x) to (t);
	\draw[edge] (t) to (y);
	\draw[edge] (t) to (p);
	\draw[edge,] (p) to (q);
	\draw[edge] (d) to (b);
	\draw[edge] (c2) to (b);    
	\draw[edge] (e) to (b);
	\draw[edge] (e) to (i);
	\draw[edge] (f) to (d);
	\draw[edge] (g) to (e);
	\draw[edge] (e) to (i);
	\draw[edge] (h) to (e);
	\draw[edge] (j) to (x);
 	\draw[edge] (c1) to (z);
 	\draw[edge] (z) to (t);
\end{tikzpicture}
\subcaption{}
  \end{minipage}
  \hfill
  \begin{minipage}[b]{0.49\textwidth}
    \centering
    \scriptsize\begin{tabular}{l c } 
\hline\hline
Methodology & p-value \\ [0.5ex] 
\hline
Baseline & \textbf{0.1694} \\
GSS & \textbf{0.19617} \\ 
CMIM+SVR & \textbf{0.3187} \\
CMIM+kNNR & \textbf{0.1258} \\
CMIM+RFR & \textbf{0.1376}  \\
MIM+SVR &\textbf{ 0.318705}  \\
MIM+kNNR &\textbf{ 0.1258}  \\
MIM+RFR & \textbf{0.13759} \\
Adaboost+SVR &\textbf{ 0.279}  \\
Adaboost+kNNR &\textbf{ 0.20346}\\
Adaboost+RFR &\textbf{ 0.169} \\
C4.5+ DTR & \textbf{0.0889} \\
\hline 
\end{tabular}
      \caption*{(b)}
    \end{minipage}
  \end{minipage}
    \caption{ Results in the table are p-values generated by performing a T-test on SSE over the ground truth graph (a) for the target variable \textcolor{orange}{T}.  Data shift between source and target occurs by change in the distribution of the context variable \textcolor{green}{$C_1$}, over a Gaussian distribution with sample size = 1000. The bold results indicate that the average performance of \rctl~is not significantly different from the average performance of other approaches for the p-value 0.05.}\label{ttest1}
\end{figure*}

\begin{figure*}
    \centering
    \begin{minipage}{\textwidth}
  \begin{minipage}[b]{0.49\textwidth}
    \centering
    \begin{tikzpicture}[transform shape,scale=.6]
	\tikzset{vertex/.style = {shape=circle,align=center,draw=black, fill=white}}
\tikzset{edge/.style = {->,> = latex',thick}}
    \node[vertex,thick,fill=green] (c1) at  (-2.5,5) {c1};
    \node[vertex,thick,fill= gray!40,dashed] (u) at  (-1.5,6.2) {U};
    \node[vertex,thick,fill=green] (c2) at  (-0.5,5) {c2};
    \node[vertex,thick](x) at  (-0.5,3.8) {X};
    \node[vertex,thick,fill=Dandelion ] (t) at  (-0.5,1.4) {T};
    \node[vertex,thick ] (y) at  (-1.5,.2) {Y};
	\node[vertex,thick] (p) at  (0.8,1.4) {P};
	\node[vertex,thick] (q) at  (0.8,0.2) {Q};
	\node[vertex,thick] (b) at  (0.8,3.8) {B};
    \node[vertex,thick] (d) at  (0.8,5) {D};
	\draw[edge,fill = Cyan] (u) to (c1);
	\draw[edge,fill = Cyan] (u) to (c2);
	\draw[edge] (c1) to (y);
	\draw[edge] (c2) to (x);
	\draw[edge] (x) to (t);
	\draw[edge] (t) to (y);
	\draw[edge] (t) to (p);
	\draw[edge,] (p) to (q);
	\draw[edge] (d) to (b);
	\draw[edge] (c2) to (b);    
\end{tikzpicture}
\subcaption{}
  \end{minipage}
  \hfill
  \begin{minipage}[b]{0.49\textwidth}
    \centering
\scriptsize\begin{tabular}{l c } 
\hline\hline 
Methodology & p-value \\ [0.5ex] 
\hline
Baseline & \textbf{0.1535} \\
GSS & \textbf{0.1681} \\ 
CMIM+SVR &\textbf{ 0.2712} \\
CMIM+kNNR & \textbf{0.1224} \\
CMIM+RFR & \textbf{0.1885}  \\
MIM+SVR & \textbf{0.65673}  \\
MIM+kNNR & 0.04307 \\
MIM+RFR &\textbf{ 0.16345} \\
Adaboost+SVR & \textbf{0.2660}  \\
Adaboost+kNNR & \textbf{0.0526} \\
Adaboost+RFR & \textbf{0.1185}\\
C4.5+ DTR & 0.0295 \\
ESS & 0.0222 \\
\hline 
\end{tabular}
 \caption*{(b)}
    \end{minipage}
  \end{minipage}
    \caption{ Results in the table are p-values generated by performing a T-test on MSE over the ground truth graph (a) for the target variable \textcolor{orange}{T}.  Data shift between source and target occurs by change in the distribution of the context variable \textcolor{green}{$C_1$} \& \textcolor{green}{$C_2$} , over a Gaussian distribution with sample size = 1000. The  bold  results  indicate  that  the average performance of~\rctl~ is not significantly different from the average performance of other approaches for the p-value 0.05.}
\end{figure*}

\begin{figure*}
    \centering
    \begin{minipage}{\textwidth}
  \begin{minipage}[b]{0.49\textwidth}
    \centering
    \begin{tikzpicture}[transform shape,scale=.6]
	\tikzset{vertex/.style = {shape=circle,align=center,draw=black, fill=white}}
\tikzset{edge/.style = {->,> = latex',thick}}
	\node[vertex,thick] (l) at  (-5.3,5) {L};
	\node[vertex,thick] (k) at  (-3.7,5) {K};
	\node[vertex,thick] (j) at  (-3.7,3.8) {J};
	\node[vertex,thick](m) at  (-5.3,3.8) {M};
	\node[vertex,thick] (n) at  (-4.5,2.6) {N};
    \node[vertex,thick,fill=green] (c1) at  (-2.5,5) {c1};
    \node[vertex,thick,fill= gray!40,dashed] (u) at  (-1.5,6.2) {U};
    \node[vertex,thick] (c2) at  (-0.5,5) {c2};
    \node[vertex,thick](x) at  (-0.5,3.8) {X};
    \node[vertex,thick,fill=Dandelion ] (t) at  (-0.5,1.4) {T};
    \node[vertex,thick ] (y) at  (-1.5,.2) {Y};
	\node[vertex,thick] (p) at  (0.8,1.4) {P};
	\node[vertex,thick] (q) at  (0.8,0.2) {Q};
	\node[vertex,thick] (b) at  (0.8,3.8) {B};
    \node[vertex,thick] (d) at  (0.8,5) {D};
    \node[vertex,thick] (e) at  (2,5) {E};
	\node[vertex,thick] (i) at  (2,3.8) {I};
    \node[vertex,thick] (f) at  (0.8,6.2) {F};
    \node[vertex,thick] (g) at  (2,6.2) {G};
    \node[vertex,thick] (h) at  (3.2,6.2) {H};
	\draw[edge] (k) to (l);
	\draw[edge] (k) to (m);
	\draw[edge] (k) to (j);
	\draw[edge] (m) to (n);
	\draw[edge] (j) to (n);
	\draw[edge] (l) to (m);
	\draw[edge,fill = Cyan] (u) to (c1);
	\draw[edge,fill = Cyan] (u) to (c2);
	\draw[edge] (c1) to (y);
	\draw[edge] (c2) to (x);
	\draw[edge] (x) to (t);
	\draw[edge] (t) to (y);
	\draw[edge] (t) to (p);
	\draw[edge,] (p) to (q);
	\draw[edge] (d) to (b);
	\draw[edge] (c2) to (b);    
	\draw[edge] (e) to (b);
	\draw[edge] (e) to (i);
	\draw[edge] (f) to (d);
	\draw[edge] (g) to (e);
	\draw[edge] (e) to (i);
	\draw[edge] (h) to (e);
	\draw[edge] (j) to (x);
\end{tikzpicture}
\subcaption{}
  \end{minipage}
  \hfill
  \begin{minipage}[b]{0.49\textwidth}
    \centering
\scriptsize\begin{tabular}{l c } 
\hline\hline 
Methodology & p-value \\ [0.5ex] 
\hline
Baseline & \textbf{0.01747} \\
GSS & \textbf{6.14934e-08} \\ 
CMIM+SVR &\textbf{ 2.9747e-08} \\
CMIM+kNNR &\textbf{ 0.00445}\\
CMIM+RFR &\textbf{ 1.7186e-17 } \\
MIM+SVR & \textbf{5.415e-07} \\
MIM+kNNR & \textbf{0.0057} \\
MIM+RFR & \textbf{8.48536e-17 }\\
Adaboost+SVR & \textbf{0.0001} \\
Adaboost+kNNR & \textbf{1.0681e-07} \\
Adaboost+RFR & 0.16304\\
C4.5+ DTR & \textbf{4.2182e-08} \\
\hline 
\end{tabular}
\caption*{(b)}
    \end{minipage}
  \end{minipage}
    \caption{ Results in the table are p-values generated by performing a T-test on MSE over the ground truth graph (a) for the target variable \textcolor{orange}{T}.  Data shift between source and target occurs by change in the distribution of the context variable \textcolor{green}{$C_1$} , over a Gaussian distribution with sample size = 1000. The  bold  results  indicate  that  the average performance of~\rctl~ is significantly different from the average performance of other approaches for the p-value 0.05.}
\end{figure*}

\begin{figure*}
    \centering
    \begin{minipage}{\textwidth}
  \begin{minipage}[b]{0.49\textwidth}
    \centering
    \begin{tikzpicture}[transform shape,scale=.6]
	\tikzset{vertex/.style = {shape=circle,align=center,draw=black, fill=white}}
\tikzset{edge/.style = {->,> = latex',thick}}
	   \node[vertex,thick,fill=green] (c1) at  (-2.5,5) {c1};
    \node[vertex,thick,fill= gray!40,dashed] (u) at  (-1.5,6.2) {U};
    \node[vertex,thick,fill=green] (c2) at  (-0.5,5) {c2};
    \node[vertex,thick](x) at  (-0.5,3.8) {X};
    \node[vertex,thick,fill=Dandelion ] (t) at  (-0.5,1.4) {T};
    \node[vertex,thick ] (y) at  (-1.5,.2) {Y};
	\node[vertex,thick] (p) at  (0.8,1.4) {P};
	\node[vertex,thick] (q) at  (0.8,0.2) {Q};
	\node[vertex,thick] (b) at  (0.8,3.8) {B};
    \node[vertex,thick] (d) at  (0.8,5) {D};
	\draw[edge,fill = Cyan] (u) to (c1);
	\draw[edge,fill = Cyan] (u) to (c2);
	\draw[edge] (c1) to (y);
	\draw[edge] (c2) to (x);
	\draw[edge] (x) to (t);
	\draw[edge] (t) to (y);
	\draw[edge] (t) to (p);
	\draw[edge,] (p) to (q);
	\draw[edge] (d) to (b);
	\draw[edge] (c2) to (b); 
\end{tikzpicture}
\subcaption{}
  \end{minipage}
  \hfill
  \begin{minipage}[b]{0.49\textwidth}
    \centering
\scriptsize\begin{tabular}{l c } 
\hline\hline 
Methodology & p-value \\ [0.5ex] 
\hline
Baseline & \textbf{2.5901e-09} \\
GSS &\textbf{ 9.1341e-06} \\ 
CMIM+SVR & 0.7350 \\
CMIM+kNNR & 0.1154\\
CMIM+RFR &\textbf{ 0.0265} \\
MIM+SVR & 0.94294 \\
MIM+kNNR & 0.1687 \\
MIM+RFR &\textbf{0.0190} \\
Adaboost+SVR & 0.87851 \\
Adaboost+kNNR & \textbf{9.9208e-09} \\
Adaboost+RFR & 0.1205\\
C4.5+ DTR & \textbf{2.6341e-05} \\
ESS & 0.6341 \\
\hline 
\end{tabular}
 \caption*{(b)}
    \end{minipage}
  \end{minipage}
    \caption{ Results in the table are p-values generated by performing a T-test on MSE over the ground truth graph (a) for the target variable \textcolor{orange}{T}.  Data shift between source and target occurs by change in the distribution of the context variable \textcolor{green}{$C_1$}, over Gaussian data with sample size = 1000. The bold results indicate that the average performance of \rctl~is significantly better than the  average performance of other approaches for the p-value 0.05.}
\end{figure*}

\begin{figure*}
    \centering
    \begin{minipage}{\textwidth}
  \begin{minipage}[b]{0.49\textwidth}
    \centering
    \begin{tikzpicture}[transform shape,scale=.6]
	\tikzset{vertex/.style = {shape=circle,align=center,draw=black, fill=white}}
\tikzset{edge/.style = {->,> = latex',thick}}
	\node[vertex,thick] (l) at  (-5.3,5) {L};
	\node[vertex,thick] (k) at  (-3.7,5) {K};
	\node[vertex,thick] (j) at  (-3.7,3.8) {J};
	\node[vertex,thick](m) at  (-5.3,3.8) {M};
	\node[vertex,thick] (n) at  (-4.5,2.6) {N};
    \node[vertex,thick,fill=green] (c1) at  (-2.5,5) {c1};
    \node[vertex,thick,fill= gray!40,dashed] (u) at  (-1.5,6.2) {U};
    \node[vertex,thick, fill = green] (c2) at  (-0.5,5) {c2};
    \node[vertex,thick](x) at  (-0.5,3.8) {X};
    \node[vertex,thick,fill=Dandelion ] (t) at  (-0.5,1.4) {T};
    \node[vertex,thick ] (y) at  (-1.5,.2) {Y};
	\node[vertex,thick] (p) at  (0.8,1.4) {P};
	\node[vertex,thick] (q) at  (0.8,0.2) {Q};
	\node[vertex,thick] (b) at  (0.8,3.8) {B};
    \node[vertex,thick] (d) at  (0.8,5) {D};
    \node[vertex,thick] (e) at  (2,5) {E};
	\node[vertex,thick] (i) at  (2,3.8) {I};
    \node[vertex,thick] (f) at  (0.8,6.2) {F};
    \node[vertex,thick] (g) at  (2,6.2) {G};
    \node[vertex,thick] (h) at  (3.2,6.2) {H};
	\draw[edge] (k) to (l);
	\draw[edge] (k) to (m);
	\draw[edge] (k) to (j);
	\draw[edge] (m) to (n);
	\draw[edge] (j) to (n);
	\draw[edge] (l) to (m);
	\draw[edge,fill = Cyan] (u) to (c1);
	\draw[edge,fill = Cyan] (u) to (c2);
	\draw[edge] (c1) to (y);
	\draw[edge] (c2) to (x);
	\draw[edge] (x) to (t);
	\draw[edge] (t) to (y);
	\draw[edge] (t) to (p);
	\draw[edge,] (p) to (q);
	\draw[edge] (d) to (b);
	\draw[edge] (c2) to (b);    
	\draw[edge] (e) to (b);
	\draw[edge] (e) to (i);
	\draw[edge] (f) to (d);
	\draw[edge] (g) to (e);
	\draw[edge] (e) to (i);
	\draw[edge] (h) to (e);
	\draw[edge] (j) to (x);
\end{tikzpicture}
\subcaption{}
  \end{minipage}
  \hfill
  \begin{minipage}[b]{0.49\textwidth}
    \centering
\scriptsize\begin{tabular}{l c } 
\hline\hline 
Methodology & p-value \\ [0.5ex] 
\hline
Baseline & \textbf{0.0225} \\
GSS & \textbf{1.2961e-05}\\ 
CMIM+SVR & \textbf{1.212e-10} \\
CMIM+kNNR &\textbf{ 1.6860e-14}\\
CMIM+RFR &\textbf{ 0.0001} \\
MIM+SVR & \textbf{1.212e-10} \\
MIM+kNNR & \textbf{3.344e-14} \\
MIM+RFR &\textbf{0.0001} \\
Adaboost+SVR & \textbf{1.212e-10} \\
Adaboost+kNNR & \textbf{5.148e-13} \\
Adaboost+RFR & 0.1290\\
C4.5+ DTR & \textbf{2.562e-10}\\
\hline 
\end{tabular}
 \caption*{(b)}
    \end{minipage}
  \end{minipage}
    \caption{ Results in the table are p-values generated by performing a T-test on MSE over the ground truth graph (a) for the target variable \textcolor{orange}{T}.  Data shift between source and target occurs by change in the distribution of the context variables \textcolor{green}{$C_1$}\& \textcolor{green}{$C_2$}, over Discrete data with sample size = 1000. The bold results indicate that the average performance of \rctl~ is significantly better than the average performance of other approaches for the p-value 0.05.}
\end{figure*}

\begin{figure*}
    \centering
    \begin{minipage}{\textwidth}
  \begin{minipage}[b]{0.49\textwidth}
    \centering
    \begin{tikzpicture}[transform shape,scale=.6]
	\tikzset{vertex/.style = {shape=circle,align=center,draw=black, fill=white}}
\tikzset{edge/.style = {->,> = latex',thick}}
	   \node[vertex,thick,fill=green] (c1) at  (-2.5,5) {c1};
    \node[vertex,thick,fill= gray!40,dashed] (u) at  (-1.5,6.2) {U};
    \node[vertex,thick,fill=green] (c2) at  (-0.5,5) {c2};
    \node[vertex,thick](x) at  (-0.5,3.8) {X};
    \node[vertex,thick,fill=Dandelion ] (t) at  (-0.5,1.4) {T};
    \node[vertex,thick ] (y) at  (-1.5,.2) {Y};
	\node[vertex,thick] (p) at  (0.8,1.4) {P};
	\node[vertex,thick] (q) at  (0.8,0.2) {Q};
	\node[vertex,thick] (b) at  (0.8,3.8) {B};
    \node[vertex,thick] (d) at  (0.8,5) {D};
	\draw[edge,fill = Cyan] (u) to (c1);
	\draw[edge,fill = Cyan] (u) to (c2);
	\draw[edge] (c1) to (y);
	\draw[edge] (c2) to (x);
	\draw[edge] (x) to (t);
	\draw[edge] (t) to (y);
	\draw[edge] (t) to (p);
	\draw[edge,] (p) to (q);
	\draw[edge] (d) to (b);
	\draw[edge] (c2) to (b); 
\end{tikzpicture}
\subcaption{}
  \end{minipage}
  \hfill
  \begin{minipage}[b]{0.49\textwidth}
    \centering
\scriptsize\begin{tabular}{l c } 
\hline\hline 
Methodology & p-value \\ [0.5ex] 
\hline
Baseline &\textbf{ 0.00013} \\
GSS & \textbf{0.0001} \\ 
CMIM+SVR &\textbf{ 4.785e-10} \\
CMIM+kNNR & \textbf{5.4748e-10} \\
CMIM+RFR & \textbf{0.0059} \\
MIM+SVR & \textbf{4.7853e-10} \\
MIM+kNNR & \textbf{4.7034e-10}\\
MIM+RFR &\textbf{ 0.0056 }\\
Adaboost+SVR & \textbf{4.785e-10} \\
Adaboost+kNNR & \textbf{7.088e-10 }\\
Adaboost+RFR & 0.5531\\
C4.5+ DTR & \textbf{ 3.136e-09}\\
ESS &\textbf{ 0.0006} \\
\hline 
\end{tabular}
 \caption*{(b)}
    \end{minipage}
  \end{minipage}
    \caption{ Results in the table are p-values generated by performing a T-test on MSE over the ground truth graph (a) for the target variable \textcolor{orange}{T}.  Data shift between source and target occurs by change in the distribution of the context variables \textcolor{green}{$C_1$}\& \textcolor{green}{$C_2$}, over Discrete data. The bold results indicate that the average performance of \rctl~ is significantly better than the average performance of other approaches for the p-value 0.05.}
\end{figure*}

\begin{figure*}
    \centering
    \begin{minipage}{\textwidth}
  \begin{minipage}[b]{0.49\textwidth}
    \centering
    \begin{tikzpicture}[transform shape,scale=.6]
	\tikzset{vertex/.style = {shape=circle,align=center,draw=black, fill=white}}
\tikzset{edge/.style = {->,> = latex',thick}}
	\node[vertex,thick] (l) at  (-5.3,5) {L};
	\node[vertex,thick] (k) at  (-3.7,5) {K};
	\node[vertex,thick] (j) at  (-3.7,3.8) {J};
	\node[vertex,thick](m) at  (-5.3,3.8) {M};
	\node[vertex,thick] (n) at  (-4.5,2.6) {N};
    \node[vertex,thick,fill=green] (c1) at  (-2.5,5) {c1};
    \node[vertex,thick,fill= gray!40,dashed] (u) at  (-1.5,6.2) {U};
    \node[vertex,thick] (c2) at  (-0.5,5) {c2};
    \node[vertex,thick](x) at  (-0.5,3.8) {X};
    \node[vertex,thick,fill=Dandelion ] (t) at  (-0.5,1.4) {T};
    \node[vertex,thick ] (y) at  (-1.5,.2) {Y};
	\node[vertex,thick] (p) at  (0.8,1.4) {P};
	\node[vertex,thick] (q) at  (0.8,0.2) {Q};
	\node[vertex,thick] (b) at  (0.8,3.8) {B};
    \node[vertex,thick] (d) at  (0.8,5) {D};
    \node[vertex,thick] (e) at  (2,5) {E};
	\node[vertex,thick] (i) at  (2,3.8) {I};
    \node[vertex,thick] (f) at  (0.8,6.2) {F};
    \node[vertex,thick] (g) at  (2,6.2) {G};
    \node[vertex,thick] (h) at  (3.2,6.2) {H};
	\draw[edge] (k) to (l);
	\draw[edge] (k) to (m);
	\draw[edge] (k) to (j);
	\draw[edge] (m) to (n);
	\draw[edge] (j) to (n);
	\draw[edge] (l) to (m);
	\draw[edge,fill = Cyan] (u) to (c1);
	\draw[edge,fill = Cyan] (u) to (c2);
	\draw[edge] (c1) to (y);
	\draw[edge] (c2) to (x);
	\draw[edge] (x) to (t);
	\draw[edge] (t) to (y);
	\draw[edge] (t) to (p);
	\draw[edge,] (p) to (q);
	\draw[edge] (d) to (b);
	\draw[edge] (c2) to (b);    
	\draw[edge] (e) to (b);
	\draw[edge] (e) to (i);
	\draw[edge] (f) to (d);
	\draw[edge] (g) to (e);
	\draw[edge] (e) to (i);
	\draw[edge] (h) to (e);
	\draw[edge] (j) to (x);
\end{tikzpicture}
\subcaption{}
  \end{minipage}
  \hfill
  \begin{minipage}[b]{0.49\textwidth}
    \centering 
\scriptsize\begin{tabular}{l c } 
\hline\hline 
Methodology & p-value \\ [0.5ex] 
\hline
Baseline & \textbf{0.0001} \\
GSS &\textbf{ 1.585e-06} \\ 
CMIM+SVR & \textbf{7.227e-15 }\\
CMIM+kNNR & \textbf{4.352e-09} \\
CMIM+RFR &\textbf{ 5.212e-06} \\
MIM+SVR &\textbf{7.227e-15} \\
MIM+kNNR &\textbf{7.327e-09 }\\
MIM+RFR & \textbf{0.0001} \\
Adaboost+SVR & \textbf{7.229e-15}  \\
Adaboost+kNNR & \textbf{3.760e-08} \\
Adaboost+RFR & 0.054\\
C4.5+ DTR &  \textbf{8.380e-10}\\
\hline 
\end{tabular}
 \caption*{(b)}
    \end{minipage}
  \end{minipage}
    \caption{ Results in the table are p-values generated by performing a T-test on MSE over the ground truth graph (a) for the target variable \textcolor{orange}{T}.  Data shift between source and target occurs by change in the distribution of the context variable \textcolor{green}{$C_1$}, over Discrete data. The bold results indicate that the average performance of \rctl~ is significantly better than the average performance of other approaches for the p-value 0.05.}
\end{figure*}

\begin{figure*}
    \centering
    \begin{minipage}{\textwidth}
  \begin{minipage}[b]{0.49\textwidth}
    \centering
    \begin{tikzpicture}[transform shape,scale=.6]
	\tikzset{vertex/.style = {shape=circle,align=center,draw=black, fill=white}}
\tikzset{edge/.style = {->,> = latex',thick}}
	\node[vertex,thick] (l) at  (-5.3,5) {L};
	\node[vertex,thick] (k) at  (-3.7,5) {K};
	\node[vertex,thick] (j) at  (-3.7,3.8) {J};
	\node[vertex,thick](m) at  (-5.3,3.8) {M};
	\node[vertex,thick] (n) at  (-4.5,2.6) {N};
    \node[vertex,thick,fill=green] (c1) at  (-2.5,5) {c1};
    \node[vertex,thick,fill= gray!40,dashed] (u) at  (-1.5,6.2) {U};
    \node[vertex,thick, fill = green] (c2) at  (-0.5,5) {c2};
    \node[vertex,thick](x) at  (-0.5,3.8) {X};
    \node[vertex,thick,fill=Dandelion ] (t) at  (-0.5,1.4) {T};
    \node[vertex,thick ] (y) at  (-1.5,.2) {Y};
	\node[vertex,thick] (p) at  (0.8,1.4) {P};
	\node[vertex,thick] (q) at  (0.8,0.2) {Q};
	\node[vertex,thick] (b) at  (0.8,3.8) {B};
    \node[vertex,thick] (d) at  (0.8,5) {D};
    \node[vertex,thick] (e) at  (2,5) {E};
	\node[vertex,thick] (i) at  (2,3.8) {I};
    \node[vertex,thick] (f) at  (0.8,6.2) {F};
    \node[vertex,thick] (g) at  (2,6.2) {G};
    \node[vertex,thick] (h) at  (3.2,6.2) {H};
	\draw[edge] (k) to (l);
	\draw[edge] (k) to (m);
	\draw[edge] (k) to (j);
	\draw[edge] (m) to (n);
	\draw[edge] (j) to (n);
	\draw[edge] (l) to (m);
	\draw[edge,fill = Cyan] (u) to (c1);
	\draw[edge,fill = Cyan] (u) to (c2);
	\draw[edge] (c1) to (y);
	\draw[edge] (c2) to (x);
	\draw[edge] (x) to (t);
	\draw[edge] (t) to (y);
	\draw[edge] (t) to (p);
	\draw[edge,] (p) to (q);
	\draw[edge] (d) to (b);
	\draw[edge] (c2) to (b);    
	\draw[edge] (e) to (b);
	\draw[edge] (e) to (i);
	\draw[edge] (f) to (d);
	\draw[edge] (g) to (e);
	\draw[edge] (e) to (i);
	\draw[edge] (h) to (e);
	\draw[edge] (j) to (x);
\end{tikzpicture}
\subcaption{}
  \end{minipage}
  \hfill
  \begin{minipage}[b]{0.49\textwidth}
    \centering 
\scriptsize\begin{tabular}{l c } 
\hline\hline 
Methodology & p-value \\ [0.5ex] 
\hline
Baseline &\textbf{0.1651} \\
GSS & \textbf{0.1687 }\\ 
CMIM+SVR & 3.5196e-07 \\
CMIM+kNNR &0.0277 \\
CMIM+RFR & \textbf{0.066} \\
MIM+SVR & 3.519e-07 \\
MIM+kNNR &0.0277 \\
MIM+RFR & \textbf{0.0664 }\\
Adaboost+SVR & \textbf{0.3315} \\
Adaboost+kNNR &\textbf{0.1858} \\
Adaboost+RFR &\textbf{ 0.1865}\\
C4.5+ DTR &  0.0228\\
\hline 
\end{tabular}
\caption*{(b)}
    \end{minipage}
  \end{minipage}
    \caption{ Results in the table are p-values generated by performing a T-test on MSE over the ground truth graph (a) for the target variable \textcolor{orange}{T}.  Data shift between source and target occurs by change in the distribution of the context variable \textcolor{green}{$C_1$}\& \textcolor{green}{$C_2$}, over Discrete data with sample size = 10000. The bold results indicate that the average performance of \rctl~ is comparable to the average performance of other approaches for the p-value 0.05.}
\end{figure*}

\begin{figure*}
    \centering
    \begin{minipage}{\textwidth}
  \begin{minipage}[b]{0.49\textwidth}
    \centering
    \begin{tikzpicture}[transform shape,scale=.6]
	\tikzset{vertex/.style = {shape=circle,align=center,draw=black, fill=white}}
\tikzset{edge/.style = {->,> = latex',thick}}
	\node[vertex,thick] (l) at  (-5.3,5) {L};
	\node[vertex,thick] (k) at  (-3.7,5) {K};
	\node[vertex,thick] (j) at  (-3.7,3.8) {J};
	\node[vertex,thick,fill=Dandelion](m) at  (-5.3,3.8) {M};
	\node[vertex,thick] (n) at  (-4.5,2.6) {N};
    \node[vertex,thick,fill=green] (c1) at  (-2.5,5) {c1};
    \node[vertex,thick,fill= gray!40,dashed] (u) at  (-1.5,6.2) {U};
    \node[vertex,thick] (c2) at  (-0.5,5) {c2};
    \node[vertex,thick](x) at  (-0.5,3.8) {X};
    \node[vertex,thick ] (t) at  (-0.5,1.4) {T};
    \node[vertex,thick ] (y) at  (-1.5,.2) {Y};
	\node[vertex,thick] (p) at  (0.8,1.4) {P};
	\node[vertex,thick] (q) at  (0.8,0.2) {Q};
	\node[vertex,thick] (b) at  (0.8,3.8) {B};
    \node[vertex,thick] (d) at  (0.8,5) {D};
    \node[vertex,thick] (e) at  (2,5) {E};
	\node[vertex,thick] (i) at  (2,3.8) {I};
    \node[vertex,thick] (f) at  (0.8,6.2) {F};
    \node[vertex,thick] (g) at  (2,6.2) {G};
    \node[vertex,thick] (h) at  (3.2,6.2) {H};
	\draw[edge] (k) to (l);
	\draw[edge] (k) to (m);
	\draw[edge] (k) to (j);
	\draw[edge] (m) to (n);
	\draw[edge] (j) to (n);
	\draw[edge] (l) to (m);
	\draw[edge,fill = Cyan] (u) to (c1);
	\draw[edge,fill = Cyan] (u) to (c2);
	\draw[edge] (c1) to (y);
	\draw[edge] (c2) to (x);
	\draw[edge] (x) to (t);
	\draw[edge] (t) to (y);
	\draw[edge] (t) to (p);
	\draw[edge,] (p) to (q);
	\draw[edge] (d) to (b);
	\draw[edge] (c2) to (b);    
	\draw[edge] (e) to (b);
	\draw[edge] (e) to (i);
	\draw[edge] (f) to (d);
	\draw[edge] (g) to (e);
	\draw[edge] (e) to (i);
	\draw[edge] (h) to (e);
	\draw[edge] (j) to (x);
\end{tikzpicture}
\subcaption{}
  \end{minipage}
  \hfill
  \begin{minipage}[b]{0.49\textwidth}
    \centering  
\scriptsize\begin{tabular}{l c } 
\hline\hline 
Methodology & p-value \\ [0.5ex] 
\hline
Baseline & \textbf{6.7513e-12} \\
GSS & \textbf{0.0009} \\ 
CMIM+SVR & \textbf{0.0056} \\
CMIM+kNNR &\textbf{1.03614e-06} \\
CMIM+RFR & \textbf{1.2254e-11} \\
MIM+SVR & \textbf{1.0319e-05} \\
MIM+kNNR & \textbf{1.1642e-07}\\
MIM+RFR & \textbf{4.537e-10}\\
Adaboost+SVR & \textbf{0.0305} \\
Adaboost+kNNR & \textbf{0.0138}\\
Adaboost+RFR & \textbf{3.993e-10}\\
C4.5+ DTR &  \textbf{2.0194e-07}\\
\hline 
\end{tabular}
\caption*{(b)}
    \end{minipage}
  \end{minipage}
    \caption{ Results in the table are p-values generated by performing a T-test on MSE over the ground truth graph (a) for the target variable \textcolor{orange}{M}.  Data shift between source and target occurs by change in the distribution of the context variable \textcolor{green}{$C_1$}, over Discrete data with sample size = 1000. The bold results indicate that the average performance of \rctl~ is significantly better than the average performance of other approaches for the p-value 0.05.}\label{fig:targetM}
\end{figure*}

\end{document}